\documentclass[11pt]{article}
\usepackage{geometry}
\usepackage{amsmath,amsthm,amssymb,xcolor,fullpage,hyperref,bbm}
\usepackage{mathrsfs}
\usepackage{algorithm, algorithmicx, algpseudocode}
\usepackage{dsfont}
\usepackage[style=trad-alpha,maxcitenames=1,maxalphanames=4]{biblatex}
\hypersetup{
  colorlinks=true,
  linkcolor=blue,
  filecolor=blue,
  citecolor = blue,      
  urlcolor=cyan,
}

\usepackage[capitalize,noabbrev]{cleveref}
\crefformat{equation}{#2(#1)#3}
\Crefformat{equation}{#2(#1)#3}

\Crefformat{figure}{#2Figure #1#3}
\Crefname{assumption}{Assumption}{Assumptions}
\Crefname{defn}{Definition}{Definitions}
\Crefformat{assumption}{#2Assumption #1#3}
\crefformat{defn}{#2Definition #1#3}
\Crefformat{defn}{#2Definition #1#3}

\makeatletter
\newtheorem*{rep@theorem}{\rep@title}
\newcommand{\newreptheorem}[2]{%
\newenvironment{rep#1}[1]{%
 \def\rep@title{#2 \ref{##1}}%
 \begin{rep@theorem}}%
 {\end{rep@theorem}}}
\makeatother

\theoremstyle{plain}
\newtheorem{theorem}{Theorem}
\newtheorem{lemma}[theorem]{Lemma}

\newtheorem{proposition}[theorem]{Proposition}

\newtheorem{claim}[theorem]{Claim}
\newtheorem{corollary}[theorem]{Corollary}

\theoremstyle{definition}
\newtheorem{definition}[theorem]{Definition}

\newtheorem{assumption}{Assumption}
\theoremstyle{remark}
\newtheorem{remark}[theorem]{Remark}

\numberwithin{theorem}{section}

\newcount\COMMENTS
\newcommand{\yuval}[1]{\ifnum\COMMENTS=1\textcolor{red}{[yuval: #1]}\fi}
\newcommand{\costis}[1]{\ifnum\COMMENTS=1\textcolor{blue}{[costis: #1]}\fi}
\newcommand{\fish}[1]{\ifnum\COMMENTS=1\textcolor{orange}{[Fish: #1]}\fi}
\newcommand{\noah}[1]{\ifnum\COMMENTS=1\textcolor{purple}{[ng: #1]}\fi}
\newcommand{\todo}[1]{\ifnum\COMMENTS=1\textcolor{red}{TODO: #1}\fi}

\ifnum\COMMENTS=1
\usepackage[inline]{showlabels} 
\fi

\newcount\ARXIV
\newcommand{\unfinished}[1]{\ifnum\ARXIV=0{#1}\fi}

\newcommand{\ch}{\mathsf{ch}}
\newcommand{\SA}{\mathscr{A}}
\newcommand{\Fbase}{F_{\mathsf{base}}}
\newcommand{\ubase}{\mathbf{u}_{\mathsf{base}}}

\newcommand{\pair}{\mathfrak{p}}
\newcommand{\p}[1]{\left(#1\right)}
\newcommand{\ps}[1]{\left[#1\right]}
\newcommand{\set}[1]{\left\{#1\right\}}
\newcommand{\card}[1]{\left |#1\right |}
\newcommand{\norm}[1]{\left\|#1\right\|}

\newcommand{\bx}{\mathbf{x}}
\newcommand{\bz}{\mathbf{z}}
\newcommand{\bs}{\mathbf{s}}
\newcommand{\bp}{\mathbf{p}}
\newcommand{\bff}{\mathbf{f}}

\newcommand{\bu}{\mathbf{u}}
\newcommand{\st}{\star}
\newcommand{\btau}{\underline{\tau}}
\newcommand{\etau}{\overline{\tau}}
\newcommand{\ExtRegret}{\mathbf{ExtRegret}}
\newcommand{\SwapRegret}{\mathbf{SwapRegret}}
\newcommand{\Alg}{\mathtt{Alg}}
\newcommand{\MWU}{\mathtt{MWU}}
\newcommand{\Algexternal}{\mathtt{Alg}_{\mathsf{Ext}}}
\newcommand{\ME}{\mathcal{E}}
\newcommand{\MI}{\mathcal{I}}
\newcommand{\MQ}{\mathcal{Q}}
\newcommand{\bq}{\mathbf{q}}
\newcommand{\MA}{\mathcal{A}}

\newcommand{\MG}{\mathcal{G}}
\newcommand{\MS}{\mathcal{S}}

\newcommand{\MX}{\mathcal{X}}
\newcommand{\MH}{\mathcal{H}}
\newcommand{\Ldim}{\mathrm{LDim}}
\newcommand{\esfsd}[1]{\mathrm{SFat}(#1,\delta)}
\newcommand{\MZ}{\mathcal{Z}}
\newcommand{\MT}{\mathcal{T}}
\newcommand{\MY}{\mathcal{Y}}
\newcommand{\MF}{\mathcal{F}}
\newcommand{\Bin}{\mathrm{Bin}}
\newcommand{\BN}{\mathbb{N}}
\newcommand{\co}{\mathrm{co}}
\newcommand{\Rad}{\mathfrak{R}}

\renewcommand{\^}[1]{^{\p{#1}}}

\newcommand{\vect}[1]{\mathbf{#1}}

\newcommand{\vx}{\vect{x}}
\newcommand{\vu}{\vect{u}}

\newcommand{\vut}{\vu\^{t}}
\newcommand{\vxt}{\vx\^{t}}
\newcommand{\vxti}{\vx\^{t}[i]}
\newcommand{\vf}{\vect{f}}
\newcommand{\vft}{\vect{f}\^t}

\newcommand{\EE}{\mathbb{E}}
\newcommand{\RR}{\mathbb{R}}
\newcommand{\1}{\mathbbm{1}}
\newcommand{\ep}{\epsilon}

\newcommand{\SR}{\SwapRegret}

\newcommand{\DSR}{\SR}

\newcommand{\tl}{{\text{\underline{$t$}}}}
\newcommand{\tr}{\bar{t}}
\newcommand{\roo}{\emptyset}
\newcommand{\dep}{\text{depth}}
\newcommand{\vxtr}{\vxt_{r\^t}}
\newcommand{\Swap}[2]{\text{Swap}(#1 \to #2)}

\newcommand{\E}{\mathbb{E}}

\newcommand{\by}{\mathbf{y}}
\newcommand{\BI}{\mathbb{I}}

\renewcommand{\^}[1]{^{\p{#1}}}
\newcommand{\lng}{\langle}
\newcommand{\rng}{\rangle}

\newcommand{\TreeSwap}{\mathtt{TreeSwap}}
\newcommand{\BanditTreeSwap}{\mathtt{BanditTreeSwap}}
\newcommand{\ExpMulti}{\mathtt{Exp3Multi}}
\newcommand{\One}{\mathbbm{1}}
\newcommand{\ol}{\overline}

\newcommand{\CommCe}{\mathtt{CommCE}}
\newcommand{\QueryCe}{\mathtt{QueryCE}}
\newcommand{\Algupdate}{\mathtt{Alg.update}}
\newcommand{\Algact}{\mathtt{Alg.act}}
\newcommand{\BR}{\mathbb{R}}

 \bibliography{main}

\title{From External to Swap Regret 2.0:\\ An Efficient Reduction for Large Action Spaces}

\author{Yuval Dagan\thanks{Email: \texttt{yuvald@berkeley.edu}.}\\ UC Berkeley  \and  Constantinos Daskalakis\thanks{Email: \texttt{costis@csail.mit.edu}. Supported by NSF Awards CCF-1901292, DMS-2022448, and DMS2134108, a Simons Investigator Award, the Simons Collaboration on the Theory of Algorithmic Fairness, and a DSTA grant.} \\ MIT CSAIL \and Maxwell Fishelson\thanks{Email: \texttt{maxfish@mit.edu}.} \\ MIT CSAIL \and Noah Golowich\thanks{Email: \texttt{nzg@mit.edu}. Supported by a Fannie \& John Hertz Foundation Fellowship and an NSF Graduate Fellowship.} \\ MIT CSAIL}
\date{December 6, 2023}

\begin{document}

\maketitle
\thispagestyle{empty}
\addtocounter{page}{-1}
\begin{abstract}

    We provide a novel reduction from {\em swap-regret} minimization to {\em external-regret} minimization, which improves upon the classical reductions of Blum-Mansour~\cite{Blum07:From} and Stoltz-Lugosi~\cite{Stoltz05:Internal} in that it does not require finiteness of the space of actions. We show that, whenever there exists a no-external-regret algorithm for some hypothesis class, there must also exist a no-swap-regret algorithm for that same class. For the problem of learning with expert advice, our result implies that it is possible to guarantee that the swap regret is bounded by $\epsilon$ after $(\log N)^{\tilde O(1/\epsilon)}$ rounds and with $O(N)$ per iteration complexity, where $N$ is the number of experts, while the classical reductions of Blum-Mansour and Stoltz-Lugosi require at least  $\Omega(N/\epsilon^2)$ \fish{should be $\Omega$?}\noah{updated} rounds and at least $\Omega(N^3)$ total computational cost. Our result comes with an associated lower bound, which---in contrast to that in~\cite{Blum07:From}---holds for {\em oblivious} and {\em $\ell_1$-constrained} adversaries and learners that can employ distributions over experts, showing that the number of rounds must be $\tilde{\Omega}(N/\epsilon^2)$ or exponential in $1/\epsilon$.
    
    Our reduction implies that, if no-regret learning is possible in some game,  then this game must have approximate {\em correlated equilibria}, of arbitrarily good approximation. This strengthens the folklore implication of no-regret learning that approximate {\em coarse} correlated equilibria exist. Importantly, it provides a sufficient condition for the existence of approximate correlated equilibrium which vastly extends the requirement that the action set is finite or the requirement that the action set is compact and the utility functions are continuous, allowing for games  with finite Littlestone or finite sequential fat shattering dimension, thus answering a question left open by~\cite{daskalakis2022fast,assos2023online}.
    Moreover, it answers several  outstanding questions about equilibrium computation and/or learning in games. In particular, for constant values of $\epsilon$: (a)~we show that $\epsilon$-approximate correlated equilibria in {\em extensive-form games} can be computed efficiently, advancing a long-standing open problem for extensive-form games; see e.g.~\cite{von2008extensive,farina2023polynomial}; (b)~we show that the query and communication complexities of computing $\epsilon$-approximate correlated equilibria  in $N$-action normal-form games are $N \cdot \mathrm{poly} \log(N)$ and $\mathrm{poly} \log N$ respectively, advancing an open problem of~\cite{babichenko2020informational}; (c) we show that $\epsilon$-approximate correlated equilibria of sparsity $\mathrm{poly} \log N$ can be computed efficiently, advancing an open problem of~\cite{babichenko2014simple}; (d) finally, we show that in the adversarial bandit setting, sublinear swap regret can be achieved in only $\tilde O(N)$ rounds, advancing an open problem of \cite{Blum07:From,ito2020tight}. \fish{add bandits to abstract}\noah{added}
\end{abstract}
\newpage

\noah{TODOs:
  \begin{itemize}
  \item Change to sequential rademacher. Also make sure the dual class issue is corrected.
  \item Check/finish bandits and adaptive lower bound. DONE
  \item Clean up application section in general. DONE
  \item Oblivious lower bound pass/clean. 
    \item EFG short explanation.
  \end{itemize}
  }

\section{Introduction}
\emph{No-regret learning} has been a central topic of study in game theory and online learning over the last several decades~\cite{Hannan57:Approximation,Fudenberg98:Theory,Cesa06:Prediction}. In view of the worst-case nature of the associated learning guarantee, no-regret learning has found myriad applications in a variety of settings, with varying degrees of restriction on the adversary's behavior. %
They are also 
particularly salient in game theory due to their connection with decentralized equilibrium computation. Indeed, it is well understood that, if players in a normal-form game iteratively update their strategies using a no-regret learning algorithm, then the empirical distribution of their strategies over time converges to a type of correlated equilibrium, depending on the notion of regret used.

The most commonly studied type of regret, called \emph{external regret}, measures the amount of extra utility that the agent could have gained if, instead of her realized sequence of strategies, she had instead played her best fixed action in hindsight. In a multi-agent interaction, if each agent uses a sublinear external regret learning algorithm to iteratively update her strategy, the empirical distribution of the agents' play converges to a \emph{coarse correlated equilibrium} (CCE). A CCE is a correlated distribution over actions under which no player can improve her utility if, instead of playing according to the distribution, she unilaterally switches to playing any single fixed action. CCEs are a convex relaxation of \emph{Nash equilibria}, which are computationally intractable even for normal-form games with a finite number of actions per player~\cite{Daskalakis09:Complexity,Chen09:Settling}.
While a plethora of efficient algorithms for minimizing external regret are known even when the size of the game is large (see e.g.~\cite{Fudenberg98:Theory,Cesa06:Prediction,bubeck2012regret}), the twin notions of external regret and coarse correlated equilibrium are too weak for many applications. In particular, the notion of CCE does not capture the fact that the action sampled from the CCE distribution for some player may leak information about what actions were sampled for the other players, which the player could potentially exploit to improve her utility.

Using the perspective of Bayesian rationality, Aumann introduced the concept of \emph{correlated equilibrium} (CE), which corrects for this deficit~\cite{Aumann74:Subjectivity}. A CE is a correlated distribution with the property that the action sampled for each player maximizes her expected utility against the distribution over actions sampled for the other players,  \emph{conditioning on the action sampled for this player}. Like CCE, the concept of CE is a convex relaxation of Nash equilibrium, and it can be reached in a decentralized manner by averaging the empirical play of algorithms which have sublinear \emph{swap regret}. This %
 measures the amount of extra utility that the agent could have gained, in hindsight, if she were to go back in time and transform the strategies that she played using  the best, fixed \emph{swap function} (see Definition \ref{def:SR}). %
The stronger nature of swap regret leads it to have numerous applications, including in calibration and multicalibration \cite{globus2023multicalibration,kleinberg2023ucalibration} and Bayesian games \cite{mansour2022strategizing}, amongst others. %

\subsection{Swap regret: challenges with large action spaces} \label{sec:prior-work} Despite the more appealing guarantees satisfied by swap regret minimization and its twin notion of CE, no-swap regret learning algorithms have not been as widely adopted as no-external regret ones. This is due in part to the substantially inferior quantitative guarantees offered by the best-known swap-regret-minimizing algorithms in terms of their dependence in the number of actions available to the learner. In particular, existing algorithms are inefficient in many settings of interest where the action space is exponentially large in the game's description complexity, or even infinite. To illustrate, we first consider the case of no-regret learning with a finite set of $N$ actions, which is known as the ``experts setting.'' Standard external-regret-minimizing algorithms, such as exponential weights~\cite{Cesa06:Prediction}, guarantee that the average external regret over $T$ rounds is bounded by $\ep$ as long as $T \gtrsim \frac{\log N}{\ep^2}$.\footnote{We consider normalized regret throughout the paper, i.e., we divide the cumulative regret by the number of rounds $T$.} In contrast, the best-known swap-regret-minimizing algorithms, which are all based on generic reductions from swap regret minimization to external regret minimization~\cite{Stoltz05:Internal,Blum07:From}, guarantee that the average swap regret over $T$ rounds is $\ep$ as long as $T \gtrsim \frac{N \log N}{\ep^2}$. Thus, prior work left an \emph{exponential gap} between the best-known algorithms for swap and external regret. It was explicitly asked by Blum and Mansour~\cite{Blum07:From} if this gap could be improved. This gap is particularly noteworthy in light of many recent applications of no-regret learning, such as for solving games such as Poker~\cite{brown2019superhuman} and  Diplomacy~\cite{fair2022human}, all of which have the property that $N$ is moderate or large. \yuval{more real-world applications}\noah{after monday}

Prior work also left a polynomial-sized gap in the \emph{bandit} setting, in which the learner must choose a single action each round and only receives the utility for that action. While it is known that $T \gtrsim \frac{N^2 \log N}{\ep^2}$ rounds suffice \cite{jin2022vlearning,ito2020tight} to ensure that swap regret is bounded by $\ep$, the best known lower bound was that $\frac{N \log N}{\ep^2}$ rounds are necessary \cite{ito2020tight,Blum07:From}. The bandit setting is particularly useful due to its applications in reinforcement learning \cite{jin2022vlearning} and related areas. 

 Prior to the present work, the gap between swap regret and external regret was even larger in settings where the number of actions available to the learner is unbounded or infinite. %
 For instance, suppose that each agent's action space is the set of parameters of a neural network: multi-agent interactions in which each agent chooses a neural network can  be used to model tasks such as training generative adversarial networks~\cite{goodfellow2014generative}, autonomous driving~\cite{shalev2016safe}, or economic decision making~\cite{zheng2020ai}.
 \yuval{avoided saying that NN are example of Littlestone, introduced fat shattering because littlestone is 0-1} In these cases, the number of possible networks is very large. In a more general setting, the action space is typically assumed to be constrained by a combinatorial complexity measure, such as the \emph{Littlestone dimension} or \emph{sequential fat shattering dimension} (see Section \ref{app:dimensions} for formal definitions).\yuval{are neural networks a good example for Littlestone dimension? I don't think so. In fact, a discretization will be as good.} In particular, if the learner's action space has Littlestone dimension $L$, then it was known~\cite{ben2009agnostic,alon2021adversarial} that as long as the number $T$ of rounds satisfies $T \geq \frac{L}{\ep^2}$, there is an algorithm which achieves at most $\ep$ external regret.\footnote{This bound is optimal; see \cite{ben2009agnostic}.} Since the reductions of~\cite{Stoltz05:Internal,Blum07:From} for bounding swap regret assume that the number $N$ of actions is bounded, prior to our work it was not known whether any class of finite Littlestone dimension has an algorithm with $o(T)$ swap regret, leaving open the possibility of an \emph{infinite} gap between swap and external regrets for classes of finite Littlestone dimension. \fish{In this work, we establish such an algorithm, obtaining at most $\epsilon$ swap regret after $L^{\tilde{O}(1/\epsilon)}$ rounds.}\noah{currently in these paras we're not talking about our results at all (saving that for below), so it would be a bit weird to only say that we close one of these gaps}

 \paragraph{Gaps in equilibrium computation.} The above gaps between swap and external regret also manifest as gaps between the best known results for computing $\ep$-approximate CE and CCE in various models of computation. We improve upon these gaps in the following settings:
 \begin{itemize}
     \item \emph{Normal-form games with $N$ actions.} We consider two computation problems. For simplicity we assume the number of players and $\ep$ are constants.
     \begin{itemize}
         \item In the \emph{communication complexity} model of computation (Definition \ref{def:query-c}), $\ep$-CCE may be computed with $O(\log^2 N)$ bits of communication using no-external regret algorithms together with a sampling procedure. In contrast, prior to this work, the best known bound for $\ep$-CE was exponentially worse, $O(N \log^2 N)$, using the swap regret algorithm of~\cite{Blum07:From}.
         \item In the \emph{query complexity} model of computation (Definition \ref{def:comm-c}), $\ep$-CCE may be computed using $O(N \log N)$ queries.  Prior to this work, the best known bound of $O(N^2 \log N)$ was quadratically worse for $\ep$-CE. 
         \item Finally, $\ep$-CCE which are $\mathrm{poly} \log (N)$-sparse\footnote{We use ``sparse'' to mean ``of low non-negative rank''; see the formal definition in \cref{def:sparse-eq}.} may be computed in polynomial time \cite{babichenko2014simple}, whereas prior to this work, it was unknown how to efficiently compute $\ep$-CE which are $o(N)$-sparse, marking another exponential gap (in the sparsity).
     \end{itemize}
     \item In \emph{infinite games} of Littlestone dimension $L < \infty$, for constant $\ep > 0$, $\ep$-CCE may be found in a decentralized manner by running $O(L)$ rounds of no-external regret algorithms~\cite{daskalakis2022fast}. In contrast, prior to our work it was not known if $\ep$-CE even \emph{exist} in games of finite Littlestone dimension. 
     \item  Finally, in \emph{extensive form games} with description length $n$ denoting the size of the tree, for which the number of actions\footnote{An action is specified by a contingency plan, mapping each information set to an outgoing edge at that  information set.} typically scales as $N=\exp(\Theta(n))$, $\ep$-CCE may be computed in $\mathrm{poly}(n)$ time (e.g., \cite{farina2022kernelized}).  However, prior to this work, the best known algorithms for computing $\ep$-CE took time exponential in $n$. Determining the complexity of $\ep$-CE was a well-known open question in this field; see e.g.~\cite{von2008extensive,farina2023polynomial}.\footnote{To be clear, $\ep$-CE here refers to the notion of $\ep$-approximate \emph{normal-form correlated equilibrium} (sometimes denoted $\ep$-NFCE), as opposed to relaxations of this notion,  such as extensive-form correlated equilibrium, which have been recently proposed, motivated in part by the apparent intractability of $\ep$-NFCE~\cite{von2008extensive}.}

 \end{itemize}

\subsection{Main results: near-optimal upper and lower bounds for swap regret}
\label{sec:results}
Our main upper bound is a new reduction from swap regret to external regret: any no-external regret learning algorithm can be transformed into a no-distributional swap regret learner. (These regret notions are formally defined in Section \ref{sec:prelim-online}.) %
We assume that a learner chooses, in each iteration $t \in [T]$, a distribution $\vxt \in \Delta_\MX$ over a set of actions $\MX$. After observing $\vxt$, an adversary selects a reward function $\vft \colon \MX \to \RR$, and the learner receives the reward $\vft(\vxt) = \EE_{s\^t \sim \vxt} \vft[s\^t]$. We assume the adversary's choices $\vft$ are constrained to lie in some convex function class $\mathcal{F} \subset [0,1]^\MX$.

\begin{theorem}[Informal version of Theorem \ref{thm:treeswap}] \label{thm:swap-to-external}
Let $d, M \in \mathbb{N}$ be given, and  suppose that there is a learner for some function class $\mathcal{F}$ which achieves external regret of $\ep$ after $M$ iterations. Then there is a learner for $\MF$ ($\TreeSwap$; Algorithm \ref{alg:treeswap}) which achieves a swap regret of at most $\ep + \frac{1}{d}$ after $T= M^d$ iterations.

If the per-iteration runtime complexity of the  external-regret learner is $C$, then the swap regret learner $\TreeSwap$ has a per-iteration amortized runtime complexity of $O(C)$.
\end{theorem}
Notice that the swap regret of $\TreeSwap$ depends only on the external regret of the assumed learner, and is \emph{independent of the number of actions} of the learner. In particular, it holds also for exponentially large or even infinite function classes.

\paragraph{Applications: concrete swap regret bounds.}
As applications of Theorem \ref{thm:swap-to-external}, in the setting of constant $\ep$, we are able to close all of the gaps discussed for the regret minimization and equilibrium computation problems in Section \ref{sec:prior-work}. %
We begin with the case that the learner has $N$ actions, also known as \emph{learning with expert advice}. By applying Theorem \ref{thm:swap-to-external} with action set $\MX = [N]$, and reward class given by all $[0,1]$-bounded functions, i.e., $\mathcal{F} = [0,1]^{[N]}$, we obtain:
\begin{corollary}[Upper bound for finite action swap regret; informal version of Corollary \ref{cor:treeswap-finiten}]
  \label{thm:intro-main}
  Fix $N \in \mathbb{N}$ and $\ep \in (0,1)$, and consider the setting of online learning with $N$ actions. Then for any $T$ satisfying $T \geq (\log(N)/\ep^2)^{\Omega(1/\ep)}$, there is an algorithm that, when faced with any adaptive adversary, has swap regret bounded above by $\ep$. Further, the amortized per-iteration runtime of the algorithm is $O(N)$, its worst-iteration runtime is $O(N/\epsilon)$ and its space-complexity is $O(N/\epsilon)$. %
\end{corollary}
In the regime of constant $\ep$, Corollary \ref{thm:intro-main} improves on the previously best-known complexity of $T \ge \tilde\Omega(N/\epsilon^2)$, providing an exponential improvement in the dependence on $N$. 
We note that $N/\epsilon^2$ is tight for all $\epsilon$ in the  \emph{non-distributional} setting, where the learner is allowed to randomize over her actions but has to play a concrete action rather than a probability distribution \cite{Blum07:From}. Thus, Theorem \ref{thm:intro-main} shows that for a constant $\epsilon$, a distributional swap regret of at most $\epsilon$ can be achieved with exponentially fewer rounds.
Another advantage of our result is an improved total runtime of $\tilde O(N)$ for constant $\ep$, compared to the previous $\Omega(N^3)$ runtime of \cite{Blum07:From}, which answers an open question from that paper for constant $\ep$.

Next, we apply Theorem \ref{thm:swap-to-external} to an arbitrary function class $\MF \subset \{0,1\}^\MX$ whose dual has finite Littlestone dimension.  That is, the class of functions indexed by actions of the learner, which, via slight abuse of notation, we denote by $ \MX := \{ f \mapsto f(s) \ : \ s \in \mathcal{X}\} \subset \{0,1\}^{\MF}$, has finite Littlestone dimension.
\footnote{Technically, in order to ensure convexity of $\MF$, we need to apply Theorem \ref{thm:swap-to-external} to the \emph{convex hull} of $\MF$. Doing so does not materially affect the guarantees; see \cref{sec:infinite-games} for a more detailed discussion.}  %
\begin{corollary}[Swap regret for Littlestone classes; informal version of Corollaries \ref{cor:ldim-ce} and \ref{cor:ce-existence}]
  \label{cor:ldim-infinite}
  If the class $\MX$ has Littlestone dimension at most $L$, then for any $T \geq (L/\ep^2)^{\Omega(1/\ep)}$, there is a learner whose swap regret is at most $\ep$. In particular, games with finite Littlestone dimension admit no-swap regret learners and thus have $\ep$-approximate CE for all $\ep > 0$. 
\end{corollary}
We remark that even the \emph{existence} of approximate CEs in  games of finite Littlestone dimension was previously unknown. 
We refer the reader to Section \ref{app:dimensions} for a definition of Littlestone dimension and its real-valued generalizations. %

Finally, we prove an upper bound on the swap regret in the bandit setting that is tight up to $\mathrm{poly} \log N$ factors when $\ep = O(1)$. While the result is not a direct consequence of \cref{thm:swap-to-external}, the overall structure of the algorithm and analysis are similar:
\begin{theorem}[Bandit swap regret; Informal version of \cref{thm:bandit-tree-swap}]
  \label{thm:bandit-ts-informal}
Let $N \in \mathbb{N}, \ep \in (0,1)$ be given, and consider any $T \geq N \cdot (\log(N)/\ep)^{O(1/\ep)}$. Then there is an algorithm in the  \emph{adversarial bandit} setting with $N$ actions  ($\BanditTreeSwap$; \cref{alg:bandit-treeswap}) which achieves swap regret bounded above by $\ep$ after $T$ iterations. 
\end{theorem}
Concretely, for $\ep = O(1)$, \cref{thm:bandit-ts-informal} guarantees that $\tilde O(N)$ rounds suffice to achieve swap regret of at most $\ep$. Interestingly, this implies that, for obtaining swap regret bounded by $\ep = O(1)$, there is only a polylogarithmic gap between the number of rounds needed in the adversarial bandit setting and the full-information non-distributional setting \cite{Blum07:From}. This is in stark contrast to the situation for \emph{external regret}, for which there is an \emph{exponential} gap between the full-information non-distributional setting (where $O(\log N)$ rounds suffice) and the adversarial bandit setting (where $\Omega(N)$ rounds are needed) \cite{lattimore2020bandit}. 
Finally, we remark that our algorithm for the bandit setting is readily seen to be computationally efficient.

\paragraph{Applications: equilibrium computation.}   Next, we discuss implications of Corollary \ref{thm:intro-main} for equilibrium computation. By considering the setting where players in a normal-form game run (a slight variant of) the algorithm of Corollary \ref{thm:intro-main}, we may obtain low query and communication protocols for learning in normal-form games.
  \begin{corollary}[Query and communication complexity upper bound; informal version of Corollaries \ref{cor:commce} and \ref{cor:queryce}]
In normal-form games with a constant number of players and $N$ actions per player, the communication complexity of computing an $\ep$-approximate CE is $\log(N)^{\tilde O(1/\ep)}$ and the query complexity of computing an $\ep$-approximate CE is $N \cdot \log(N)^{\tilde O(1/\ep)}$. 
\end{corollary}
\noah{todo add some refs to prior work}

Finally, we remark that our main reduction can be used to obtain efficient algorithms for computing $\ep$-CE when $N$ is exponentially large if there are nevertheless efficient external regret algorithms.
This is the case in particular for the setting of extensive form games \cite{farina2022kernelized,Kroer20:Faster,farina2021better}: %
\begin{corollary}[Extensive form games; informal version of \cref{thm:efg-formal}]
  \label{cor:efg}
  For any constant $\ep$, there is an algorithm which computes an $\epsilon$-approximate CE of any given extensive form game, %
  with runtime polynomial in the representation of the game (i.e., polynomial in the number of nodes in the game tree and in the number of outgoing edges per node).
\end{corollary}
\cref{cor:efg} is an immediate consequence of \cref{thm:swap-to-external} (i.e., \cref{thm:treeswap}) and the fact that there are efficient external regret minimization algorithms in extensive-form games. This is classically known as a consequence of the \emph{counterfactual regret minimization} algorithm, i.e., Theorem 4 of \cite{zinkevich2008regret}, or improved recent results, such as Theorem 5.5 of \cite{farina2022kernelized}, as well as \cite{Celli19:Learning,farina2021better,Kroer20:Faster}.

\paragraph{Near-matching lower bounds.} 
Theorem \ref{thm:swap-to-external} and Corollary~\ref{thm:intro-main} require the number of rounds $T$ to be exponential in $1/\ep$, where $\ep$ denotes the desired swap regret.  The following lower bound shows this dependence is necessary, even facing an \emph{oblivious} adversary that is constrained to choose reward vectors with \emph{constant $\ell_1$ norm}: 
\begin{theorem}[Lower bound for swap regret with oblivious adversary; restatement of Corollary \ref{cor:lower-main}]
  \label{thm:lower-intro}
  Fix $N \in \mathbb{N}$, $\ep \in (0,1)$, and let $T$ be any number of rounds satisfying \noah{reworded, can we indeed have $T$ be arbitrary so that the below holds?} \fish{I checked the math again and it indeed holds.  Is the concern that $T$ would have to be a perfect power of 2? Or that the first term is independent of $N$?  The two cases that lead to this min are $T<N$ and $T>N$ respectively.}\noah{thanks, I just wanted to make sure that the statement matches the proof, no particular concern}
  \begin{align}
    T \leq O(1) \cdot \min \left\{\exp(O(\ep^{-1/6})), \frac{N}{\log^{12}(N) \cdot \ep^2} \right\}\label{eq:lower-bound-tmin}.
  \end{align}
  Then, there exists an \emph{oblivious} adversary on the function class $\MF = \set{\vf \in [0,1]^N \middle| \norm{\vf}_1 \leq 1}$ such that any learning algorithm run over $T$ steps will incur swap regret at least $\epsilon$.\footnote{If we replace the requirement that $\norm{\bff}_1 \le 1$ with $\norm{\bff}_\infty \le 1$, then the bounds are slightly improved, with $1/6$ replaced with $1/5$ and $12$ with $10$.}
\end{theorem}
Theorem~\ref{thm:lower-intro} establishes:
\begin{itemize}
    \item The first $\tilde\Omega\p{\min(1,\sqrt{N/T})}$ swap regret lower bound for \emph{distributional swap regret}. 
    \item The first $\tilde\Omega\p{\min(1,\sqrt{N/T})}$ swap regret lower bound achieved by an oblivious adversary. In particular,  the adversary samples  reward functions $\vf\^{1:T}$ from some fixed distribution before the first round of learning, independently of the actions of the learner.  Moreover, this distribution is independent of the description of the learning algorithm. 
    \item The first $\tilde\Omega\p{\min(1,\sqrt{N/T})}$ swap regret lower bound from an adversary that plays distributions over a function class of  \emph{constant} Littlestone dimension (namely, the class of point functions on $[N]$, which has Littlestone dimension 1). \yuval{changed from fat to littlestone, you can change back if you prefer. Honestly, we can remove} \noah{looks good, modified slightly} %
    \end{itemize}
    Finally, while the lower bound of $\exp(\ep^{-1/6})$ rounds necessary (to ensure swap regret is bounded by $\ep$) from \cref{thm:lower-intro} does not quite match the upper bound of $\exp(\ep^{-1})$ (from \cref{thm:intro-main}; ignoring $\log N$ factors), we can improve the lower bound somewhat if we allow the adversary to be adaptive. In particular, in \cref{thm:adaptive-lb}, we give an \emph{entirely different} (and somewhat simpler) construction which shows that $T \geq \exp(\Omega(\ep^{-3}))$ rounds are necessary to ensure that swap regret is bounded above by $\ep$.  
    
    \todo{Does this seemingly contradict our upper bound to the reader?}
    \unfinished
    {\noah{commenting this out for now b/c I think we should add more details}
Thanks to its generality, this result also yields a lower bound for the \emph{polytope swap regret} from \cite{mansour2022strategizing}, which was used for learning in strategic environments (while our algorithm implies improved upper bounds for this notion).}

\noah{discuss BM distributional lower bound}

\paragraph{Concurrent work.}
We have been recently made aware of concurrent work by Peng and Rubinstein \cite{peng2023fast}, which proves similar upper and lower bounds to \cref{thm:swap-to-external,thm:lower-intro}. Moreover, they derive a similar set of applications for equilibrium computation problems. %

\subsection{Proof sketch of the upper bound (\cref{thm:swap-to-external})}
We overview the proof of Theorem \ref{thm:swap-to-external}. Recall that we are given $M,d \in \mathbb{N}$, and will construct a swap regret learner for $T = M^d$ rounds. We assume access to a no-external regret learner ($\Algexternal$) that, over $M$ rounds, produces a sequence of distributions which has an external regret of at most $\epsilon \in (0,1)$. We will show that there is an algorithm $\TreeSwap$ with swap regret at most $O(\epsilon + \frac 1d)$. 

 $\TreeSwap$ is defined formally in Algorithm \ref{alg:treeswap}. The algorithm simulates multiple instances of $\Algexternal$ at \emph{levels} $i = 0, 1, \ldots, d-1$, which are arranged as the nodes a depth-$d$ $M$-ary tree. We traverse the $T = M^d$ leaves of the tree in order, one per round. At each round $t$, the $\TreeSwap$ algorithm outputs the uniform mixture over the $d$ distributions produced by the $\Algexternal$ instances on the root-to-leaf path for the current leaf. %

\paragraph{Updating $\Algexternal$ instances.} Next we describe how the $\Algexternal$ instances at each node of the tree are updated over the course of the $T$ rounds. 
 Notice that the $M^i$ instances of $\Algexternal$ at each level $i$ are used during a disjoint set of $M^{d-i}$ consecutive rounds: the first algorithm in level $i$ is used during rounds $1,\dots,M^{d-i}$, the second during rounds $M^{d-i}+1,\dots,2M^{d-i}$, and so on. Each of these $\Algexternal$ instances will be run in a \emph{lazy} fashion, only producing $M$ different distributions over the corresponding $M^{d-i}$ rounds.  The first algorithm in level $i$ will be called to produce a distribution at round $1$, and then play that distribution repeatedly for rounds $1,\dots,M^{d-i-1}$.  At round $M^{d-i-1}$, we finally update the state of the algorithm based on the average reward over the previous $M^{d-i-1}$ rounds.  The algorithm then produces a new distribution, which it plays for rounds $M^{d-i-1}+1,\dots,2M^{d-i-1}$, and so on.  All algorithms in level $i$ will be run in this way: updating every $M^{d-i-1}$ rounds on an \emph{average reward function} from the previous $M^{d-i-1}$ rounds. 
 According to the guarantee of our external regret algorithm, each of these instances will have external regret bounded above by $\epsilon$ relative to the $M$ distributions it produces and the $M$ average reward functions on which it updates. %

\paragraph{Swap regret bound.} To bound the swap regret of our algorithm, let us first denote by $R_i$ the average reward of all the algorithms $\Algexternal$ in level $i$ over all $T$ rounds. Further, for each $i=0,\dots,d$, we define $S_i$ in the following manner. For each block of rounds of size $M^{d-i}$, consider the average reward of the best fixed action in hindsight; then we define $S_i$ to be the  average of these best-in-hindsight rewards over all blocks at level $i$. By the external regret guarantee of $\Algexternal$, we know that $S_i - R_i \le \epsilon$. This is due to the fact that each level-$i$ algorithm is run during a block of $M^{d-i}$ rounds, and therefore competes with the best fixed action over that block of rounds.  Moreover, the contribution to the swap regret of $\TreeSwap$ from all algorithms at level $i$ is at most $S_{i+1} - R_i$. This is due to the fact that these level-$i$ algorithms repeatedly play actions for blocks of $M^{d-i-1}$ rounds, and so the best swaps of these actions correspond to the best fixed actions over the blocks of that length.  The total swap regret is then bounded by
\[
\frac{1}{d}\sum_{i=0}^{d-1} (S_{i+1} - R_i)
= \frac{1}{d}\sum_{i=0}^{d-1} (S_i - R_i) + \frac{S_{d} - S_0}{d}
\le \epsilon + \frac{1}{d},
\]
where we used that $S_i-R_i \le \epsilon$ and that $S_d - S_0 \le 1$ since the utilities are bounded between $0$ and $1$. This concludes the proof.

\subsection{Proof sketch for the lower bound (\cref{thm:lower-intro})}
To prove \cref{thm:lower-intro}, we consider two cases depending on the values of $N,T$ (which correspond to which of the terms on the right-hand side of \cref{eq:lower-bound-tmin} is larger): 
\paragraph{Case 1: $N \geq 4T$.}
As a warm-up, we present a strategy for the adversary that does not quite work.  Then, we show how to fix it, describing a true strategy that achieves the desired lower bound.
In both the warm-up and true strategies, we will consider an adversary that selects ``point function'' rewards at each time step $t$: one action $u\^t \in [N]$ will receive a reward of $1$, and all other actions $0$. To describe these strategies, we will relate the actions to vertices in a full binary tree.  Assume that $T=2^{D}$ for some $D \in \mathbb{N}$. Consider a full binary tree of depth $D$, containing $2^{D+1}-1$ vertices, and denote its vertex set by $V$. In our warm-up construction, each vertex will correspond to a single-action.\footnote{Eventually, we will consider a construction wherein each vertex corresponds to two actions.} That is, in each round $t$, the learner plays a vertex $v\^t \in V$ and the adversary plays a vertex $u\^t \in V$. The reward of the learner is $\1[v\^t=u\^t]$.
While our lower bound is valid also for the distributional setting, we analyze for simplicity the case where the learner has to play a concrete action in each round. However, the same proof goes through if they are allowed to output a distribution over vertices. Here is the strategy of the adversary: let us order the children of each internal node by `left' and `right'. This will create an ordering over the root-to-leaf paths in the tree: the first path goes left until reaching the leaf, the second path goes left except for the last step that is taken right, etc. Enumerate the paths by indices in $[T]$ according to this ordering, where path $t$ is called $P_t$. For each time step $t$, the adversary will select at random a vertex, out of the $D+1$ vertices in path $P_t$, according to the following distribution: the probability of the vertex at depth $i$ is $(i+1)/(1 + 2 + \cdots + (D+1))$. The important property here, is that vertices get higher weight as we go down the tree. 

Let us analyze the swap regret of any learner facing this adversary. Recall that this approach does not quite work for the adversary, but we will show how to fix it. At a high level, the goal of the adversary strategy is to increase swap regret every round $t$ as follows.  
\begin{itemize}
    \item If the learner plays an internal node on $P_t$, they will incur swap regret to the node at depth one greater on $P_t$, which gets slightly more expected reward.
    \item If the learner plays the leaf node of $P_t$, there is a constant probability that the adversary will not play this leaf. In this case, the learner will incur a swap regret from that leaf.
    \item If the learner plays a node not on $P_t$, they will receive no reward and incur swap regret.
\end{itemize}

However, the problem is that the learner will later have a chance to undo this incurred swap regret.  For an internal node $v$, let $[\tl_v,\tr_v]$ be the interval of time steps for which $v$ is on $P_t$.  Let's say the learner plays $v$ heavily during the first half of these time steps $[\tl_v,(\tl_v+\tr_v)/2]$: the times in which the left child of $v$ is present on $P_t$.  The learner will incur swap regret from $v$ to its left child.  However, let's say the learner continues to play $v$ for much of the second half of the interval $[(\tl_v+\tr_v)/2, \tr_v]$.  During these times, the adversary never plays the left child of $v$, while continuing to play $v$ with some probability, undoing the swap regret of $v$.

To account for this, the adversary actually plays the following ``true'' strategy instead.  In this strategy, each node is associated with two actions: $v, \dot{v}$.  During the first half $[\tl_v,(\tl_v+\tr_v)/2]$, as before, the adversary will choose $v$ with probability $(\text{depth}(v)+1)/(1 + \cdots + (D+1))$.  However, with probability $1/2$, the adversary replaces $v$ with $\dot{v}$ at the halfway point $t=(\tl_v+\tr_v)/2$. That means, with probability $1/2$, for the second half $[(\tl_v+\tr_v)/2, \tr_v]$, the adversary will choose $\dot{v}$ with probability $(\text{depth}(v)+1)/(1 + \cdots + (D+1))$ and never choose $v$.  On the other hand, with probability $1/2$ there is no replacement, and the adversary continues to select $v$ with probability $(\text{depth}(v)+1)/(1 + \cdots + (D+1))$, not $\dot{v}$.  This accomplishes the following. With probability $1/4$, $v$ will get replaced at the halfway point $t=(\tl_v+\tr_v)/2$ but the left child of $v$ will {\it not} get replaced at its halfway point $t=(3/4)\tl_v+(1/4)\tr_v$.
In this event, which happens with constant probability, both $v$ and its left child will be played with non-zero probability over the interval $[\tl_v,(\tl_v+\tr_v)/2]$ and at no other time steps.  Thus, the learner will not have a chance to undo the swap regret.

The formal version of this argument, presented in Section \ref{sec:oblivious-lb}, breaks into case work.  We lower bound the swap regret of action $v$ by considering the reward of swapping $v$ with the best of the 4 actions associated with its 2 children.  In addition, we consider a swap from $v$ to the root of the tree, in the event that $v$ is played on many rounds outside of the interval $[\tl_v,\tr_v]$.  Bounds for each case culminate in the following.  Letting $X_v$ be the total number of rounds the learner plays action $v$, we show that the best swap of action $v$ increases expected total reward by $\tilde{\Omega}(X_v)$.  Thus, the total swap regret of the learner would be $\tilde{\Omega}\p{\sum_v X_v} = \tilde{\Omega}\p{T}$, and her average swap regret would be $\tilde{\Omega}\p{1} = \Omega\p{\frac{1}{\text{polylog}(T)}}$, which is at least $\ep$ for $T \leq \exp(\text{poly }1/\ep)$. 
\paragraph{Case 2:  $N < 4T$.} This case is very similar to the first. In fact, in Section \ref{sec:oblivious-lb}, we define the adversary strategy in a general way that avoids breaking into cases manually.  The key difference in this case is that we don't have enough actions to associate 2 actions with each node of a full binary tree with $T$ leaves.  In this case, we consider a full binary tree with $N/4$ leaves. We again have an adversary that iterates through the root-to-leaf paths in DFS order.  In this case though, each iteration corresponds to a batch in which the adversary plays a distribution over that root-to-leaf path repeatedly for $4T/N$ time steps.

The other key difference here is that we need to associate each leaf with two actions $\ell,\dot{\ell}$. As discussed before, there is a single coin flip for each of the internal nodes $v$ that determines if it gets replaced at time $t=(\tl_v+\tr_v)/2$.  On the other hand, for leaf nodes $\ell$, we have a coin flip at every single time step in its batch, determining which of $\ell,\dot{\ell}$ will be played with non-zero probability.  Letting $X_\ell$ be the total number of rounds the learner plays $\ell$, due to the random deviation in the selection of the adversary, the expected swap regret to $\dot{\ell}$ in $\tilde\Omega\p{\sqrt{X_\ell}}$.  Thus, the total swap regret of a learner that plays only leaf actions over all $N/4$ batches will be $\tilde\Omega\p{\sum_\ell \sqrt{T/N}} = \tilde\Omega\p{\sqrt{NT}}$ and her average swap regret will be $\tilde\Omega\p{\sqrt{N/T}}$, as desired.

\subsection{Discussion}
\label{sec:discussion}
\todo{adjust using adaptive lower bound, which gets better lower bound}
We next compare the guarantees of Corollary~\ref{thm:intro-main}, \cref{thm:lower-intro}, and \cref{thm:adaptive-lb} (recall that \cref{thm:adaptive-lb} yields a quantitatively stronger lower bound than \cref{thm:lower-intro} with the stronger notion of \emph{adaptive} adversary).  Let $\mathscr{M}(N, \ep)$ denote the smallest $T_0 \in \mathbb{N}$ so that, for all $T \geq T_0$, there is a learning algorithm whose action set is $[N]$ and for which the swap regret over $T$ rounds is bounded above by $\ep$. Then by \cref{thm:intro-main,thm:lower-intro,thm:adaptive-lb},\footnote{The $\frac{\log N}{\epsilon^2}$ term in the lower bound of (\ref{eq:minimax-swap}) comes the classic external regret lower bound. The second term in the minimum of the upper bound of (\ref{eq:minimax-swap}) comes from the Blum-Mansour algorithm~\cite{Blum07:From}.}
\begin{align}
&\Omega(1) \cdot \min \left\{  \frac{\log N}{\ep^2}  + 2^{\Omega(\ep^{-1/3})}, \frac{N}{\log^{10}(N) \cdot \ep^2} \right\}\nonumber\\
&\leq \mathscr{M}(N, \ep) \leq O(1) \cdot \min \left\{ (\log(N)/\ep^2)^{O(1/\ep)}, \frac{N \log N}{\ep^2} \right\}\label{eq:minimax-swap}.
\end{align}
The second terms in the upper and lower bounds in (\ref{eq:minimax-swap}) differ by a $\mathrm{poly} \log(N)$ factor, which is insignificant compared to $N$. The first terms differ in that (a) the $\log(N)$ is in the base of the exponent in the upper bound but not the lower bound, and (b) the exponent in the lower bound is $\ep^{-1/3}$ but is $\ep^{-1}$ in the upper bound. We remark that the term $2^{\Omega(\ep^{-1/3})}$ in the lower bound comes from \cref{thm:adaptive-lb}, which is stronger than the bound of $2^{\Omega(\ep^{-1/6})}$ from \cref{thm:lower-intro}.  %

The swap regret bound of Corollary \ref{thm:intro-main} improves upon those of of Stoltz-Lugosi and of Blum-Mansour~\cite{Stoltz05:Internal,Blum07:From} when the accuracy parameter $\ep$ and number of actions $N$ satisfy $\ep \gg \frac{\log \log N}{\log N}$. \todo{is this inequality the wrong way?}\noah{no should be fine?}
In particular, for $\ep = O(1)$, our reduction bounds swap regret above by $\ep$ via an efficient algorithm with $T \leq \mathrm{poly}\log N$ rounds, whereas~\cite{Stoltz05:Internal,Blum07:From} require $T \geq \tilde\Omega(N)$.

\paragraph{Outline of the paper.} After stating preliminaries in \cref{sec:prelim}, we  prove our main reduction from swap to external regret (\cref{thm:swap-to-external}) in \cref{sec:reduction}. Then, in \cref{sec:nexperts,sec:infinite-games,sec:qc-cc,sec:bandit-swap}, we provide applications of this reduction. In \cref{sec:oblivious-lb} we prove our main lower bound (\cref{thm:lower-intro}); part of the proof is deferred to \cref{app:lower-bound}. Finally, in \cref{sec:adaptive-lb}, we prove our alternative adaptive lower bound with superior rates.

\section{Preliminaries}
\label{sec:prelim}

\subsection{Online Learning}
\label{sec:prelim-online}

The setting of online learning entails a repeated interaction between a learner and adversary over $T$ rounds.  For each time step, $t \in [T]$, the learner, whose action space is denoted by $\MX$, selects a distribution $\vxt \in \Delta_\MX$, and the adversary responds with a reward function $\vect{f}\^t \in \mathcal{F}$ where $\mathcal{F} \subseteq [0,1]^\mathcal{X}$.  The learner receives reward $\vect{f}\^t(\vxt) := \E_{s \sim \vxt}[\mathbf{f}\^t[s]]$. (For $f \in \MF$, $s \in \mathcal{S}$, and $\bx \in \Delta_\MX$, we use square brackets to denote $f[s]$ and parentheses to denote $f(\bx)$.) To avoid measurability issues, we assume that $\MX$ is countable and equipped with the discrete sigma algebra.

As an example, consider the classical ``experts'' setting, in which  the learner selects a distribution $\bx\^t \in \Delta_N$ over $N$ actions and the adversary selects an arbitrary reward vector $\mathbf{f}\^t \in [0,1]^N$. Here we have $\mathcal{X} = [N]$ and $\MF = [0,1]^N$. The learner's reward at each round $t$ is given by $\langle \mathbf{f}\^t, \bx\^t \rangle$. %

The learner's performance over the $T$ rounds of learning is typically evaluated via ``regret'': a comparison between the total reward obtained by the learner and some benchmark.  There are several notions of regret, depending on the particular benchmark. 
The \emph{external regret} of the learner is the difference between her total reward and the best reward obtained by a fixed action $s^\st \in \MX$ in hindsight:
\begin{definition}\label{def:ER}
  Given sequences $\bx\^{1:T} = (\bx\^1, \ldots, \bx\^T)$ and $\bff\^{1:T} = (\bff\^1, \ldots, \bff\^T)$ of play by the learner and adversary, the \emph{external regret} corresponding to these sequences is
    \begin{equation}\label{eq:ER}
        \ExtRegret\p{\vx\^{1:T},\vect{f}\^{1:T}} := \sup_{s^\st \in \MX} \frac{1}{T} \sum_{t=1}^T \p{\vft[s^\st] - \vft(\vxt)}.
    \end{equation}
\end{definition}
We say that a learner has regret $\epsilon$ if her regret against any adversary is bounded above by $\epsilon$ after $T$ time steps of learning. The \emph{distributional swap regret} of a learner is the difference between her total reward and the reward obtained by swapping each of the learner's played actions with the best action that could have been played in its place. \noah{slightly changed the def here so we have to avoid thinking about what a linear function is on an infinite-dimensional space} %
\begin{definition}\label{def:SR}
    Given sequences $\bx\^{1:T} = (\bx\^1, \ldots, \bx\^T)$ and $\bff\^{1:T} = (\bff\^1, \ldots, \bff\^T)$ of play by the learner and adversary, the \emph{swap regret} corresponding to these sequences is
    {\upshape
    \begin{align}
        \SR\p{\vx\^{1:T},\bff\^{1:T}} &= \sup_{\pi : \MX \to \MX} \frac 1T \sum_{i \in \MX} \sum_{t=1}^T \bx\^t[i] \cdot \left( \bff\^t[\pi(i)] - \bff\^t[i] \right) \nonumber\\
        &= \frac{1}{T}\sum_{i \in \MX} \sup_{j \in \MX} \p{\sum_{t=1}^T \vxti \p{\vft[j]-\vft[i]}}\label{eq:SR}
    \end{align}
    }
  \end{definition}
  Notice that in the case that $\MX$ is finite, we may also write the swap regret as
  \[
  \SR\p{\vx\^{1:T},\bff\^{1:T}} =   
    \max_{\substack{\phi:\Delta_\MX \to \Delta_\MX\\ \text{$\phi$ linear}}} \frac{1}{T} \sum_{t=1}^T \vft\p{\phi(\vxt)} - \vft\p{\vxt}.
  \]
  We note that \emph{distributional} swap-regret relates to a setting in which the learner is allowed to play a distribution over actions. We remark that, for learning with $N$ experts, \cite{Blum07:From} established a lower bound of $\Omega(\sqrt{NT})$ in the setting where the learner must play a concrete action at each time step (possibly in a probabilistic way).  Not constraining the learner in this way, we are able to break this lower bound. 

  When the sequences $\bx\^{1:T}, \bff\^{1:T}$ are understood from context, we will at times abbreviate $\ExtRegret(\bx\^{1:T}, \bff\^{1:T})$ by $\ExtRegret(T)$ and $\SwapRegret(\bx\^{1:T}, \bff\^{1:T})$ by $\SwapRegret(T)$.
\noah{For technical reasons, we assume throughout the paper that $\mathcal{F}$ is closed under convex combinations, namely, for any $f_1,\dots,f_m \in \mathcal{F}$, $\frac{1}{m} \sum_{i=1}^m f_i \in \mathcal{F}$.}\noah{do we want to do this or assume existence of ext regret learning algorithms for $\Delta_\MF$?} \fish{I vote we assume convex combos. We don't wanna think of the standard $\mathcal{F}=[0,1]^N$ as $\Delta_{\set{0,1}^N}$}

\subsection{Games}

An $m$-player \emph{normal-form game} is a pair $(S,A)$ where $S = S_1 \times \cdots \times S_m$ and $A=(A_1,\cdots,A_m)$ with each $A_j: S \to \RR$.  Each $S_j$ is the set of actions (i.e., pure strategies) available to player $j$ and each $A_j$ is the reward function, or payoff matrix, of player $j$, which maps the set of action profiles $S$ to the real numbers.  Each player $j \in [m]$ selects a distribution over actions $\vx_j \in \Delta_{S_j}$ with the goal of maximizing their reward $\EE_{s \sim \prod_j \vx_j} \ps{A_j(s_1,\cdots,s_m)}$.  We are interested in the following equilibrium notions: %

\begin{definition}\label{def:CCE}
    An \emph{$\epsilon$-coarse correlated equilibrium (CCE)} is a distribution $\mu \in \Delta_S$ so that, for every player $j$ and every deviation $d_j \in S_j$,
    \begin{equation}\label{eq:CCE}
        \EE_{s \sim \mu}\ps{A_j(d_j,s_{-j})}  \leq \EE_{s \sim \mu}\ps{A_j(s_j,s_{-j})} +\epsilon.
    \end{equation}
\end{definition}

\begin{definition}\label{def:CE}
    An \emph{$\epsilon$-correlated equilibrium (CE)} is a distribution $\mu \in \Delta_S$ so that, for every player $j$ and every deviation map $\pi_j: S_j \to S_j$
    \begin{equation}\label{eq:CE}
        \EE_{s \sim \mu}\ps{A_j(\pi(s_j),s_{-j})}  \leq \EE_{s \sim \mu}\ps{A_j(s_j,s_{-j})} +\epsilon.
    \end{equation}
  \end{definition}

  Next, we define the notion of a sparse (C)CE:
  \begin{definition}[Sparse equilibrium]\label{def:sparse-eq}
Let $k \in \mathbb{N}$. An $\ep$-CE (or $\ep$-CCE) $\mu \in \Delta_S$ is \emph{$k$-sparse} if $\mu$ is can be written as a mixture over at most $k$ product distributions, i.e., distributions in $\Delta(S_1) \times \cdots \times \Delta(S_m)$. %
  \end{definition}
Typically we will consider the case that $S_j = [n]$ for each $j \in [m]$. Such a normal-form game is defined by its payoff matrices $A_j : [n]^m \to \BR$ for $j \in [m]$. 
  
  \subsection{Alternate models of computation}
  First, we consider the \emph{query complexity} model of computation. Here, a normal-form game with payoff matrices $(A_1, \ldots, A_m)$ is fixed, but is \emph{unknown} to the learning algorithm. The algorithm is allowed to make adaptive randomized queries to single elements of $A_1, \ldots, A_m$.
  \begin{definition}[Query complexity]\label{def:query-c}
Given a confidence level $\delta \in (0,1)$, the \emph{query complexity} of an equilibrium concept (e.g., $\ep$-CCE or $\ep$-CE) is the minimal $Q$ such that there exists an algorithm which, on any input, outputs the specified notion of equilibrium with probability at least $1-\delta$, using only $Q$ queries.
\end{definition}

Next, we consider the \emph{communication complexity} model of computation. Here, each of the $m$ players $j \in [m]$ is given its own payoff matrix $A_j$, but does not know the payoff matrices of other agents. The agents are allowed to communicate in an adaptive randomized manner, according to some \emph{communication protocol}, over an arbitrary number of rounds.
\begin{definition}[Communication complexity]
  \label{def:comm-c}
Given a confidence level $\delta \in (0,1)$, the \emph{communication complexity} of an equilibrium concept (e.g., $\ep$-CCE or $\ep$-CE) is the minimal $C$ such that there exists a communication protocol which, on any input, exchanges at most $C$ bits of communication between the players and terminates with each player having agreed upon the \emph{same} distribution satisfying the equilibrium concept. 
\end{definition}

Typically, the scaling of the query and communication complexities of equilibrium concepts with respect to $\delta$ is  $\log^{O(1)}(1/\delta)$; thus we often omit the parameter $\delta$ in our discussions.

\paragraph{Miscellaneous notation.} We use the letter $C$ to denote absolute constants in our proofs. To avoid cluttering notation, we use the convention that the value of $C$ may change from line to line.

\section{A new reduction from no-swap regret to no-external regret}
\label{sec:reduction}
In this section, we prove our main upper bound (\cref{thm:swap-to-external}, formally stated as \cref{thm:treeswap} below), which gives a reduction from no-swap regret learning to no-external regret learning with no dependence on the number of the learner's actions. Following, we discuss several applications, including faster rates for swap regret in the experts setting with many experts (\cref{sec:nexperts}) and for infinite function classes (\cref{sec:infinite-games}), upper bounds on the query and communication complexities of computing a correlated equilibrium in normal-form games (\cref{sec:qc-cc}), and a nearly-tight bandit $\ep$-swap regret algorithm for constant $\ep$ (\cref{sec:bandit-swap}). 

Let $\MX$ be a set representing the set of actions of a learning algorithm, and let $\MF \subset [0,1]^\MX$ denote a class of utility functions for the learning algorithm, which is closed under convex combinations: for all $\bff_1,\bff_2 \in \MF$ and all $\lambda \in [0,1]$, $\lambda \bff_1 + (1-\lambda)\bff_2 \in \MF$. We assume that $\MF$ admits a no-external regret learning algorithm:%
\begin{assumption}[No-external regret algorithm]
  \label{asm:no-external-reg}
  For any $T \in \mathbb{N}$, there is an algorithm $\Alg = \Alg(T)$, together with functions $\Algact, \Algupdate$,  which perform the following updates with an adaptive adversary over $T$ rounds. In each round $t \in [T]$:
  \begin{itemize}
  \item  $\Alg$ produces an action $\bx\^t := \Algact()$, where $\bx\^t \in \Delta(\MX)$;
  \item The adversary observes $\bx\^t$ and then chooses a function $\bff\^t \in \MF$;
  \item $\Alg$ updates its internal state according to $\Algupdate(\bff\^t)$. 
  \end{itemize}
  We assume that the external regret of $\Alg$ with respect to the functions $\bff\^t$ is bounded by $R_{\Alg}(T)$:
  \begin{align}
    \ExtRegret\p{\vx\^{1:T},\vect{f}\^{1:T}} \leq R_{\Alg}(T)\nonumber.
  \end{align}
\end{assumption}

\begin{algorithm}
  \caption{\texttt{TreeSwap}$(\MF, \MX, \Alg, T, M, d)$}
	\label{alg:treeswap}
	\begin{algorithmic}[1]%
      \Require Action set $\MX$, utility class $\MF$, no-external regret algorithm $\Alg$, time horizon $T$, parameters $M, d$ with $T \leq M^d$. 
      \State \label{line:treeswap-params} For each sequence $\sigma \in \bigcup_{h=1}^d \{0, 1, \ldots, M-1\}^{h-1}$, initialize an instance of $\Alg$ with time horizon $M$, denoted $\Alg_\sigma$.
      \For{$1 \leq t \leq T$}
      \State Let $\sigma = (\sigma_1, \ldots, \sigma_d)$ denote the base-$M$ representation of $t-1$.\label{line:base-M}
      \For{$1 \leq h \leq d$}
      \If{$\sigma_{h+1} = \cdots = \sigma_d = 0$ or $h=d$}
      \State If $\sigma_h > 0$, %
      call  $\Alg_{\sigma_{1:h-1}}.\mathtt{update}\left(\frac{1}{M^{d-h}} \cdot \sum_{s=t-M^{d-h}}^{t-1} \bff\^s \right)$.\label{line:alg-update}
      \State $\Alg_{\sigma_{1:h-1}}.\mathtt{curAction} \gets \Alg_{\sigma_{1:h-1}}.\mathtt{act()}$.
      \EndIf
      \EndFor
      \State Output the uniform mixture $\vx\^t := 
      \frac{1}{d} \sum_{h=1}^d
       \Alg_{\sigma_{1:h-1}}.\mathtt{curAction}$, and observe $\bff\^t$.\label{line:play-xt}
      \EndFor
    \end{algorithmic}
  \end{algorithm}

  \begin{theorem}
    \label{thm:treeswap}
    Suppose that an action set $\MX$ and a utility function class $\MF$ are given as above, together with an algorithm $\Alg$ satisfying the conditions of Assumption \ref{asm:no-external-reg}. Suppose that $T, M, d \in \mathbb{N}$ are given for which $M \geq 2$ and $M^{d-1} \leq T \leq M^d$. Then given an adversarial sequence $\bff\^1, \ldots, \bff\^T \in \MF$,  $\texttt{TreeSwap}(\MF, \MX, \Alg, T)$ (Algorithm \ref{alg:treeswap}) produces a sequence of iterates $\bx\^1, \ldots, \bx\^T$ satisfying the following swap regret bound:
    \begin{align}
\SwapRegret(\bx\^{1:T}, \bff\^{1:T}) \leq {R_{\Alg}(M)} + \frac{3}{d}.\nonumber
  \end{align}
\end{theorem}
$\TreeSwap$ (\cref{alg:treeswap}) simulates multiple instances of the no-external regret algorithm $\Alg$. These instances of $\Alg$ are arranged in a depth-$d$, $M$-ary tree. For simplicity, suppose that $T = M^d$, so that there is one leaf of this tree for each time step. At each time step $t \in [T]$, consider the leaf of the tree corresponding to $t$. The root-to-leaf path for this leaf may be identified with the base-$M$ representation of $t$, which we denote by $\sigma_{1:d} = (\sigma_1, \ldots, \sigma_d) \in \{0, 1, \ldots, M-1\}^d$ (\cref{line:base-M}). In particular, for $h \in [d]$, $\sigma_h$ indexes the child taken at the $h$th step in this root-to-leaf path. Each node along this root-to-leaf path may therefore be identified with some prefix of $\sigma_{1:d}$, namely $\sigma_{1:h-1}$ for each $h \in [d]$. 

For each $t \in [T]$, the distribution $\bx\^t$ played by $\TreeSwap$ (\cref{line:play-xt}) is the uniform average over the distributions played by the instances of $\Alg$ (denoted by $\Alg_{\sigma_{1:h-1}}$ for $h \in [d]$) at each node along the root-to-leaf path at step $t$. The instances $\Alg_{\sigma_{1:h-1}}$ are updated in a lazy fashion, as follows: every $M^{d-h}$ rounds $t$ when $\sigma_{1:h-1}$ lies on the current root-to-leaf path, the utility functions $\bff\^t$ are averaged and fed to the $\mathtt{update}$ procedure of $\Alg_{\sigma_{1:h-1}}$ (\cref{line:alg-update}). Thus, each instance $\Alg_{\sigma_{1:h-1}}$ is updated a total of $M$ times in the course of $\TreeSwap$.

While the above discussion assumed that $T = M^d$ for simplicity, the proof below considers the setting of general values of $T$ treated by the statement of \cref{thm:treeswap}. In particular, the theorem will typically be applied in the following manner (see \cref{sec:nexperts,sec:infinite-games}): given some value of $T$, we choose $M$ as a function of $T$ and then let $d$ be chosen so as to satisfy $M^{d-1} \leq T \leq M^d$. As long as $T$ is sufficiently large, the resulting value of $3/d$ will be sufficiently small, which will ensure that $\SwapRegret(T)$ is small.

\begin{proof}[Proof of Theorem \ref{thm:treeswap}]
  For each $1 \leq h \leq d$, let $\Sigma_{h-1} := \{ \sigma_{1:h-1} = (\sigma_1, \ldots, \sigma_{h-1}) :\ 1 + \sum_{g=1}^{h-1} \sigma_g \cdot M^{d-g} \leq T \}$ denote the set of sequence prefixes of length $h-1$ encountered over the course of $T$ rounds of $\TreeSwap$. (Note that $\Sigma_{h-1}$ depends on $T$.)  Consider any  sequence $\sigma_{1:h-1} = (\sigma_1, \ldots, \sigma_{h-1}) \in \Sigma_{h-1}$ representing a node in the tree, and define $M'(\sigma_{1:h-1}) := \max \{ \sigma_h \in \{0, \ldots, M-1\} \ : \ \sigma_{1:h} \in \Sigma_h \}$. Then let  $\bx_{\sigma_{1:h-1}}\^0, \ldots, \bx_{\sigma_{1:h-1}}\^{M'(\sigma_{1:h-1})}$ denote the $M'$ actions taken by the algorithm $\Alg_{\sigma_{1:h-1}}$ over the course of the $T$ rounds. Moreover, we let $\btau(\sigma_{1:h-1}) := 1 + \sum_{g=1}^{h-1} \sigma_g \cdot M^{d-g}$ denote the first round when $\sigma_{1:h-1}$ is encountered, and $\etau(\sigma_{1:h-1}) := \max\{T,  \btau(\sigma_{1:h-1}) + (M^{d-h+1}-1)\}$ denote the last round when $\sigma_{1:h-1}$ is encountered.
    
    We may then bound the total (unnormalized) utility of the learner over the $T$ rounds, as follows:
    \begin{align}
 \sum_{t=1}^T\bff\^t(\bx\^t) = \frac{1}{d} \sum_{h=1}^d  \sum_{\sigma_{1:h-1} \in \Sigma_{h-1}} \sum_{\sigma_{h} = 0}^{M'(\sigma_{1:h-1})}    \sum_{s=\btau(\sigma_{1:h})}^{\etau(\sigma_{1:h})} \bff\^s (\bx_{\sigma_{1:h-1}}\^{\sigma_{h}})\label{eq:learner-utility}.
    \end{align}
    Now consider any function $\pi : \MX \to \MX$. We define $f \circ \pi : \MX \to \BR$ to be the function $(f \circ \pi)[s] = f[\pi(s)]$, for $s \in \MX$. Then for $\bx \in \Delta(\MX)$, we have $(f \circ \pi)(\bx) = \E_{s \sim \bx}[f[\pi(s)]]$.   
    The learner's utility under the swap function $\pi$ is given by
    \begin{align}
 \sum_{t=1}^T (\bff\^t \circ \pi)(\bx\^t) =& \frac 1d \sum_{h=1}^d \sum_{\sigma_{1:h-1} \in \Sigma_{h-1}} \sum_{\sigma_{h} = 0}^{M'(\sigma_{1:h-1})} \sum_{s = \btau(\sigma_{1:h})}^{\etau(\sigma_{1:h})} (\bff\^s\circ\pi)(\bx_{\sigma_{1:h-1}}\^{\sigma_{h}}))\nonumber\\
      \leq &  \frac{1}{d} \sum_{h=1}^d \sum_{\sigma_{1:h}\in \Sigma_h} \max_{\bx^\st \in \MX} \sum_{s \in \btau(\sigma_{1:h})}^{\etau(\sigma_{1:h})} \bff\^s(\bx^\st)\label{eq:learner-swap}.
    \end{align}
    Subtracting (\ref{eq:learner-utility}) from (\ref{eq:learner-swap}) and using the fact that $\bff(\bx) \in [0,1]$ for all $\bff,\bx$, we see that
    \begin{align}
      & \sum_{t=1}^T (\bff\^t\circ\pi)(\bx\^t) - \bff\^t(\bx\^t) \nonumber\\
     \leq  &  \frac{1}{d} \sum_{h=2}^{d-1} \sum_{\sigma_{1:h-1}\in \Sigma_{h-1}} \max_{\bx^\st \in \MX} \left( \sum_{s \in\btau(\sigma_{1:h-1})}^{\etau(\sigma_{1:h-1})} \bff\^s(\bx^\st) - \sum_{\sigma_h=0}^{M'(\sigma_{1:h-1})} \sum_{s = \btau(\sigma_{1:h})}^{\etau(\sigma_{1:h})} \bff\^s(\bx_{\sigma_{1:h-1}}\^{\sigma_h})\right) + \frac 1d \sum_{s=1}^T \max_{\bx^\st \in \MX} \bff\^s(\bx\^\st) \nonumber\\
      \leq & \frac 1d \sum_{h=2}^{d-1} \left( M^{d-h+1} +  \sum_{\scriptsize\substack{\sigma_{1:h-1} \in \Sigma_{h-1}:\\ \etau(\sigma_{1:h-1}) - \btau(\sigma_{1:h-1}) = M^{d-h+1}-1}} M \cdot M^{d-h} \cdot R_{\Alg_{\sigma_{1:h-1}}}(M)\right)+  \frac{1}{d} \sum_{s=1}^T \max_{\bx^\st \in \MX} \bff\^s(\bx\^s)\nonumber\\
      \leq & T \cdot R_{\Alg}(M) +   \frac{1}{d} \sum_{h=2}^{d-1} M^{d-h+1} + \frac{T}{d} \leq T \cdot R_{\Alg}(M) + \frac{3T}{d}\label{eq:final-reg-ub},
    \end{align}
    where the second inequality above uses the external regret assumption of $\Alg_{\sigma_{1:h-1}}$ for each possible choice of $\sigma_{1:h-1}$ (Assumption \ref{asm:no-external-reg}), as well as the following observation: for each $2 \leq h \leq d-1$, there is at most a single sequence $\sigma_{1:h-1} \in \Sigma_{h-1}$ for which $\etau(\sigma_{1:h-1}) - \btau(\sigma_{1:h-1}) < M^{d-h+1} - 1$. For such a sequence $\sigma_{1:h-1}$, we may trivially upper bound 
    \begin{align}
\max_{\bx^\st \in \MX} \left( \sum_{s \in\btau(\sigma_{1:h-1})}^{\etau(\sigma_{1:h-1})} \bff\^s(\bx^\st) - \sum_{\sigma_h=0}^{M'(\sigma_{1:h-1})} \sum_{s = \btau(\sigma_{1:h})}^{\etau(\sigma_{1:h})} \bff\^s(\bx_{\sigma_{1:h-1}}\^{\sigma_h})\right)\leq M^{d-h+1}\nonumber.
    \end{align}
    The final inequality in the display (\ref{eq:final-reg-ub}) uses the fact that, since $M \geq 2$ and $T \geq M^{d-1}$ by assumption, $\sum_{h=2}^{d-1} M^{d-h+1} \leq 2M^{d-1} \leq 2T$. 
    
    Dividing the display (\ref{eq:final-reg-ub}) by $T$ gives the desired regret bound.
  \end{proof}

  \subsection{Application: swap regret for $N$ experts}
  \label{sec:nexperts}
  First, we show how \cref{thm:treeswap} can be used to bound the swap regret in the ``finite experts'' setting: the number of rounds required to obtain sublinear swap regret is only polylogarithmic in the number of experts $N$, which improves upon the bound of $\Theta(N \log N)$ rounds from \cite{Blum07:From,Stoltz05:Internal}. 
  \begin{corollary}
    \label{cor:treeswap-finiten}
    Fix $N \in \mathbb{N}$. Then, letting $\MX = [N]$ and $\MF = [0,1]^{[N]}$, for any $\ep \in (0,1)$, there is an algorithm for which the swap regret is bounded above as
$
\SwapRegret(T) \leq \ep$
for any $T$ which satisfies $T \geq (\log(N)/\ep^2)^{O(1/\ep)}$.

Each iteration takes at most $O(N/\epsilon)$ time, and the amortized runtime over the $T$ iterations is $O(NT)$. The total space complexity is $O(N/\epsilon)$.
  \end{corollary}
  \begin{proof}
    Given $\ep \in (0,1)$ and $T$ as in the statement of the corollary, we may choose $M = \lceil\log(N)/\ep^2\rceil$, and $d \geq \lceil 1/\ep \rceil$, so that $M^{d-1} \leq T \leq M^d$.
    
    We then call $\texttt{TreeSwap}(\MF, \MX, \Alg, T, M, d)$, where $\Alg$ is the multiplicative weights algorithm (denoted $\MWU$), which obtains a regret bound of $R_{\Alg}(M) \leq C \sqrt{\log(N)/M}$, for some constant $C$. %
 Then Theorem \ref{thm:treeswap} guarantees that
    \begin{align}
\SwapRegret(T) \leq O \left( \sqrt{\frac{\log N}{\log(N)/\ep^2}} + \frac{1}{d}\right) \leq O(\ep)\nonumber.
    \end{align}
    To get the desired bound of $\epsilon$, we can scale $\ep$ down by a constant factor.

    To bound the time complexity, we need to consider two types of operations:
    \begin{itemize}
        \item \emph{Maintaining the cumulative reward vector:} at any iteration $t \in [T]$ the algorithm would maintain the sum $\sum_{s=1}^t \bff\^s$. Since each $\bff\^s$ is an $N$-dimensional vector, this takes time $O(N)$ per iteration and $O(NT)$ total.
        \item \emph{Number of MWU updates:} altogether, $\TreeSwap$ has $O(T/M)$ instances of $\Alg=\MWU$, each making $M$ updates, so there are $O(T)$ updates to all instances of $\MWU$ in total. We will show that each update takes time at most $O(N)$ below. To do so, note that an update consists of the following operations:
        \begin{itemize}
        \item \emph{Computing the average reward vector that is fed into the $\MWU$ instance (\cref{line:alg-update}):} In order to update each $\MWU$ instance $\MWU_{\sigma_{1:h-1}}$ in $\TreeSwap$, we have to average the reward vectors over the $M^{d-h}$ preceding  iterations. We say that an instance $\MWU_{\sigma_{1:h-1}}$ is \emph{live at round $t$} if the base-$M$ representation of $t-1$ beings with $\sigma_{1:h-1}$ (i.e., if $\sigma_{1:h-1}$ is on the root-to-leaf path corresponding to the current iteration $t$). Note that at each round $t$, there are $d$ live instances of $\MWU$. Since we maintain the current cumulative reward vector at each iteration, it suffices to save in memory, for each \emph{live} execution of the MWU, the cumulative loss up to the time of its last update. %
          Fix any such live instance, and let $t_0$ be the time of its last update: then for this instance, we  store $\sum_{s=1}^{t_0} \bff\^s$.
          Then, the cumulative loss since the last update is just $\sum_{s=1}^t \bff^s - \sum_{s=1}^{t_0} \bff^s$. Since each of these summands is stored in memory as an $N$-dimensional vector, computing the difference takes $O(N)$ time.
            \item \emph{Computing the next action to take in a single $\MWU.\mathtt{update}$ call:} Each update step of $\MWU$ takes time $O(N)$.
            \item \emph{Computing the uniform mixture over the actions suggested by the $d$ live $\MWU$ instances at each round (\cref{line:play-xt}):}  %
              To compute this change, we account for the change to each of the live $\MWU$ instances which is updated for each iteration $t$: for each such update, we need to perform a $O(N)$-time operation. %
        \end{itemize}
        There are at most $d = O(1/\epsilon)$ updates in each iteration, which implies a worst-case runtime of $O(Nd)=O(N/\epsilon)$ and a total runtime of $O(NT)$, since there are at most $O(T)$ updates in total, as argued above.
    \end{itemize} 
    This sums up to an $O(NT)$ runtime total and a maximum of $O(N/\epsilon)$ per-iteration runtime. To bound the space complexity, notice that, for each live execution of an $\MWU$ instance, we need to store the cumulative loss since its last update. There are $d$ such instances, one for each level. Each cumulative loss is an $N$-dimensional vector, which yields a total space of $O(Nd)$. Additionally, we need to store the current distribution over actions for each $\MWU$ instances. This takes space $O(N)$ per $\MWU$ instance and $O(Nd)= O(N/\epsilon)$ space overall. 
  \end{proof}

  \subsection{Application: Swap regret and approximate CE for infinite games}
  \label{sec:infinite-games}
  In abstract function classes, the existence of no-external regret learners depend on various \emph{dimensions} of the class (see Appendix~\ref{app:dimensions}).  In general, for a bounded real-valued function class $\MH$, there exist online learners with regret approaching 0 as $T \to \infty$ if (and only if) its \emph{sequential Rademacher complexity} $\Rad_T(\MH)$ converges to 0 as $T \to \infty$ (\cref{thm:rakhlin-online}). In the special case that $\MH$ is  binary-valued, the condition that $\Rad_T(\MH) \to 0$ is equivalent to finiteness of the  \emph{Littlestone dimension}, $\Ldim(\MH)$, of $\MH$. In the real-valued setting, $\Rad_T(\MH) \to 0$ if and only if the \emph{sequential $\delta$-fat shattering dimension} of $\MH$, $\esfsd{\MH}$, is finite for all $\delta>0$ (see \cref{prop:ldim-sfat}).
  
  Combining \cref{thm:treeswap} and \cref{thm:rakhlin-online}, we can show the existence of a learning algorithm whose swap regret converges to 0 as long as the learner's action set has vanishing sequential Rademacher complexity. Since we denote the learner's action space by $\MX$ and the space of reward functions by $\MF \subset [0,1]^\MX$, the relevant function class is the class $\{ f \mapsto f(x) \ : \ x \in \MX \}$; we denote this class by $\MX$, with slight abuse of notation.
  \begin{corollary}
    \label{cor:rad-swap}
There is a constant $C > 0$ so that the following holds. Suppose that $\MX, \MF$ are given as above, and let $\ep \in (0,1)$ be given. Define $\tau_\MX(\ep) := \inf_{\tau \in \BN} \{ \tau \ : \ \Rad_\tau(\MX) \leq \ep \}$. Then for any $T \geq (\tau_{\ep/C}(\MX))^{C/\ep}$ there is an algorithm for which the swap regret is bounded as $\SwapRegret(T) \leq \ep$. 
\end{corollary}
\begin{proof}
  We use \cref{thm:treeswap} with action set given by $\MX$ and the function class given by the convex hull of $\MF$, denoted $\co(\MF)$. Accordingly, define $\bar\MX \subset [0,1]^{\co(\MF)}$ by $\bar \MX := \{ \bar f \mapsto \bar f (x) \ : \ x \in \MX\}$, where the domain for $\bar \MX$ is $\co(\MF)$. 
  
  \cref{thm:rakhlin-online} gives that \cref{asm:no-external-reg} is satisfied by some external-regret algorithm $\Alg$ with $R_{\Alg}(T) \leq 2 \Rad_T(\bar\MX)$.  As a straightforward consequence of Jensen's inequality and \cref{def:seq-rad}, we have that $\Rad_T(\bar\MX) \leq \Rad_T(\MX)$ (in fact, equality holds). Thus, applyying \cref{thm:treeswap} with $M = \tau_\MX(\ep/C)$ a choice of $d$ satisfying $M^{d-1} \leq T \leq M^d$, we obtain that, as long as $C$ is sufficiently large, $\TreeSwap(\co(\MF), \MX, \Alg, T)$ obtains $\SwapRegret(T) \leq \ep$. 
\end{proof}

By \cref{prop:ldim-sfat}, we obtain the following immediate corollary.
\begin{corollary}
  \label{cor:ldim-ce}
  Consider $\MX, \MF$ as above, and suppose $\ep \in (0,1)$ is given. Then we have the following:
  \begin{itemize}
  \item If $\MF$ is binary-valued and $\Ldim(\MX) = L$, then for all $T$ satisfying $T \geq (L/\ep)^{O(1/\ep)}$, there is an algorithm certifying $\SwapRegret(T) \leq \ep$.
  \item If $\esfsd{\MX} \leq O(\delta^{-p})$ for some $p \in [0,2)$, then letting $I = \int_0^1 \sqrt{\esfsd{\MX}} d\delta$, for $T \geq (I/\ep)^{O(1/\ep)}$, there is an algorithm satisfying $\SwapRegret(T) \leq \ep$.   
  \end{itemize}
\end{corollary}
\begin{remark}
    Recall that our algorithm plays a \emph{distribution} over actions in each round. In infinite action-spaces, distributions can be infinitely-supported. Yet, for the classes considered in Corollary~\ref{cor:ldim-ce}, \emph{uniform convergence} holds. This means that any infinitely-supported distribution can be replaced by a uniform distribution over a sufficiently-large i.i.d. sample.
\end{remark}

Next, we utilize the fact that if each player uses a learning algorithm with swap regret bounded by $\ep$, then the empirical average of their joint strategy profiles is an $\ep$-approximate CE. Since there is a learning algorithm with respect to $\MX, \MF$ as above with swap regret converging to 0 as long as $\Rad_T(\MX) \to 0$ as $T \to \infty$ by \cref{cor:rad-swap} (which, in turn, by \cref{prop:ldim-sfat} holds if $\esfsd{\MX} < \infty$ for all $\delta > 0$), we obtain the following as a further corollary:
\begin{corollary}\label{cor:ce-existence}
  Let $(S,A)$ be an $m$-player game. Suppose either of the below conditions hold:
  \begin{itemize}
  \item $\Rad_T(S,A) \to 0$ as $T \to \infty$ (which, in turn, holds if $\esfsd{S,A} < \infty$ for all $\delta > 0$);
  \item $(S,A)$ is binary-valued and $\Ldim(S,A) < \infty$.
  \end{itemize}
  Then $(S,A)$ has an $\ep$-CE for all $\ep > 0$. %
\end{corollary}

\subsection{Application: query and communication complexities of computing correlated equilibria}
\label{sec:qc-cc}

\begin{algorithm}
  \caption{\texttt{CommCE}$(A_1, \ldots, A_m, \ep,\delta)$}
	\label{alg:commce}
	\begin{algorithmic}[1]%
      \Require $m$-player, $N$-action normal-form game specified by payoff matrices $A_1, \ldots, A_m : [N]^m \to [0,1]$, real numbers $\ep,\delta \in (0,1)$.
      \State Set $L \gets \frac{C\log(TNm/\delta)}{\ep^2}$, $M = \lceil \log(N)/\ep^2 \rceil$, $d = \lceil 1/\ep \rceil$, and $T = M^{d}$, for a sufficiently large constant $C$. \label{line:set-params-commce}
      \State Each player $i$ initializes an instance $\TreeSwap_i$  of $\TreeSwap(\MF_i, \MX, \MWU, T,M,d)$, where $\MX = [N]$, $\MF_i := \{ a_i \mapsto  \E_{a_{-i} \sim \bp}[ A_i(a_i, a_{-i})]  \ : \ \bp \in \Delta_{[N]^{m-1}} \}$. %
      \For{$1 \leq t \leq T$}
      \For{Player $1 \leq i \leq m$}
      \State Player $i$ receives a distribution over actions at round $t$, $\bx_i\^t \in \Delta_{\MX}$, from $\TreeSwap_i$.
      \State Sample $a_{i,1}\^t, \ldots, a_{i,L}\^t \sim \bx_i\^t$. %
      \State %
      Send $a_{i,1}\^t, \ldots, a_{i,L}\^t$ to all other players.
      \EndFor
      \State For $i \in [m]$, define $\hat\bu_i\^t \in [0,1]^N$ by, for $a \in [N]$, $ \hat\bu_i\^t[a] \gets \frac 1L \sum_{\ell=1}^L[A_i(a, a_{-i,\ell}\^t)]$.\label{line:compute-uhat}
      \State For $i \in [m]$, update $\TreeSwap_i$ with the function $\bx \mapsto \lng \bx,  \hat\bu_i\^t\rng$.
      \EndFor
      \State \Return the distribution $\frac 1T \sum_{t=1}^T( \bx_1\^t \times \cdots \times \bx_m\^t) \in \Delta_{[N]^m}$. 
    \end{algorithmic}

  \end{algorithm}
In this section, we discuss corollaries of \cref{thm:treeswap} for communication and query-efficient protocols for computing correlated equilibria in normal-form games. We begin with communication complexity (\cref{def:comm-c}): we show that for $\ep = O(1)$ and $m = O(1)$, the communication complexity of computing an $\ep$-CE in a $m$-player normal-form game is $\mathrm{poly} \log N$, where $N$ denotes the number of actions per player. In contrast, the best known prior bound was $\tilde O(N)$, using the sublinear swap regret algorithm of \cite{Blum07:From} (see also \cite{babichenko2020informational}).
  \begin{corollary}[Communication complexity of CE]
    \label{cor:commce}
    Let $N,m \in \mathbb{N}$ be fixed and suppose $A_1, \ldots, A_m : [N]^m \to [0,1]$  are the payoff matrices of an $m$-player $N$-action normal-form game $G$. Then for any $\ep,\delta \in (0,1)$, there is an algorithm which %
    outputs an $\ep$-CE of $G$ with probability $1-\delta$, for which the  communication complexity is $(\log(N)/\ep)^{O(1/\ep)} \cdot O(m\log(m/\delta))$.
  \end{corollary}
  \begin{proof}
    We show that the algorithm $\CommCe(A_1, \ldots, A_m, \ep, \delta)$ (Algorithm \ref{alg:commce}) has the desired properties. \cref{alg:commce} works by having each player $i \in [m]$ run an instance $\TreeSwap_i$ of $\TreeSwap$ (\cref{alg:treeswap}). Each instance $\TreeSwap_i$ is used with action set equal to $\MX = \Delta_N$, and with function class $\MF_i = \left\{ a_i \mapsto  \E_{a_{-i} \sim \bp}[A_i(a_i, a_{-i})] \ : \ \bp \in \Delta_{[N]^{m-1}} \right\}$. Note that $\MF_i$ is a convex hull of at most $N^{m-1}$ vectors. The external regret algorithm used in the $\TreeSwap_i$ instance is simply $\MWU$.     As in \cref{alg:commce}, we let the distribution played by agent $i$'s instance $\TreeSwap_i$ be denoted by $\bx_i\^t \in \Delta_\MX$.  Instead of transmitting $\bx_i\^t$ to each other player each round (which would take at least $N$ bits), each player $i$ samples $L= O \left( \frac{ \log(TNm/\delta)}{\ep^2} \right)$ actions from $\bx_i\^t$ and transmits those actions to all other players, which is then used in other players' approximation of their utility, denoted by $\hat \bu_{i'}\^t$ in \cref{alg:commce}.

    Fix $\ep,\delta \in (0,1)$. We define, for each $i \in [m]$ and $t \in [T]$, $\bu_i\^t \in [0,1]^N$ by $\bu_i\^t[a] := \E_{a_{i'} \sim \hat\bx_{i'}\^t\  \forall i' \neq i} [A_i(a, a_{-i})]$ for $a \in [N]$. %
    We choose $M = \lceil \log(N)/\ep^2 \rceil$, $d = \lceil 1/\ep \rceil$, and $T = M^{d}$. Recall that each algorithm $\TreeSwap_i$ used in $\CommCe$ is an instance of $\TreeSwap(\MF_i, \MX, \MWU, T, M, d)$. 

    \paragraph{Step 1: correctness.} Note that the instances of $\MWU$ %
    satisfy Assumption \ref{asm:no-external-reg} with external regret bound $O(\ep)$.

    Hoeffding's inequality together with the choice of $L$ in Line \ref{line:set-params-commce} of \cref{alg:commce} (as long as $C$ is sufficiently large) and a union bound over $t \in [T], i \in [m]$ ensures that under some event $\ME$ which occurs with probability at least $1-\delta$, %
    \begin{align}
      \max_{t \in [M]} \max_{i \in [m]} \left\| \hat \bu_i\^t - \bu_i\^t \right\|_\infty \leq \ep\nonumber.
    \end{align}
    
    Note that the utility vectors used to update $\TreeSwap_i$ for each player $i$ are given by $\hat\bu_i\^t$; thus, by \cref{thm:treeswap}, the swap regret of $\TreeSwap_i$ may be bounded above as follows:
    \begin{align}
\max_{\pi : [N] \to [N]} \frac 1T \sum_{t=1}^T \sum_{a=1}^N \bx_i\^t[a] \cdot (\hat\bu_i\^t[\pi(a)] - \hat\bu_i\^t[a]) \leq O(\ep)\nonumber.
    \end{align}
    Under $\ME$, it follows that, for each $i \in [m]$,
        \begin{align}
      & \max_{\pi : [N] \to [N]}  \frac 1T \sum_{t=1}^T \left(\E_{a_{i'} \sim \bx_{i'}\^t\ \forall i' \in [m]}(A_i(\pi(a_i), a_{-i}) - A_i(a_1, \ldots, a_m)) \right)\nonumber\\
          =& \max_{\pi : [N] \to [N]} \frac 1T \sum_{t=1}^T \sum_{a=1}^N \bx_i\^t[a] \cdot \left( \bu_i\^t[\pi(a)] - \bu_i\^t[a] \right)\nonumber\\
          \leq & \max_{\pi : [N] \to [N]} \frac 1T \sum_{t=1}^T \sum_{a=1}^N \bx_i\^t[a] \cdot \left( \hat\bu_i\^t[\pi(a)] - \hat\bu_i\^t[a] \right) + O(\ep)\leq O(\ep)\nonumber,
    \end{align}
    which implies that $\frac 1T \sum_{t=1}^T (\bx_1\^t \times \cdots \times \bx_m\^t)$ is an $O(\ep)$-CE of $G$. The claimed result follows by rescaling $\ep$ by a constant factor.

    \paragraph{Step 2: communication cost.} %
For each $t \in [T]$, each player $i \in [m]$ must communicate $L$ actions, which takes $O(L \log n)$ bits. 
    Hence the total communication cost is $O(Tm L \log(N))$, which may be bounded as follows:
    \begin{align}
TmL \log(N) \leq O \left( \frac{m \log(N/\ep) \log(TNm/\delta)}{\ep^2} \cdot \left( \frac{\log N}{\ep^2} \right)^{C/\ep} \right) \leq O(m \log(m/\delta)) \cdot \left( \frac{\log N}{\ep} \right)^{O(1/\ep)}.\nonumber
    \end{align}
  \end{proof}

  As a corollary of \cref{cor:commce}, we obtain that there is a computationally efficient algorithm  to compute a $\mathrm{poly} \log N$-sparse $\ep$-CE in $O(1)$-player normal-form games. Prior to the present work, the smallest sparsity obtained by any efficient algorithm was $\tilde O(N)$ \cite{babichenko2014simple}. 
  \begin{corollary}[Efficient computation of sparse CE]
    \label{cor:sparse-ce}
    For any $N$-action $m$-player game and $\ep, \delta \in (0,1)$, there is an algorithm which computes a $(\log(N)/\ep)^{O(1/\ep)}$-%
    sparse CE in time $\mathrm{poly}(N^m, (\log(N/\delta)/\ep)^{O(1/\ep)})$ with probability at least $1-\delta$.
\end{corollary}
\begin{proof}
  We simply run the algorithm $\CommCe$, which clearly runs in the claimed time. The sparsity of the returned solution is $T$.
\end{proof}

Next, we proceed to our corollary for query complexity (\cref{def:query-c}): we show that for $\ep = O(1)$ and $m = O(1)$, the query complexity of computing an $\ep$-CE in an $m$-player normal-form game is $\mathrm{poly} \log N$; the best known prior bound was $\tilde O(N^2)$, using the sublinear swap regret algorithm of \cite{Blum07:From} (see also \cite{babichenko2020informational}). 
  \begin{corollary}[Query complexity of CE]
    \label{cor:queryce}
        Let $N,m \in \mathbb{N}$ be fixed and suppose $A_1, \ldots, A_m : [N]^m \to [0,1]$  are the payoff matrices of an $m$-player $N$-action normal-form game $G$. Then for any $\ep,\delta \in (0,1)$, there is an algorithm which %
        outputs an $\ep$-CE of $G$ with probability $1-\delta$, which queries  at most  $(\log(N)/\ep)^{O(1/\ep)} \cdot O(mN\log(m/\delta))$ 
    entries of the payoff matrices $A_i$.
  \end{corollary}
  The proof of \cref{cor:queryce} proceeds via essentially the same approach as for \cref{cor:commce}; for completeness, we write out the details.  
  \begin{proof}[Proof of Corollary \ref{cor:queryce}]
    Fix $\ep, \delta \in (0,1)$. 
    We show that the algorithm $\QueryCe(A_1, \ldots, A_m, \ep, \delta)$ (Algorithm \ref{alg:queryce}) has the desired properties. The algorithm $\QueryCe$ operates by having each player $i \in [m]$ run an instance of $\TreeSwap$, denoted by $\TreeSwap_i$. Each instance $\TreeSwap_i$ is used with action set $\MX = \Delta_N$,  with function class $\MF = [0,1]^{[N]} = \{ a \mapsto \bu[a] \ : \ \bu \in [0,1]^N \}$, and with the external regret algorithm set to multiplicative weights, $\MWU$. Let $\bx_i\^t\in\Delta_N$ denote the distribution played by each $\TreeSwap_i$ instance at round $t$.

    Define, for each $i \in [m]$ and $t \in [T]$, $\bu_i\^t \in [0,1]^N$ by $\bu_i\^t[a] := \E_{a_{i'} \sim \bx_{i'}\^t\  \forall i' \neq i} [A_i(a, a_{-i})]$ for $a \in [N]$. $\bu_i\^t$ represents the ``ideal'' utility vector that would be passed to each $\TreeSwap_i$ instance at round $t$. Computing $\bu_i\^t$ exactly would take too many queries to the payoffs $A_i$, so instead $\QueryCe$ computes an approximation of them: 
    in particular, it samples $L = O(\frac{\log(TNm/\delta)}{\ep^2})$ actions from each $\bx_i\^t$ and uses these samples to compute an empirical estimate $\hat \bu_i\^t$ of $\bu_i\^t$ (\cref{line:compute-uhat}).

    Hoeffding's inequality together with the choice of $L$ in Line \ref{line:set-params-q} of \cref{alg:queryce} (as long as $C$ is chosen sufficiently large) and a union bound over $t \in [T], i \in [m]$, and the $N$ coordinates of each utility vector, ensure that, under some event $\ME$ occurring with probability $1-\delta$ over the execution of $\CommCe$, for all $i \in [m]$ and $t \in [T]$, it holds that $\| \bu_i\^t - \hat \bu_i\^t \|_\infty \leq \ep/4$. The guarantee for $\TreeSwap_i$ (\cref{thm:treeswap}, which with the settings of the parameters used in $\QueryCe$ becomes exactly Corollary \ref{cor:treeswap-finiten})  together with the choice of $T,M,d$ in Line \ref{line:set-params-q} ensures that, for each $i \in [m]$,
    \begin{align}
\max_{\pi : [N] \to [N]}\frac{1}{T} \sum_{t=1}^T\sum_{j=1}^N  \bx_i\^t[j] \cdot \left( \hat \bu_i\^t[\pi(j)] - \hat \bu_i\^t[j]\right) \leq \ep/2\nonumber.
    \end{align}
    Under the event $\ME$, it then  follows that, for each $i \in [m]$,
    \begin{align}
      & \max_{\pi : [N] \to [N]}  \frac 1T \sum_{t=1}^T \left(\E_{a_{i'} \sim \bx_{i'}\^t\ \forall i' \in [m]}(A_i(\pi(a_i), a_{-i}) - A_i(a_1, \ldots, a_m)) \right)\nonumber\\
      =& \max_{\pi : [N] \to [N]} \frac 1T \sum_{t=1}^T \sum_{j=1}^N \bx_i\^t[j] \cdot \left( \bu_i\^t[\pi(j)] - \bu_i\^t[j] \right) \leq \ep\nonumber,
    \end{align}
    which implies that $\frac 1T \sum_{t=1}^T (\bx_1\^t \times \cdots \times \bx_m\^t)$ is an $\ep$-CE of $G$.

    The total number of queries made in the course of $\QueryCe$ (in Line \ref{line:compute-uhat-q}) is 
    \[
    NL \cdot Tm  \leq O \left( \frac{Nm \log(TNm/\delta)}{\ep^2} \cdot \left( \frac{\log N}{\ep^2} \right)^{C/\ep} \right) \leq O(Nm \log(m/\delta)) \cdot \left( \frac{\log N}{\ep}\right)^{O(1/\ep)}.
   \]
 \end{proof}

  \begin{algorithm}
  \caption{\texttt{QueryCE}$(A_1, \ldots, A_m, \ep,\delta)$}
	\label{alg:queryce}
	\begin{algorithmic}[1]%
      \Require $m$-player, $N$-action normal-form game specified by payoff matrices $A_1, \ldots, A_m : [N]^m \to [0,1]$, real numbers $\ep,\delta \in (0,1)$.
      \State Set $M \gets \frac{C \log N}{\ep^2}$, $d \gets \lceil C/\ep \rceil$, $T \gets M^d$, and $L \gets \frac{C \log(TNm/\delta)}{\ep^2}$, for a sufficiently large constant $C$. \label{line:set-params-q}
      \State Each player $i$ initializes an instance $\TreeSwap_i$  of $\TreeSwap(\MF, \MX, \MWU, T, M, d)$, where $\MX = \Delta_N$, $\MF := \{ a \mapsto \bu[a] \ : \ \bu \in [0,1]^N \}$, and $\MWU$ denotes the multiplicative weights algorithm.
      \For{$1 \leq t \leq T$}
      \For{Player $1 \leq i \leq m$}
      \State Player $i$ receives a distribution over actions at round $t$, $\bx_i\^t \in \Delta_{\MX}$, from $\TreeSwap_i$. 
      \State Sample $L$ actions from $\bx_i\^t$, denoted $a_{i,1}, \ldots, a_{i,L}$. %
      \State Define $\hat \bu_i\^t \in [0,1]^N$ by, $\hat \bu_i\^t[j] \gets \frac 1N \sum_{\ell=1}^L A_i(j, (a_{i', \ell})_{i' \neq i})$ by making $NL$ queries to $A_i$. \label{line:compute-uhat-q}
      \State Update $\TreeSwap_i$ with the function $a_i \mapsto \hat\bu_i\^t[a_i]$.
      \EndFor
      \EndFor
      \State \Return the distribution $\frac 1T \sum_{t=1}^T( \bx_1\^t \times \cdots \times \bx_m\^t) \in \Delta_{[N]^m}$. 
    \end{algorithmic}

  \end{algorithm}

  \subsection{Application: extensive-form games}
  \label{sec:efg}
  In this section, we use \cref{thm:treeswap} to efficiently compute $\ep$-CE in extensive-form games in polynomial time when $\ep = O(1)$. We begin by briefly introducing the terminology and notation for extensive-form games. An \emph{$m$-player extensive-form game (EFG)}
  $G$ is specified by a tree $\MT$, where the children of each node $h$ of $\MT$ are indexed by a set of \emph{actions} $A(h)$ at the node $h$. To each non-leaf node $h$ of $\MT$ is associated some player in $[m] \cup \{ \ch \}$, where $\ch$ denotes the \emph{chance player}, which plays actions at each of its nodes according to some fixed probabilities. For each $i \in [m]$, player $i$'s nodes are partitioned into \emph{information sets}: nodes in the same information set cannot be distinguished by player $i$. We assume \emph{perfect recall}, which means that each player does not forget any information (i.e., distinct information sets cannot have as descendents two nodes in the same information set). Letting the set of leaves of $\MT$ be denoted by $Z$, the specification of the EFG is completed by functions $u_i : Z \to \BR$ for each $i \in [m]$, which describe each player's utility received upon reaching each leaf.

  Let $\MI_i$ denote the set of information sets of player $i$; for an information set $I \in \MI_i$, let $A(I)$ denote the set of actions available at each node of $I$.  A \emph{policy} or (\emph{normal-form plan}) for player $i$ in the EFG $G$ is a function $\pi_i$ which maps each $I \in \MI_i$ to some element of $A(I)$. We denote the set of policies of player $i$ by $\Pi_i$. An EFG may be viewed as a normal form game where player $i$'s action set is $\Pi_i$, and her value function is $V_i(\pi_1, \ldots, \pi_m) := \sum_{z \in \MZ} u_i(z) \cdot p^{\ch}(z) \cdot\prod_{j \in [m]} p^{\pi_j}(z)$, where $p^{\ch}(z)$ denotes the probability that the chance player takes a sequence of actions consistent with reaching $z$, and $p^{\pi_j}(z) \in \{0,1\}$  denotes the indicator of whether $\pi_j$ takes a sequence of actions consistent with reaching $z$.   %

  \paragraph{Sequence-form polytope.}   A \emph{sequence} for player $i$ consists of either (a) a pair $(I, a)$, where $I \in \MI_i$ and $a \in A(I)$; or (b) the empty sequence $\emptyset$. The set of sequences for player $i$ is denoted by $\Sigma_i$. For an information set $I \in \MI_i$, we let $\sigma_i(I) \in \Sigma_i$ denote the unique sequence for player $i$ which leads to $I$ (which is unique by perfect recall).  Similarly, for a leaf $z \in Z$, let $\sigma_i(z) \in \Sigma_i$ denote the unique sequence of player $i$ which leads to $z$. 
  Existing external regret algorithms for learning in EFGs make use of the \emph{sequence-form polytope} $\MQ_i$ for each player $i$, defined below:
  \begin{align}
\MQ_i := \left\{ \bq \in \BR^{\Sigma_i} \ : \ \bq[\emptyset] = 1, \  \bq[\sigma_i(I)] = \sum_{a \in A(I)} \bq[(I, a)] \ \forall I \in \MI_i \right\}.\nonumber
  \end{align}
  Each element $\bq \in \MQ_i$ corresponds to the following distribution $P(\bq)$ over policies $\Pi_i$: at each information set $I$, sample action $a \in A(I)$ with probability $\frac{\bq[(I,a)]}{\bq(\sigma(I))}$, independently at each information set.\footnote{This distribution over policies may have support which is exponential in $|\MI_i|$; there is also an equivalent distribution over policies  which is guaranteed to have polynomial-size support, which can be computed efficiently: see Theorem 4 of \cite{Celli19:Learning}.}
  The key property of $P(\bq)$ is that for all $z \in Z$,\noah{cite?}
  \begin{align}\label{eq:q-equiv}
\bq[\sigma_i(z)] = \E_{\pi_i \sim P(\bq)}[p^{\pi_i}(z)].
  \end{align}
  
We let $A_i := \max_{I \in \MI_i} |A(I)|$. 
  The following result establishes a guarantee for efficient external regret minimization in EFGs, thus verifying \cref{asm:no-external-reg}:
  \begin{theorem}[Theorem 5.5 of \cite{farina2022kernelized}]
    \label{thm:komwu}
    There is an algorithm running in time $\mathrm{poly}(|\MI_i|, A_i)$ which, given sequentially an adversarial sequence $\bu\^1, \ldots, \bu\^T \in [0,1]^{\Sigma_i}$, produces a sequence $\bq\^1, \ldots, \bq\^T \in \MQ_i$ satisfying the following external regret guarantee:
    \begin{align}
\max_{\bq^\st \in \MQ_i} \sum_{t=1}^T \left( \lng \bq^\st, \bu\^t \rng - \lng \bq\^t, \bu\^t \rng \right) \leq O \left( \sqrt{|\MI_i| \log(A_i) / T}\right)\nonumber.
    \end{align}
  \end{theorem}
  \cref{thm:komwu} improves the classical guarantee of \emph{counterfactual regret minimization} \cite[Theorem 4]{zinkevich2008regret}, which obtains a regret guarantee of $O(|\MI_i| \sqrt{A_i/T})$. We may now combine \cref{thm:komwu,thm:treeswap}, as follows:
  \begin{theorem}
    \label{thm:efg-formal}
    Given an $m$-player extensive-form game $G$, write $I^\st := \max_{i \in [m]} |\MI_i|$ and $A^\st := \max_{i \in [m]} A_i$. 
    For any $\ep \in (0,1)$, there is an algorithm running in time $\mathrm{poly}(m, A^\st, (I^\st \log A^\st / \ep)^{1/\ep})$ which outputs an $\ep$-approximate (normal-form) CE of $G$. 
\end{theorem}
\begin{proof}
  Fix $T = (I^\st \log A^\st / \ep)^{C/\ep}$ for a sufficiently large constant $C$. 
  For each $i \in [m]$, write $\MF_i :=\mathrm{co}( \{ \pi_i \mapsto V_i(\pi_i, \pi_{-i}) \ : \ \pi_{-i} \in \prod_{i' \neq i} \Pi_{i'} \})$, where $\mathrm{co}(\cdot)$ denotes the convex hull. We let each player $i$ run the algorithm $\TreeSwap(\MF_i, \Pi_i, \Alg_i, T)$, where $\Alg_i$ is the algorithm of \cref{thm:komwu}, with appropriate pre-processing and post-processing modifications (due to the fact that the algorithm of \cref{thm:komwu} does not technically speaking play actions in $\Delta(\Pi_i)$ nor does it receive as feedback functions in $\MF_i$). After describing these modifications, we will apply the guarantee of \cref{thm:treeswap} with $M = \frac{\max_i \{ |\MI_i| \log(A_i)\}}{\ep^2}$, and $d$ chosen so that $M^{d-1} \leq T \leq M^d$, so that $d \geq 1/\ep$. 

  To avoid confusion, we denote rounds of execution for each $\TreeSwap_i$ instance by $s \in [T]$, and rounds of execution for each $\Alg_i$ instance (used within $\TreeSwap_i$) by $t \in [M]$. 
  We next describe the pre-processing and post-processing modifications for $\Alg_i$:
  \begin{itemize}
      \item \emph{Post-processing for $\Alg_i$:} At each round $t$ of execution of $\Alg_i$, it produces a vector $\bq_i\^t \in \Delta(\MQ_i)$; it passes $P(\bq_i\^t) \in \Delta(\Pi_i)$ for use in $\TreeSwap$. %

      \item \emph{Post-processing for $\TreeSwap_i$:} At each round $s$ of execution of $\TreeSwap_i$, the algorithm $\TreeSwap_i$ produces a distribution $P_i\^s \in \Delta(\Pi_i)$, which, by the post-processing for $\Alg_i$ described above and \cref{line:play-xt} of \cref{alg:treeswap}, can be expressed as a uniform mixture of the form $P_i\^s = \frac 1d \sum_{h=1}^d P(\bq_i\^{s,h}) \in \Delta(\Pi_i)$, for vectors $\bq_i\^{s,h} \in \MQ_i$. Given the distributions $P_i\^s$ for each $i \in [m]$, we need to produce a function $\bff_i\^s \in \MF_i$ to give as feedback for each $\TreeSwap_i$ instance. To do so, we define $\bar \bq_i\^s := \frac 1d \sum_{h=1}^d \bq_i\^{s,h}$, and then define
        \begin{align}
\bff_i\^s(\pi_i) := \E_{\pi_j \sim P(\bar\bq_j\^s)\ \forall j \neq i}[V_i(\pi_i, \pi_{-i})].\label{eq:fs-structure}
        \end{align}
      \item \emph{Pre-processing for $\Alg_i$:} Each $\Alg_i.\mathtt{update}$ procedure in $\TreeSwap$  is passed as input an average $\bar\bff$ of elements $\bff\^s \in \MF_i$ (in \cref{line:alg-update}). 
For each of these elements $\bff_i\^s$, we compute some vector $\bu_i\^s \in [0,1]^{\Sigma_i}$, as defined below, and then average these vectors $\bu_i\^s$, producing some $\bar\bu \in [0,1]^{\Sigma_i}$. The resulting average vector $\bar\bu$ is then passed to the algorithm of \cref{thm:komwu}. 

By construction, each $\bff_i\^s$ may be written of the form in \cref{eq:fs-structure}. For each such $\bff_i\^s$,  
$\Alg_i$ computes the utility vector $\bu_i\^s \in [0,1]^{\Sigma_i}$ defined by $\bu_i\^s[\sigma_i(z)] := u_i(z) p^{\ch}(z) \prod_{j \neq i} \E_{\pi_j \sim P(\bar\bq_j\^s)}[p^{\pi_j}(z)]$ for leaves $z$, and $\bu_i\^s[\sigma_i] = 0$ for all other sequences $\sigma_i \in \Sigma_i$. Since each distribution $P(\bar\bq_j\^s)$ randomizes independently at each information set, it is clear that $\E_{\pi_j \sim P(\bar\bq_j\^s)}[p^{\pi_j}(z)]$, and thus $\bu_i\^s$, can be computed in polynomial time.

Note that, for each $\bq_i \in \MQ_i$, we have
    \begin{align}
      \lng \bq_i, \bu_i\^s \rng =&  \sum_{z \in \MZ} u_i(z) p^{\ch}(z) \E_{\pi_i \sim P(\bq_i)}[p^{\pi_i}(z)] \cdot \prod_{j \neq i} \E_{\pi_j \sim P(\bar\bq_j\^s)}[p^{\pi_j}(z)] \nonumber\\
      =&  \E_{\pi_i \sim P(\bq_i)}\E_{\pi_j \sim P(\bar\bq_j\^s) \ \forall j \neq i}[V_i(\pi_1, \ldots, \pi_m)]\label{eq:qu-ip}.
    \end{align}
  \end{itemize}
  We claim that $\Alg_i$, as defined above, satisfies the external regret bound of \cref{asm:no-external-reg}. To prove this, consider any instance of $\Alg_i$ in $\TreeSwap$, and %
  consider an adversarial sequence $\bar \bff_i\^1, \ldots, \bar\bff_i\^M \in \MF_i$ passed to $\Alg_i$. Each $\bar \bff_i\^t$ may be written as an average of functions $\bar\bff_i\^{t,h}$, of the form  $\bar \bff_i\^{t,h}(\pi_i) = \E_{\pi_j \sim P(\bar\bq_i\^{t,h})\ \forall j \neq i}[V_i(\pi_i, \pi_{-i})]$, where $\bar\bq_i\^{t,h} \in \MQ_i$ (this is exactly the form of \cref{eq:fs-structure}).  %
  Accordingly, let $\bar\bu\^1, \ldots, \bar\bu\^M \in [0,1]^{\Sigma_i}$ denote the vectors produced by the above pre-processing procedure for $\Alg_i$, and recall that, for $t \in [M]$, $P(\bq_i\^t) \in \Delta(\Pi_i)$ denotes the post-processed action distribution produced by $\Alg_i$. Then we have
  \begin{align}
    \max_{\pi_i^\st \in \Pi_i} \sum_{t=1}^M \left( \bar \bff\^t(\pi_i^\st) - \E_{\pi_i \sim P(\bq_i\^t)} \bar\bff\^t(\pi_i) \right)= \max_{\bq^\st \in \MQ_i} \sum_{t=1}^M \left(\lng \bq^\st, \bar\bu\^t \rng - \lng \bq_i\^t, \bar\bu\^t \right) \leq O \left( \sqrt{|\MI_i| \log(A_i) \cdot M}\right)\nonumber,
  \end{align}
  where the equality uses the fact that each $\pi_i\^s \in \Pi_i$ can be expressed as $P(\bq^\st)$ for some $\bq^\st \in \MQ_i$, as well as \cref{eq:fs-structure,eq:qu-ip}, and the inequality uses \cref{thm:komwu}. Thus, each $\Alg_i$ satisfies \cref{asm:no-external-reg} with $R_{\Alg_i}(M) = O \left( \sqrt{|\MI_i| \log(A_i)/M}\right)$.

 Then by \cref{thm:treeswap}, each player's swap regret may be bounded as follows:
  \begin{align}
    & \max_{\phi : \Pi_i \to \Pi_i} \frac 1T \sum_{s=1}^T \sum_{\pi_i \in \Pi_i} P_i\^s(\pi_i) \cdot \left( \E_{\pi_j \sim P_j\^s \ \forall j \neq i}[V_i(\phi(\pi_i), \pi_{-i}) - V_i(\pi_i, \pi_{-i})] \right)  \nonumber\\
    =&  \max_{\phi : \Pi_i \to \Pi_i} \frac 1T \sum_{s=1}^T \sum_{\pi_i \in \Pi_i} P_i\^s(\pi_i) \cdot \left( \E_{\pi_j \sim P(\bar\bq_j\^s) \ \forall j \neq i}[V_i(\phi(\pi_i), \pi_{-i}) - V_i(\pi_i, \pi_{-i})] \right) \nonumber\\
    =& \max_{\phi : \Pi_i \to \Pi_i} \frac 1T \sum_{s=1}^T \sum_{\pi_i \in \Pi_i} P_i\^s(\pi_i) \cdot \left( \bff_i\^s[\phi(\pi_i)] - \bff_i\^s[\pi_i] \right) \leq \frac{3}{d} + O \left( \max_i \sqrt{|\MI_i|\log(A_i)/M} \right) \leq O(\ep)\nonumber,
  \end{align}
  where the first equality uses the definition of $\bar \bq_j\^s$ and the fact that $\E_{\pi_j \sim P(\bar\bq_j\^s)}[p^{\pi_j}(z)] = \E_{\pi_j \sim P_j\^s}[p^{\pi_j}(z)]$ for all leaves $z \in Z$ (by \cref{eq:q-equiv}), the second equality uses \cref{eq:fs-structure}, and the first inequality uses the guarantee of \cref{thm:treeswap}.  By rescaling $\ep$ by a constant factor, we see that each player obtains swap regret of at most $\ep$, meaning that $\frac{1}{T} \sum_{s=1}^T P_1\^s \times \cdots \times P_m\^s \in \Delta(\Pi_1 \times \cdots \times \Pi_m)$ is a (normal-form) $\ep$-approximate CE of the EFG. 
\end{proof}
  
  {
    \subsection{Application: bandit no-swap regret algorithm}
    \label{sec:bandit-swap}

  \begin{algorithm}
  \caption{\texttt{BanditTreeSwap}$(N, T, M, d)$}
	\label{alg:bandit-treeswap}
	\begin{algorithmic}[1]%
      \Require Action set $[N]$, time horizon $T$, parameters $M, d$ with $T/N \leq M^d$. 
      \State \label{line:bandit-params} For each $h \in [d]$ and sequence  $\sigma \in  \{0, 1, \ldots, M-1\}^{h-1}$, initialize an instance of $\ExpMulti(N, M, M^{-1/2}, K^{-1} T^{-1/6}, K)$ with $K = \frac{2NM^{d-h}}{d}$ and time horizon $M$, denoted $\ExpMulti_\sigma$. ({\emph{See \cref{alg:exp3-multi}.}})
      \For{$1 \leq t \leq T$} %
      \State Let $\sigma = (\sigma_1, \ldots, \sigma_d)$ denote the base-$M$ representation of $\lfloor \frac{t-1}{N} \rfloor$.
      \For{$1 \leq h \leq d$}
      \If{$\sigma_{h+1} = \cdots = \sigma_d = 0$ or $h=d$}
      \State If $\sigma_h > 0$, set $\ExpMulti_{\sigma_{1:h-1}}.\mathtt{curDistribution} \gets \ExpMulti_{\sigma_{1:h-1}}.\mathtt{update}()$.\label{line:curdistribution-bandit} %
      \EndIf
      \EndFor
      \State \label{line:sample-ht} Sample $h\^t \sim [d]$ uniformly at random.
      \State \label{line:sample-at} Output an action $a\^t \sim \ExpMulti_{\sigma_{1:h\^t-1}}.\mathtt{curDistribution}$.
      \State Observe reward $\bu\^t[a\^t]$, and call $\ExpMulti_{\sigma_{1:h\^t-1}}.\mathtt{storeSample}(a\^t, \bu\^t[a\^t])$. 
      \EndFor
    \end{algorithmic}
  \end{algorithm}

  In this section, we discuss an application of our techniques in $\TreeSwap$ to the \emph{bandit} setting. The bandit setting with a finite number of arms is identical to the ``$N$ experts" setting of \cref{sec:nexperts}, with the exception that the learning algorithm has to choose a single action $a\^t \in [N]$ at each round $t$. As feedback, the learner only sees the coordinate of the utility vector which it selected at round $t$: in particular, if the adversary plays $\bu\^t \in [0,1]^N$, the learner observes only $\bu\^t[a\^t]$ (as opposed to the entire vector $\bu\^t$). In the bandit setting, the swap regret is defined exactly as in \cref{def:SR} with the distributions $\bx\^t$ interpreted as singletons on $a\^t$: explicitly, we have
  \[
\SwapRegret(a\^{1:T}, \bu\^{1:T}) = \sup_{\pi : [N] \to [N]} \frac 1T \sum_{t=1}^T \left( \bu\^t[\pi(a\^t)] - \bu\^t[a\^t] \right).
\]
When the context is clear, we abbreviate $\SwapRegret(T) = \SwapRegret(a\^{1:T}, \bu\^{1:T})$. 
  \cite{Blum07:From} showed that any algorithm in the bandit setting which achieves $\SwapRegret(T) \leq \ep$ requires $T \geq \Omega\left(\frac{N}{\ep^2}\right)$ rounds; this bound was improved to $T \geq \Omega\left( \frac{N \log N}{\ep^2} \right)$ by \cite{ito2020tight}. The best upper bounds were larger by a polynomial factor: \cite{Blum07:From} showed that it suffices to have $T \leq O\left( \frac{N^3 \log N}{\ep^2} \right)$ which was improved to $T \leq O \left( \frac{N^2 \log N}{\ep^2} \right)$ by \cite{ito2020tight,jin2022vlearning}, still leaving a quadratic gap from the lower bound of \cite{Blum07:From}.\footnote{We remark that the upper bound of \cite{ito2020tight} bounds only the weaker notion of pseudo-swap regret.} Our upper bound in \cref{thm:bandit-tree-swap} below closes this quadratic gap up to $\mathrm{poly} \log N$ factors in the setting of constant $\ep$. Finally, for simplicity, we state and prove \cref{thm:bandit-tree-swap} for an \emph{oblivious} adversary, as is somewhat standard in the adversarial bandit setting \cite{lattimore2020bandit}. However, our techniques extend readily (with some more cumbersome notation) to the adaptive adversary setting. 
  \begin{theorem}
    \label{thm:bandit-tree-swap}
    Let $N \in \mathbb{N}$, $\ep \in (0,1)$ be given, and consider any $T \geq N \cdot (\log(N)/\ep)^{O(1/\ep)}$. Let $\bu\^1, \ldots, \bu\^T$ be a fixed (deterministic) sequence of reward vectors (i.e., produced by an oblivious adversary). Then there is a bandit algorithm which, at each time step $t$, plays an action $a\^t$ and observes only $\bu\^t[a\^t]$, for which the expected swap regret may be bounded by
    \[
      \E[\SwapRegret(a\^{1:T}, \bu\^{1:T})] \leq \ep.
      \]
    \end{theorem}
    \cref{thm:bandit-tree-swap} is proved using a variant of $\TreeSwap$, namely $\BanditTreeSwap$ (\cref{alg:bandit-treeswap}). $\BanditTreeSwap$ operates in a similar manner to $\TreeSwap$, namely by choosing $M,d$ so that $T \approx M^d$ and then constructing a $M$-ary tree of depth $d$, at each node of which lies a bandit no-external regret algorithm, which we instantiate as  $\ExpMulti$ (\cref{alg:exp3-multi}), discussed below. Due to the challenges of the bandit setting, the semantics of the algorithm $\ExpMulti$ are slightly different from those of the external regret minimizer $\Alg$ used in $\TreeSwap$. In particular, each instance $\ExpMulti$ operates over multiple rounds (when $T = M^d$ in the context of $\BanditTreeSwap$, the number of rounds for each $\ExpMulti$ instance will be $M$). Within each round $s$, $\ExpMulti$ fixes a distribution $\bp\^s \in \Delta_N$ and draws several actions $a\^{s,k} \sim \bp\^s$. It is then given samples of the form $(a\^{k,s}, u\^{k,s})$, for scalars $u\^{k,s} \in \BR$. 

    To process these samples, we assume that $\ExpMulti$ has the following subprocedures:
    \begin{itemize}
    \item A procedure $\ExpMulti.\mathtt{storeSample}(a, u)$, which stores the sample $(a,u) \in [N] \times \BR$ in a memory buffer for the current round.
    \item A procedure $\ExpMulti.\mathtt{update}(s)$, which takes as input the current round index $s$ (so that $1 \leq s \leq M$ in the context of $\BanditTreeSwap$) and uses the samples stored in the current round $s$ to compute a distribution over actions $\bp\^{s+1} \in [N]$ to be played in the subsequent round. $\ExpMulti.\mathtt{update}(s)$ then returns this distribution $\bp\^{s+1}$. Note that $\ExpMulti.\mathtt{update}(s)$ always marks the end of the current round $s$ and the beginning of round $s+1$.  
    \end{itemize}

    \paragraph{Instantiation of $\ExpMulti$.}%
    The algorithm $\ExpMulti$ (\cref{alg:exp3-multi}) is a variant of the algorithm \texttt{EXP3-IX} which obtains sublinear regret for the adversarial bandit problem (see Chapter 12 of \cite{lattimore2020bandit}).\footnote{We use \texttt{EXP3-IX} instead of the more well-known algorithm $\mathtt{EXP3}$ because we wish to obtain high-probability bounds, which $\mathtt{EXP3}$ does not guarantee \cite[Exercise 11.6]{lattimore2020bandit}.} The $\mathtt{update}$ procedure of $\ExpMulti$ is similar to that of \texttt{EXP3-IX}: suppose $\ExpMulti$ was given samples for round $s$ of the form $(a\^{s,k}, u\^{s,k})$ for $1 \leq k \leq K\^s$, where $K\^s$ denotes the total number of samples in round $s$. Each scalar $u\^{s,k}$ should be interpreted as the $a\^{s,k}$-th entry of some reward vector $\bu\^{s,k} \in [0,1]^N$, namely $u\^{s,k} = \bu\^{s,k}[a\^{s,k}]$. Then $\ExpMulti.\mathtt{update}(s)$  constructs an importance-weighted estimator of $\sum_{k=1}^{K\^s} \bu\^{s,k}$ (\cref{line:importance-weighted}), and uses this importance weighted estimator to compute a multiplicative weights update to produce $\bp\^{s+1}$ (\cref{line:exp3-ew-update}).

    We will show that, for all instantiations of $\ExpMulti$ in $\BanditTreeSwap$, the value of $K\^s$, for all rounds $s$, will be bounded below by $\Omega(N)$ with high probability. Thus, at a high level, one can think of each round of $\ExpMulti$ as an attempt to ``simulate'' a full-information update of the exponential weights algorithm: the $K\^s = \Omega(N)$ steps within round $s$ allow one to construct an estimator $\hat \bu\^s$ of the utility vector $\sum_{k=1}^{K\^s} \bu\^{k,s}$ whose entries generally have variance $O(1)$. A key challenge in analyzing $\ExpMulti$ is that the distribution $\bp\^s$ (from which the actions $a\^{k,s}$ are drawn) is not uniform; thus, one must carefully account for the fact that some entries of the estimator $\hat \bu\^s$ may have large variance.

    Below we state the main technical lemma in the proof of \cref{thm:bandit-tree-swap}. In turn, its proof makes use of an external regret bound for $\ExpMulti$, which is stated below in \cref{lem:exp3-full}.
      \begin{lemma}
    \label{lem:bandit-tree-swap}
  There is a constant $C > 0$ so that the following holds.  Let $N, T, M,d \in \mathbb{N}$ be given so that $T$ is a multiple of $N$, and $M^{d-1} \leq T/N \leq M^d$. 
  Let $\bu\^1, \ldots, \bu\^T \in [0,1]^N$ be a fixed (deterministic) sequence of reward vectors (i.e., produced by an oblivious adversary). %

  Then the expected swap regret of $\BanditTreeSwap(N, T, M, d)$ (\cref{alg:bandit-treeswap}), which produces a sequence of actions $a\^1, \ldots, a\^T \in [N]$, may be bounded as follows:
    \begin{align}
T \cdot \E[\SwapRegret(a\^{1:T}, \bu\^{1:T})] \leq \frac{3T}{d} + C \sqrt{TN \log(NT)} + C \cdot T \cdot \frac{\log^2(N)}{M^{1/6}} + \frac{CdT^2}{N} e^{-N/(3d)}\nonumber. %
    \end{align}
  \end{lemma}

        \begin{algorithm}
    \caption{\texttt{Exp3Multi}$(N, T, \eta, \gamma, K)$}
    \label{alg:exp3-multi}
    \begin{algorithmic}[1]
      \Require Action set $[N]$, time horizon $T$, step size $\eta > 0$, parameters $\gamma > 0$ and $K \in \mathbb{N}$. 
      \State Initialize $\hat L\^0_a = 0$ for all $a \in [N]$ and $\bp\^1 \gets \mathsf{Unif}([N])$.
      \For{round $1 \leq t \leq T$}
      \State Let $K\^t \in \mathbb{N}$ denote the number of samples in round $t$. \Comment{\emph{$K\^t$ need not be known prior to the round.}}
      \For{step $1 \leq k \leq K\^t$}
      \State Play action $a_k\^t \sim \bp\^t$, and receive $u_k\^t = \bu_k\^t[a_k\^t]$. %
      \State Call $\mathtt{Exp3Multi.storeSample}(a_k\^t, u_k\^t)$.
      \EndFor
      \State Set $\bp\^{t+1} \gets \mathtt{Exp3Multi.update}(t)$.
      \EndFor

      \Function{$\mathtt{Exp3Multi.update}$}{$t$}
      \State \label{line:importance-weighted} For $a \in [N]$, define $\hat Y_a\^t := \frac{1}{K} \cdot \sum_{k=1}^{K\^t} \left( \frac{\mathbbm{1}\{a_k\^t = a\} \cdot (1-u_k\^t)}{ \bp\^t[a] + \gamma} \right)$.
      \State Define $\hat L_a\^{t} \gets \hat L_a\^{t-1} + \hat Y_a\^t$ for all $a \in [N]$.
      \State \label{line:exp3-ew-update} \Return the distribution $\bp \in \Delta_N$ defined by, for $a \in [N]$,
      \begin{align}
\bp[a] = \frac{\exp \left(-\eta \hat L\^{t}_a\right)}{\sum_{b = 1}^N \exp \left(-\eta \hat L\^{t}_b\right)}.\nonumber
      \end{align}
      \EndFunction
    \Function{$\mathtt{storeSample}$}{$a,u$}
    \State Store $(a,u)$ in memory.
    \EndFunction
  \end{algorithmic}
\end{algorithm}

    First, we prove \cref{thm:bandit-tree-swap}, assuming \cref{lem:bandit-tree-swap}. 
    \begin{proof}[Proof of \cref{thm:bandit-tree-swap}]
Fix $\ep > 0$, and let $T = N \cdot (\log(N)/\ep)^{C_0/\ep}$,  for a constant $C_0$ to be specified below.   By increasing $T$ by at most $N$, we may assume without loss of generality that $T$ is a multiple of $N$. 
      Moreover, note that \cite[Corollary 25]{jin2022vlearning}\footnote{In particular, set $H=1$ in the statement of Corollary 25.}  establishes that there is an algorithm achieving expected swap regret $\E[\SwapRegret(T)] \leq O(N\sqrt{\log(N)/T})$; in particular, given $\ep \in (0,1)$, if we take $T \geq C_1 \cdot N^2 \log(N)/\ep^2$, for a sufficiently large constant $C_1 > 0$, we obtain $\E[\SwapRegret(T)] \leq \ep$. Thus, we may assume from here on that $\ep$ is chosen so that $T= N \cdot (\log(N)/\ep)^{C_0/\ep} < C_1 N^2 \log(N)/\ep^2$. As long as $C_0$ is sufficiently large, this inequality only holds when $\ep \geq 1/\log(N)$.

      Let $C$ denote the constant in the statement of \cref{lem:bandit-tree-swap}, define $M := \lceil C^6 \log^{12}(N) \ep^{-6} \rceil$, and choose $d \geq \lceil \ep^{-1} \rceil$ so that $M^{d-1} \leq T/N \leq M^d$ (such a choice of $d$ is possible  by our choice of $T = N \cdot (\log(N)/\ep)^{C_0/\ep}$, as long as $C_0$ is chosen sufficiently large). Note that $d \leq 1 + \log(T/N)/\log(M) \leq 2C_0/\ep \cdot \log(\log(N)/\ep) \leq 4C_0 \log(N) \cdot \log\log(N) < C_2\sqrt{N}/3 \leq C_2 N$ for a sufficiently large constant $C_2$. 
      
      Then by \cref{lem:bandit-tree-swap} with the chosen values of $M, d$, we obtain a (normalized) swap regret of
      \begin{align}
        & \frac{3}{d} + C \sqrt{N \log(NT) / T} + \frac{C \log^2(N)}{M^{1/6}} + \frac{CdT^2}{N} e^{-N/(3d)} \nonumber\\
        \leq & 4 \ep + C \sqrt{\log(NT)} \cdot (T/N)^{-1/4} + CC_2 (C_1 N^2 \log(N)/\ep^2)^2 e^{-N/(3d)}\nonumber\\
        \leq & 4 \ep + C \sqrt{\log(NT)} \cdot (T/N)^{-1/4} + CC_2 (C_1 N^2 \log^3(N))^2 e^{-C_2 \sqrt{N}} \leq C_3 \ep,\nonumber
      \end{align}
      where the last inequality holds for a sufficiently large constant $C_3$ and it uses the fact that $1/\log(N) \leq \ep$. 
The theorem statement follows by rescaling $\ep$ by a factor of $C_3$. 
    \end{proof}

    The proof of \cref{lem:bandit-tree-swap} proceeds in a similar manner to that of \cref{thm:treeswap}, with the added complication that we need to ensure that the empirical estimates derived from the sampled actions $a\^t$ in $\BanditTreeSwap$ concentrate to their means. 
  \begin{proof}[Proof of \cref{lem:bandit-tree-swap}]
    For each $1 \leq h \leq d$, let $\Sigma_{h-1} := \{ \sigma_{1:h-1} = (\sigma_1, \ldots, \sigma_{h-1}) :\ 1 + N \cdot \sum_{g=1}^{h-1} \sigma_g \cdot M^{d-g} \leq T \}$ denote the set of prefixes of sequences of length $h-1$ encountered over the course of $T$ rounds of $\BanditTreeSwap$. (Note that $\Sigma_{h-1}$ depends on $T$.)  Consider any  sequence $\sigma_{1:h-1} = (\sigma_1, \ldots, \sigma_{h-1}) \in \Sigma_{h-1}$ representing a node in the tree, and define $M'(\sigma_{1:h-1}) := \max \{ \sigma_h \in \{0, \ldots, M-1\} \ : \ \sigma_{1:h} \in \Sigma_h \}$. Then let  $\bx_{\sigma_{1:h-1}}\^0, \ldots, \bx_{\sigma_{1:h-1}}\^{M'(\sigma_{1:h-1})}$ denote the $M'(\sigma_{1:h-1})$ distributions chosen by the algorithm $\ExpMulti_{\sigma_{1:h-1}}$ over the course of the $T$ rounds (i.e., the $M'(\sigma_{1:h-1})$ different values taken by $\ExpMulti_{\sigma_{1:h-1}}.\mathtt{curDistribution}$). For $h \geq 0$, we let $\btau(\sigma_{1:h}) := 1 + N \cdot \sum_{g=1}^{h} \sigma_g \cdot M^{d-g}$ denote the first round when $\sigma_{1:h}$ is encountered, and $\etau(\sigma_{1:h}) := \max\{T,  \btau(\sigma_{1:h}) +  N \cdot M^{d-h}-1\}$ denote the last round when $\sigma_{1:h}$ is encountered. Finally, write
    \begin{align}
      \MT(\sigma_{1:h}) =& \left\{ t \in [\btau(\sigma_{1:h}), \etau(\sigma_{1:h})] \ : \ h\^t-1 = h-1 \right\}\nonumber\\
      \ol{\MT}(\sigma_{1:h-1}) =&\bigcup_{\sigma_h = 0}^{M'(\sigma_{1:h-1})} \MT(\sigma_{1:h}) = \left\{ t \in [\btau(\sigma_{1:h-1}), \etau(\sigma_{1:h-1})] \ : \ h\^t - 1 = h-1 \right\}\nonumber.%
    \end{align}

    Notice that the only randomness used in $\BanditTreeSwap$ is the draws of $h\^t, a\^t$ in Lines \ref{line:sample-ht} and \ref{line:sample-at}. Accordingly, let $\MF\^t$ denote the sigma algebra generated by $\{ (h\^s, a\^s) \}_{s=1}^t$. We let $\E_t[\cdot]$ denote expectation conditioned on $\MF\^t$.    Given $t \in [T]$ for which the binary representation of $\lfloor \frac{t-1}{N} \rfloor$ is $\sigma = (\sigma_1, \ldots, \sigma_d)$, let $\bp\^t := \frac{1}{d} \sum_{h=1}^d \bx_{\sigma_{1:h-1}}\^{\sigma_h}$, so that, conditioned on $\MF\^{t-1}$, $a\^t \sim \bp\^t$. (This uses the sampling procedure for $a\^t$ on \cref{line:sample-at} as well as the definition of $\ExpMulti_{\sigma_{1:h-1}}.\mathtt{curDistribution}$ on \cref{line:curdistribution-bandit}.)

    First, we expand the total (unnormalized) reward of the learner over the $T$ rounds, as follows: 
    \begin{align}
      \sum_{t=1}^T \bu\^t[a\^t] =  \sum_{h=1}^d  \sum_{\sigma_{1:h-1} \in \Sigma_{h-1}} \sum_{\sigma_{h} = 0}^{M'(\sigma_{1:h-1})}   \sum_{s \in \MT(\sigma_{1:h})} \bu\^s[a\^s] %
      \label{eq:learner-utility-bandit}.
    \end{align}
    Now consider any function $\pi : [N] \to [N]$ and $\delta \in (0,1)$. For $\bu \in [0,1]^N$, define $(\pi \circ \bu) \in [0,1]^N$ by $(\pi \circ \bu)[a] := \bu[\pi(a)]$.  The learner's total reward under the swap function $\pi$ is given by $\sum_{t=1}^T \bu\^t[\pi(a\^t)]$. For each $t \in [T]$, we have that, $\E_{t-1} \left[ \bu\^t[\pi(a\^t)]\right] = \lng \bp\^t, \pi \circ \bu\^t\rng$. Thus, by the Azuma-Hoeffding inequality, we have that, for any $\delta \in (0,1)$, with probability $1-\delta$ over the randomness in $\BanditTreeSwap$,
    \begin{align}
\left| \sum_{t=1}^T  \bu\^t[\pi(a\^t)] - \lng \bp\^t, \pi \circ \bu\^t \rng \right| \leq C\sqrt{T \log(1/\delta)}\label{eq:ah-1},
    \end{align}
    for a sufficiently large constant $C$. \noah{check this}

    Next, we expand
    \begin{align}
      \sum_{t=1}^T \lng \bp\^t, \pi \circ \bu\^t \rng =& \frac 1d \sum_{h=1}^d \sum_{\sigma_{1:h-1} \in \Sigma_{h-1}} \sum_{\sigma_{h} = 0}^{M'(\sigma_{1:h-1})} \sum_{s = \btau(\sigma_{1:h})}^{\etau(\sigma_{1:h})} \lng \bx_{\sigma_{1:h-1}}\^{\sigma_h}, \pi \circ \bu\^s \rng\nonumber\\
      =&  \frac 1d \sum_{h=1}^d \sum_{\sigma_{1:h} \in \Sigma_h} \sum_{s = \btau(\sigma_{1:h})}^{\etau(\sigma_{1:h})} \lng \bx_{\sigma_{1:h-1}}\^{\sigma_h}, \pi \circ \bu\^s \rng \nonumber\\
      =& \sum_{\substack{s \in [T] \\ \sigma_{1:d} := \lfloor \frac{s-1}{N} \rfloor}} \frac 1d \sum_{h=1}^d  \lng \bx_{\sigma_{1:h-1}}\^{\sigma_h}, \pi \circ \bu\^s \rng\label{eq:banditpi-1}.
    \end{align}
    For each $\sigma_{1:h}$ and $s \in [\btau(\sigma_{1:h}), \etau(\sigma_{1:h})]$, we have that
    \begin{align}
\frac 1d \lng \bx_{\sigma_{1:h-1}}\^{\sigma_h}, \pi \circ \bu\^s \rng = \E_{s-1} \left[ \One\{h\^s - 1 \equiv h\} \cdot \lng \bx_{\sigma_{1:h-1}}\^{\sigma_h}, \pi \circ \bu\^s \rng \right]\nonumber,
    \end{align}
    where the statement $h\^s - 1 \equiv h$ is to be interpreted modulo $d$ (i.e., we have $0 \equiv d$). Thus, for each $s \in [T]$, letting the binary representation of $\lfloor \frac{s-1}{N} \rfloor$ be $\sigma_{1:d}$, we have
    \begin{align}
\frac 1d \sum_{h=1}^{d-1} \lng \bx_{\sigma_{1:h-1}}\^{\sigma_h}, \pi \circ \bu\^s \rng = \sum_{h=1}^{d-1}  \E_{s-1} \left[ \One\{h\^s - 1 \equiv h\} \cdot \lng \bx_{\sigma_{1:h-1}}\^{\sigma_h}, \pi \circ \bu\^s \rng \right]\nonumber.
    \end{align}
    Thus, by the Azuma-Hoeffding inequality, with probability $1-\delta$, \noah{check}
    \begin{align}
\left| \sum_{\substack{s \in [T]\\ \sigma_{1:d}:= \lfloor \frac{s-1}{N} \rfloor}} \frac 1d \sum_{h=1}^{d-1} \lng \bx_{\sigma_{1:h-1}}\^{\sigma_h}, \pi \circ \bu\^s \rng - \sum_{\substack{s \in [T]\\ \sigma_{1:d}:= \lfloor \frac{s-1}{N} \rfloor}} \sum_{h=1}^{d-1} \One\{ h\^s -1 = h \} \cdot \lng \bx_{\sigma_{1:h-1}}\^{\sigma_h}, \pi \circ \bu\^s \rng \right| \leq C \sqrt{T \log(1/\delta)}\label{eq:ah-2},
    \end{align}
    for a sufficiently large constant $C$. 
    We moreover have that
    \begin{align}
      &  \sum_{\substack{s \in [T]\\ \sigma_{1:d}:= \lfloor \frac{s-1}{N} \rfloor}} \sum_{h=1}^{d-1} \One\{ h\^s -1 = h \} \cdot \lng \bx_{\sigma_{1:h-1}}\^{\sigma_h}, \pi \circ \bu\^s \rng\nonumber\\
      =& \sum_{h=1}^{d-1} \sum_{\sigma_{1:h} \in \Sigma_h} \sum_{s = \btau(\sigma_{1:h})}^{\etau(\sigma_{1:h})} \One\{ h\^s - 1 = h\} \cdot \lng \bx_{\sigma_{1:h-1}}\^{\sigma_h}, \pi \circ \bu\^s \rng\nonumber\\
      = & \sum_{h=2}^d \sum_{\sigma_{1:h-1} \in \Sigma_{h-1}} \sum_{s = \btau(\sigma_{1:h-1})}^{\etau(\sigma_{1:h-1})} \One\{ h\^s - 1 = h-1\} \cdot \lng \bx_{\sigma_{1:h-2}}\^{\sigma_{h-1}}, \pi \circ \bu\^s \rng\nonumber\\
      =&  \sum_{h=2}^d \sum_{\sigma_{1:h-1} \in \Sigma_{h-1}} \sum_{s \in \ol\MT(\sigma_{1:h-1})}\lng \bx_{\sigma_{1:h-2}}\^{\sigma_{h-1}}, \pi \circ \bu\^s \rng \nonumber\\
      \leq & \sum_{h=2}^d \sum_{\sigma_{1:h-1} \in \Sigma_{h-1}} \max_{a^\st \in [N]} \sum_{s \in \ol\MT(\sigma_{1:h-1})} \bu\^s[a^\st]\label{eq:banditpi-2}.
    \end{align}

Fix any $\delta \in (0,1)$. It follows  by a union bound over all $\pi : [N] \to [N]$ in (\ref{eq:ah-1}) and (\ref{eq:ah-2}) as well as (\ref{eq:learner-utility-bandit}), (\ref{eq:banditpi-1}), and (\ref{eq:banditpi-2}) that under some event $\ME_1$ occurring with probability $1-\delta/3$, for all $\pi : [N] \to [N]$,
    \begin{align}
      & \sum_{t=1}^T \left( \bu\^t[\pi(a\^t)] - \bu\^t[a\^t] \right)\nonumber\\
      \leq & \sum_{t=1}^t \lng \bp\^t, \pi \circ \bu\^t \rng - \sum_{h=1}^d  \sum_{\sigma_{1:h-1} \in \Sigma_{h-1}} \sum_{\sigma_{h} = 0}^{M'(\sigma_{1:h-1})}   \sum_{s \in \MT(\sigma_{1:h})} \bu\^s[a\^s] + C\sqrt{TN \log(N/\delta)}\nonumber\\
      \leq & \frac{T}{d} + \sum_{h=2}^d \sum_{\sigma_{1:h-1} \in \Sigma_{h-1}} \left( \max_{a^\st \in [N]} \sum_{s \in \ol\MT(\sigma_{1:h-1})} \bu\^s[a^\st] - \sum_{s \in \ol\MT(\sigma_{1:h-1})} \bu\^s[a\^s] \right) + C\sqrt{TN \log(N/\delta)}\label{eq:swap-combine},
    \end{align}
    where we have also used that $\bu\^t \in [0,1]^N$ for all $t \in [T]$. (In particular, the first inequality above uses \cref{eq:learner-utility-bandit,eq:ah-1}, and the second inequality uses \cref{eq:banditpi-1,eq:ah-2,eq:banditpi-2}.)

    Note that the draws $h\^t \sim [d]$ in \cref{line:sample-ht} are all independent. Thus, for each sequence $\sigma_{1:h} \in \Sigma_{h}$, %
    we have that $|\MT(\sigma_{1:h})| \sim \Bin(NM^{d-h}, 1/d)$. Thus, for any such sequence $\sigma_{1:h}$, it holds by a Chernoff bound that
    \begin{align}
\Pr\left( |\MT(\sigma_{1:h})| \leq  \frac{2NM^{d-h}}{d} \right) \geq 1 - e^{-NM^{d-h}/(3d)} \geq 1-e^{-N/(3d)}\nonumber,
    \end{align}
    and in the event that $\etau(\sigma_{1:h}) - \btau(\sigma_{1:h}) = N \cdot M^{d-h} - 1$,
    \begin{align}
\Pr\left( |\MT(\sigma_{1:h})| \in \left[ \frac{NM^{d-h}}{2d}, \frac{2NM^{d-h}}{d} \right]\right) \geq 1 - 2e^{-NM^{d-h}/(3d)} \geq 1-2e^{-N/(3d)}\nonumber,
    \end{align}
    By a union bound over $h \in [d]$ and $\sigma_{h} \in \Sigma_{h}$ for which $\etau(\sigma_{1:h}) - \btau(\sigma_{1:h}) = N \cdot M^{d-h} - 1$, we have that, under some event $\ME_2$ occurring with probability at least $1-\frac{2dT}{N} e^{-N/(3d)}$, for all such sequences $\sigma_{1:h}$, $|\MT(\sigma_{1:h})| \leq \frac{2NM^{d-h}}{d}$, and for $\sigma_{1:h}$ satisfying $\etau(\sigma_{1:h}) - \btau(\sigma_{1:h}) = N \cdot M^{d-h} - 1$,  $|\MT(\sigma_{1:h})| \in [ \frac{NM^{d-h}}{2d} , \frac{2NM^{d-h}}{d}]$. 

    Then by Lemma \ref{lem:exp3-full}\footnote{Note that technically, \cref{lem:exp3-full} applies to the main procedure in $\ExpMulti$ (\cref{alg:exp3-multi}), which is not directly called in  $\BanditTreeSwap$; however, note that this procedure is simulated exactly by the instances $\ExpMulti_{\sigma_{1:h-1}}$ in $\BanditTreeSwap$, where the quantities $K\^s$ in $\ExpMulti$ correspond to the number of time steps during each round of each $\ExpMulti_{\sigma_{1:h-1}}$ instance that $h\^t =h$.} with $T = M$ and $K = \frac{2NM^{d-h}}{d}$  (which are the parameters that the instance $\ExpMulti_{\sigma_{1:h-1}}$ was initialized with on \cref{line:bandit-params} of \cref{alg:bandit-treeswap}) \noah{todo say that it's ok to apply the lemma since each exp3 alg receives a random oblivious sample} and a union bound, there is some event $\ME_3$ with probability at least $1-\delta/3$ so that, under $\ME_2 \cap \ME_3$, for all $h \in [d]$ and $\sigma_{1:h-1} \in \Sigma_{h-1}$  for which $\etau(\sigma_{1:h-1}) - \btau(\sigma_{1:h-1}) = N \cdot M^{d-h+1} - 1$, 
    \begin{align}
\max_{a^\st \in [N]} \sum_{s \in \ol\MT(\sigma_{1:h-1})} \bu\^s[a^\st] - \sum_{s \in \ol\MT(\sigma_{1:h-1})} \bu\^s[a\^s] \leq &C \cdot \frac{N M^{d-h}}{d} \cdot M^{5/6} \cdot \log^2(TN/\delta) + N \cdot M^{5/6}\label{eq:use-exp}.
    \end{align}
    Moreover, note that for each $2 \leq h \leq d-1$, there is at most a single sequence $\sigma_{1:h-1} \in \Sigma_{h-1}$ for which $\etau(\sigma_{1:h-1}) - \btau(\sigma_{1:h-1}) <  N \cdot M^{d-h+1} - 1$. For any such sequence $\sigma_{1:h-1}$, we may bound
    \begin{align}
      \label{eq:extra-mass-bound}
\max_{a^\st \in [N]} \sum_{s \in \ol\MT(\sigma_{1:h-1})} \bu\^s[a^\st] - \sum_{s \in \ol\MT(\sigma_{1:h-1})} \bu\^s[a\^s] \leq |\ol\MT(\sigma_{1:h-1})| \leq \frac{2NM^{d-h+1}}{d},
    \end{align}
    where the final inequality holds under $\ME_2$. 
    
    Combining (\ref{eq:swap-combine}), (\ref{eq:use-exp}), and \cref{eq:extra-mass-bound}, we see that, under $\ME_1 \cap \ME_2 \cap \ME_3$ (which occurs with probability at least $1-\delta - \frac{2dT}{N} e^{-N/(3d)}$),
    \begin{align}
      & \max_{\pi : [N] \to [N]} \sum_{t=1}^T \left( \bu\^t[\pi(a\^t)] - \bu\^t[a\^t] \right)\nonumber\\
      \leq & \frac{T}{d} +C \sqrt{TN \log(TN/\delta)}\nonumber\\
      & +  \sum_{h=2}^d \left( \frac{2NM^{d-h+1}}{d} +  \sum_{\scriptsize\substack{\sigma_{1:h-1} \in \Sigma_{h-1}:\\ \etau(\sigma_{1:h-1}) - \btau(\sigma_{1:h-1}) = NM^{d-h+1}}} \left(C \cdot \frac{NM^{d-h}}{d} \cdot M^{5/6} \log^2(TN/\delta) + N \cdot M^{5/6}\right)\right)\nonumber\\
      \leq  &  \frac{T}{d} +C \sqrt{TN \log(N/\delta)} + \frac{2NM^{d-1}}{d} +  C \cdot T \cdot M^{-1/6} \log^2(TN/\delta) + TM^{-1/6} \cdot \sum_{h=2}^d M^{h-d}\nonumber\\
      \leq & \frac{3T}{d} + C \sqrt{TN \log(TN/\delta)} + C \cdot T \cdot \frac{\log^2(TN/\delta)}{M^{1/6}}\nonumber,
    \end{align}
    where the second inequality uses that $\sum_{h=2}^d M^{d-h+1} \leq M^{d-1}$, and the final inequality uses that $NM^{d-1} \leq T$ by assumption and that $\sum_{h=2}^d M^{h-2} \leq 2$. Choosing $\delta = 1/T$ yields that
    \[
T \cdot \left( \delta + \frac{2dT}{N} e^{-N/(3d)} \right) \leq 1 + \frac{2dT^2}{N} e^{-N/(3d)},
    \]
    and thus the claimed expected swap regret bound holds. 
    
  \end{proof}

  \paragraph{Analysis of $\ExpMulti$. }
  Next we state the main external regret guarantee for $\ExpMulti$. It states that when each $K\^t$ is within a constant factor of some parameter $K$, then the external regret of $\ExpMulti$ is bounded by the product of $K$ and a quantity which is sublinear in the number of rounds $T$. 
  \begin{lemma}
    \label{lem:exp3-full}
    Let $N, T, K \in \mathbb{N}$ be given, and let $K\^1, \ldots, K\^T$ be fixed so that for all $t\in [T]$, $K\^t \in [K/4, K]$. let $(\bu_k\^t \in [0,1]^N)_{t \in [T], k \in [K\^t]}$ be a fixed sequence of reward vectors. Then for any $\delta \in (0,1)$, with probability at least $1-\delta$, the actions $a_k\^t$ of $\mathtt{Exp3Multi}(N, T, T^{-1/2}, K^{-1} T^{-1/6}, K)$ (\cref{alg:exp3-multi}) satisfy
    \begin{align}
\max_{a^\st \in [N]} \sum_{t=1}^T \sum_{k=1}^{K\^t} \bu_k\^t[a^\st] - \bu_k\^t[a_k\^t] \leq K \cdot O \left( \log^2(TN/\delta) \cdot T^{5/6}\right) + N \cdot T^{5/6} \nonumber. %
    \end{align}
  \end{lemma}
  \begin{proof}
Let $\eta = \frac{1}{\sqrt{T}}, \gamma = \frac{1}{KT^{1/6}}$ be the parameters passed to $\mathtt{Exp3Multi}$. 
    
    Recall from the definitions in Algorithm \ref{alg:exp3-multi} that $\hat L_a\^t = \sum_{s=1}^t \hat Y_a\^s$ and $\hat Y_a\^t := \frac{1}{K} \cdot \sum_{k=1}^{K\^t} \left( \frac{\mathbbm{1}\{a_k\^t = a\} \cdot (1-u_k\^t)}{ \bp\^t[a] + \gamma} \right)$ for each $t \in [T]$. We also define $\hat L\^t =\sum_{s=1}^t \sum_{a=1}^N  \bp\^s[a] \cdot \hat Y_a\^s$. Additionally, we define $\by\^t_k = \mathbf{1} - \bu\^t_k$, $y_k\^t = 1-u_k\^t$, and 
    \begin{align}
\by\^t := \frac 1K \sum_{k=1}^{K\^t} \by_k\^t, \qquad y\^t =\frac 1K  \sum_{k=1}^{K\^t} y_k\^t\nonumber
    \end{align}
    to denote the total \emph{loss} vectors and realized \emph{losses} in each time step $t \in [T]$, as well as
    \begin{align}
\tilde L\^t = \sum_{s=1}^t y\^s, \qquad L\^t_{a} = \sum_{s=1}^t \by\^s[a], \qquad  R\^t_{a} = \sum_{s=1}^t y\^s - \sum_{s=1}^t \by\^s[a] = \tilde L\^t - L\^t_{a}\nonumber
    \end{align}
    to denote, respectively, the cumulative loss up to $t$ experienced by the learner, the cumulative loss up to $t$ experienced by taking action $a$, and the cumulative regret associated with action $a$ up to time step $t$.  Finally, for each $k \in [K\^t]$, write $\hat Y_a\^{t,k} := \frac{\One\{a_k\^t = a\} \cdot (1-u_k\^t)}{\bp\^t[a] + \gamma}$, so that $\hat Y_a\^t = \frac{1}{K} \sum_{k=1}^{K\^t} \hat Y_a\^{t,k}$.

\paragraph{Step 1: Regret decomposition.}    We consider the following decomposition of $\hat R\^t_{a}$:
    \begin{align}
      \label{eq:reg-decomposition}
\hat R\^t_{a} = (\tilde L\^t - \hat L\^t) + (\hat L\^t - \hat L\^t_a) + (\hat L\^t_a - L\^t_a).
    \end{align}
    We bound each of the terms in (\ref{eq:reg-decomposition}) in the below lemmas, whose proofs are provided in \cref{sec:bandit-lemmas-proofs}. 
    \begin{lemma}
      \label{lem:hedge-term}
Let $\delta \in (0,1)$ be given. Then with probability at least $1-\delta$,  for each $t \in [T]$ and $a \in [N]$, 
      \begin{align}
\hat L\^t - \hat L\^t_a \leq \frac{\log N}{\eta} + 2\eta t + \frac{2\eta t \log^2(TN/\delta)}{(K\gamma)^2}\nonumber. %
      \end{align}
    \end{lemma}

    \begin{lemma}
      \label{lem:concentration-var}
      There is a sufficiently large constant $C$ so that for any $\delta \in (0,1)$, with probability at least $1-\delta$, %
      \begin{align}
\max_{a \in [N]} (\hat L\^T_a - L\^T_a) \leq  \frac{C\log(TN/\delta)}{K\gamma} \cdot \sqrt{T \log(N/\delta)}\nonumber. %
      \end{align}
    \end{lemma}

    \begin{lemma}
      \label{lem:concentration-bias}
For all $t \in [T]$, it holds that $\tilde L\^t - \hat L\^t = \gamma \sum_{a=1}^N \hat L\^t_a$. 
\end{lemma}

\paragraph{Step 2: putting it all together.} 
Using \cref{lem:hedge-term,lem:concentration-var,lem:concentration-bias}, we may now bound $\hat R\^t_a$ using (\ref{eq:reg-decomposition}) as follows: with probability at least $1-2\delta$, we have that for all $a \in [N]$, 
\begin{align}
  R\^T_a\leq & \frac{\log N}{\eta} + 2\eta T + \frac{2\eta T \log^2(TN/\delta)}{(K\gamma)^2} + \frac{C \log(TN/\delta)}{K\gamma} \cdot \sqrt{T \log(N/\delta)} + \gamma \sum_{a=1}^N \hat L_a\^T\nonumber\\
  \leq & \frac{\log N}{\eta} + 2\eta T + \frac{2\eta T \log^2(TN/\delta)}{(K\gamma)^2} + \frac{C \log(TN/\delta)}{K\gamma} \cdot \sqrt{T \log(N/\delta)} +\frac{ C \log(TN/\delta) \sqrt{T \log(N/\delta)}}{K} + \gamma NT\nonumber,
\end{align}
where the final inequality uses Lemma \ref{lem:concentration-var} again and the fact that $L_a\^T \leq T$ for all $a \in [N]$. By our choices of $\eta = 1/\sqrt{T}$ and $\gamma = \frac{1}{KT^{1/6}}$, we obtain that, with probability $1-2\delta$, 
\begin{align}
R_a\^t \leq C \log^2(TN/\delta) \cdot T^{5/6} + T^{5/6} \cdot \frac{N}{K}\nonumber.
\end{align}
The proof is completed by noting that $R_a\^t = \frac 1K \cdot \sum_{t=1}^T \sum_{k=1}^{K\^t} \bu_k\^t[a] - \bu_k\^t[a_k\^t]$.  
\end{proof}

\subsubsection{Proofs of lemmas}
\label{sec:bandit-lemmas-proofs}
In this section, we prove \cref{lem:hedge-term,lem:concentration-var,lem:concentration-bias}. We begin with the following lemma establishing concentration for the $\hat Y_a\^t$ values in each round $t$.
\begin{lemma}
  \label{lem:haty-y-conc}
  Fix $\delta \in (0,1)$. Then with probability at least $1-\delta$, for all $s \in [T]$ and $a \in [N]$,
  \begin{align}
\hat Y\^s_a  - \by\^s[a] \leq \frac{\log(TN/\delta)}{2K \gamma}\nonumber.
  \end{align}
\end{lemma}
\begin{proof}
For $s \in [T]$ and $k \in [K\^s]$, let $\MF\^{s,k}$ denote the sigma-algebra generated by all random variables prior to the beginning of step $k$ of round $s$ in $\ExpMulti$ (\cref{alg:exp3-multi}; so in particular, $\bp\^s$ is $\MF\^{s,1}$-measurable). Consider any fixed $s \in [T]$ and $a \in [N]$. Then for each $k \in [K\^s]$,
    \begin{align}
\E\left[ \frac{\One\{a_k\^s = a\} \cdot y_k\^s}{\bp\^s[a]} \mid \MF\^{s,k}\right] = \frac{\bp\^s[a] \cdot \by\^s[a]}{\bp\^s[a]} = \by\^s[a]\label{eq:tily-unbiased}.
    \end{align}
    We next apply Lemma \ref{lem:concentration-gamma} with $T = K\^s$ to the sequence of random variables $\left(\frac{\One\{a_k\^s = a\} \cdot \by\^s_k[a]}{\bp\^s[a]}\right)_{k=1}^{K\^s}$. The values of the parameters are defined as follows: all values of $y\^t$ are set to $\by\^s[a]$, all parameters $\lambda\^t$ are set to $\frac{\gamma}{\bp\^s[a]}$, and all parameters $\alpha\^t$ are set to $2\gamma$. The precondition is satisfied by (\ref{eq:tily-unbiased}). 
  Then  under some event $\ME_{a,s}$ that occurs with probability $1-\delta/(NT)$, %
    \begin{align}
      \sum_{k=1}^{K\^s} (\hat Y_a\^{s,k} - \by_k\^s[a]) =& \sum_{k=1}^{K\^s} \left(\frac{1}{1 + \frac{\gamma}{\bp\^s[a]}} \cdot \frac{\One\{a_k\^s = a\} \cdot \by_k\^s[a]}{\bp\^s[a]} - \by_k\^s[a] \right) \nonumber\\
      \leq &  \frac{\log(TN/\delta)}{2\gamma}\nonumber.
    \end{align}
The lemma statement follows by a union bound over $s$ and $a$, as well as the fact that $\hat Y_a\^s = \frac 1K \sum_{k=1}^K \hat Y_a\^{s,k}$ and $\by\^s[a] = \frac 1K \sum_{k=1}^K \by_k\^s[a]$. 
\end{proof}

Next we prove \cref{lem:concentration-var}, which establishes concentration of the values $\hat L_a\^t$ to $L_a\^t$.
\begin{proof}[Proof of Lemma \ref{lem:concentration-var}]
  Write $R :=  \frac{\log(TN/\delta)}{2K\gamma} $. 
  We have $\hat L_a\^T = \sum_{t=1}^T \hat Y_a\^t$ and $L_a\^T = \sum_{t=1}^T \by\^t[a]$. Let $\BI\^t_a := \One \left\{ \hat Y_a\^t - \by\^t[a] \leq R \right\}$. Note that $\hat Y_a\^t \leq K\^t/\gamma$ for each $a \in [N], t \in [T]$. The Azuma-Hoeffding inequality gives that, with probability at least $1-\delta/N$, 
  \begin{align}
    \hat L_a\^t - L_a\^t =& \sum_{t=1}^T (\hat Y_a\^t - \by\^t[a]) \nonumber\\
    \leq & \sum_{t=1}^T \BI_a\^t \cdot (\hat Y_a\^t- \by\^t[a]) + \sum_{t=1}^T \frac{K\^t}{\gamma} \cdot (1-\BI_a\^t)\nonumber\\
    \leq & CR \sqrt{T \log(N/\delta)} + \sum_{t=1}^T \frac{K\^t}{\gamma} \cdot (1- \BI_a\^t)\label{eq:hatla-la},
  \end{align}
  for a sufficiently large constant $C$. Azuma-Hoeffding is applied in the following manner: let $\MF\^t$ denote the sigma-algebra generated by all random variables up to the end of round $t$. Then $\BI_a\^t \cdot (\hat Y_a\^t - \by\^t[a]) \in [-1,R]$ almost surely, and for each $t$,
  \begin{align}
    \E\left[ \BI_a\^t \cdot (\hat Y_a\^t - \by\^t[a]) \mid \MF\^{t-1} \right] \leq &  \E\left[ \hat Y_a\^t - \by\^t[a] \mid \MF\^{t-1} \right]\nonumber\\
    \leq &  \frac 1K \sum_{k=1}^{K\^t}\left( \E \left[ \frac{\One\{a_k\^t = a\} \cdot \by_k\^t[a]}{\bp\^t[a]} \mid \MF\^{t-1}\right] - \by\^t_k[a]\right)=0\nonumber.
  \end{align}
  The above verifies that the assumptions of Azuma-Hoeffing are satisfied for the random variables $\BI_a\^t \cdot (\hat Y_a\^t - \by\^t[a])$ (i.e., they form a supermartingale with respect to the filtration $\MF\^t$). Taking a union bound in (\ref{eq:hatla-la}) over all $a \in [N]$ and using the fact that, by Lemma \ref{lem:haty-y-conc}, $\BI_a\^t = 1$ for all $a$ and $t$ with probability at least $1-\delta$, we conclude that with probability at least $1-2\delta$, $\max_a (\hat L_a\^t - L_a\^t) \leq CR \sqrt{T \log(N/\delta)}$, concluding the proof of the lemma. 
\end{proof}

\begin{proof}[Proof of Lemma \ref{lem:concentration-bias}]
     We compute
     \begin{align}
       \tilde L\^t - \hat L\^t =&  \sum_{s=1}^t y\^s - \sum_{s=1}^t \sum_{a=1}^N \hat \bp\^s[a] \cdot \hat Y_a\^s \nonumber\\
       =& \frac 1K \sum_{s=1}^t \left( \sum_{k=1}^{K\^s}  y_k\^s - \sum_{a=1}^N \hat \bp\^s[a] \cdot \sum_{k=1}^{K\^s} \left( \frac{\One\{a_k\^s = a\} \cdot y_k\^s}{\bp\^s[a] + \gamma}\right)\right)\nonumber\\
       = & \frac 1K \sum_{s=1}^t \left(\sum_{k=1}^{K\^s} y_k\^s - \sum_{a=1}^N \left( 1 - \frac{\gamma}{\bp\^s[a] + \gamma} \right) \cdot \sum_{k=1}^{K\^s}  \left( \One\{a_k\^s = a\} \cdot y_k\^s\right)\right)\nonumber\\
       = & \frac{\gamma}{K} \sum_{s=1}^t  \sum_{a=1}^N \sum_{k=1}^{K\^s} \frac{\One\{a_k\^s = a\} \cdot y_k\^s}{\bp\^s[a] + \gamma}\nonumber\\
       =& \gamma \sum_{s=1}^t \sum_{a=1}^N \hat Y_a\^s = \gamma \sum_{a=1}^N \hat L_a\^t, \nonumber
     \end{align}
     as desired. 
  \end{proof}

  Finally, we prove \cref{lem:hedge-term}, which uses the definition of the exponential weights updates to establish that $\hat L\^t- \hat L_a\^t$ is bounded for each $a,t$. 
  \begin{proof}[Proof of Lemma \ref{lem:hedge-term}]
    For each $a \in [N], t \in [T]$, define $\hat X_a\^t = 1-\hat Y_a\^t$, $\hat S_a\^t = \sum_{s=1}^t \hat X_a\^s$, and $\hat S\^t = \sum_{s=1}^t \sum_{a=1}^N \bp\^s[a] \cdot \hat X_a\^s$. Note that, for all $t \in [T]$ and $a \in [N]$, $\hat S_a\^t = t - \hat L_a\^t$. 
    Then from \cref{line:exp3-ew-update} of \cref{alg:exp3-multi}, we have
    \begin{align}
\bp\^t[a] = \frac{\exp \left( \eta \hat S_a\^{t-1}\right)}{\sum_{b=1}^N \exp \left( \eta \hat S_b\^{t-1} \right)}\nonumber.
    \end{align}
    Define $W\^t := \sum_{a=1}^N \exp(\eta \hat S_a\^t)$, so that $W\^0 = N$. For each $t \in [T]$, we have
    \begin{align}
      \frac{W\^t}{W\^{t-1}} =& \sum_{a=1}^N \frac{\exp(\eta \hat S_a\^{t-1}) \cdot \exp(\eta \hat X_a\^t)}{W\^{t-1}}\nonumber\\
      =& \sum_{a=1}^N \bp\^t[a] \cdot \exp(\eta) \cdot \exp(\eta (\hat X_a\^t-1))\nonumber\\
      \leq & \exp(\eta) \cdot \sum_{a=1}^N \bp\^t[a]\cdot \left( 1 + \eta (\hat X_a\^t-1) + \frac{\eta^2}{2}(\hat X_a\^t-1)^2 \right)\nonumber\\
      \leq & \exp(\eta) \cdot \exp \left( \sum_{a=1}^N \eta \bp\^t[a] (\hat X_a\^t-1) + \sum_{a=1}^N \frac{\eta^2}{2} \bp\^t[a] (\hat X_a\^t-1)^2 \right)\nonumber\\
      =&  \exp \left( \sum_{a=1}^N \eta \bp\^t[a] X_a\^t + \sum_{a=1}^N \frac{\eta^2}{2} \bp\^t[a] (\hat X_a\^t-1)^2 \right)\nonumber
    \end{align}
    where the first inequality uses that $\eta (\hat X_a\^t-1) \leq 0$ since $\hat X_a\^t = 1-\hat Y_a\^t \leq 1$, and the fact that $\exp(x) \leq 1 + x + \frac{x^2}{2}$ for $x \leq 0$. Then
    \begin{align}
      \exp \left( \eta \hat S_a\^t \right) \leq W\^t =&  W\^0 \cdot \frac{W\^1}{W\^0} \cdots \frac{W\^t}{W\^{t-1}} \leq N \cdot \exp \left(  \eta \sum_{s=1}^t \sum_{a=1}^N \bp\^s[a] \hat X_a\^s + \eta^2 \sum_{s=1}^t \sum_{a=1}^N \bp\^s[a] (\hat X_a\^s-1)^2 \right)\nonumber\\
      =& N \cdot \exp \left(  \eta \hat S\^t + \eta^2 \sum_{s=1}^t \sum_{a=1}^N \bp\^s[a] (\hat X_a\^s-1)^2 \right)\nonumber.
    \end{align}
    Taking the logarithm and rearranging gives that
    \begin{align}
\hat L\^t - \hat L\^t_a = \hat S\^t_a - \hat S\^t \leq \frac{\log N}{\eta} + \eta \sum_{s=1}^t \sum_{a=1}^N \bp\^s[a] (\hat Y_a\^s)^2\label{eq:lhat-lhata}.
    \end{align}

    Note that $\hat Y_a\^{s,k} \geq 0 $ for all $a,s,k$ and thus $\hat Y_a\^s \geq 0$ for all $a,s$. Then under the event $\ME$ of Lemma \ref{lem:haty-y-conc}, which occurs with probability at least $1-\delta$, for all $s \in [T]$ and $a \in [N]$, %
    \begin{align}
      \bp\^s[a] \cdot (\hat Y_a\^s)^2 \leq & %
      \bp\^s[a] \cdot \left( 2 (\by\^s[a])^2 + \frac{2 \log^2(TN/\delta)}{(K\gamma)^2}\right)\nonumber.
    \end{align}
    Using that $\by\^s[a] \in [0,1]$ for all $s,a$ and combining the above display with (\ref{eq:lhat-lhata}) yields that, under $\ME$, 
    \begin{align}
\hat L\^t - \hat L\^t_a \leq \frac{\log N}{\eta} + \eta \sum_{s=1}^t \sum_{a=1}^N \bp\^s[a] \cdot \left( 2 + \frac{2 \log^2(TN/\delta)}{(K\gamma)^2}\right)\leq \frac{\log N}{\eta} + 2\eta t + \frac{2\eta t \log^2(TN/\delta)}{(K\gamma)^2}\nonumber,
    \end{align}
    as desired.
  \end{proof}
  
The following lemma establishes concentration of a martingale with potentially heavy tails. 
\begin{lemma}[Lemma 12.2 of \cite{lattimore2020bandit}]
  \label{lem:concentration-gamma}
  Let $(y\^t)_{t=1}^T$ denote a fixed sequence of real numbers, let $(\MF\^t)_{t=1}^T$ be a filtration, and let $\tilde Y\^t$ be a real-valued sequence adapted to the filtration $\MF\^t$. Suppose that $\E[\tilde Y\^t \mid \MF\^{t-1}] = y\^t$ for all $T \in [T]$. 

  Moreover, let $(\alpha\^t)_{t \in [T]}$ and $(\lambda\^t)_{t \in [T]}$ be real-valued $\MF\^t$-predictable sequences of random variables so that for all $a,t$, it holds that $0 \leq \alpha\^t \tilde Y\^t \leq 2 \lambda\^t$. Then for all $\delta \in (0,1)$,
  \begin{align}
\mathbb{P} \left( \sum_{t=1}^T  \alpha\^t \cdot \left(\frac{\tilde Y\^t}{1 + \lambda\^t}- y\^t \right) \geq \log(1/\delta) \right) \leq \delta\nonumber.
  \end{align}
\end{lemma}

  }

\clearpage

\section{An oblivious lower bound on swap regret}
\label{sec:oblivious-lb}

In this section, we make a slight adjustment to notation, considering an adversary that selects, at each time step $t$, a reward vector $\vut \in [0,1]^N$ as opposed to a reward function $\vft: [N] \to [0,1]$.  These are functionally identical, and it's a change we make only to match the notational choice of classic results in this space.

\begin{theorem}\label{thm:lower-main}
    Given $T$ rounds of online learning over $N$ actions, there exists a randomized oblivious adversarial strategy $\mathcal{D} \in \Delta \p{[0,1]_N^{T}}$ that forces all learners to incur $\tilde \Omega\p{\min\set{1, \sqrt{N/T}}}$ swap regret, in expectation over the sampled adversarial actions $\vu\^{1:T} \sim \mathcal{D}$.  Any learner strategy in this setting can be viewed as a collection of maps that, at each time step $t$, maps the observed $\vu\^{1:t-1}$ to an action.  That is, for all $t \in [T]$, and all $\vxt: [0,1]_N^{t-1} \to \Delta_N$, we have
    \begin{equation*}
        \EE_{\vu\^{1:T} \sim \mathcal{D}}\ps{\DSR\p{\vx\^{t \in 1:T}\p{\vu\^{1:t-1}},\vu\^{1:T}}} = \Omega \p{\min\p{\frac{1}{\log^5(T)},\frac{\sqrt{N/T}}{\log^5(N)}}}
    \end{equation*}
\end{theorem}

\subsection{Construction of the adversary}
The adversary will produce rewards in $2^D$ batches of size $B$.
Each batch of rewards will correspond to a root-to-leaf path in a complete binary tree as follows.  Consider a complete binary tree containing $2^D$ leaf nodes.  That is, a tree of depth $D$, 
having $2^{D+1}-1$ total nodes.
We will associate each node $v$ of this tree with 2 unique actions: $a_v,\dot{a}_{v} \in [N]$.  This necessitates $2(2^{D+1}-1) \leq N$.
We label the children of each internal node as `left child' and `right child'.  This will yield an enumeration of the leaves of the tree (according to the natural DFS from the root, which explores the left child first). According to this enumeration, we denote the leaves by $\ell_1,\cdots,\ell_{2^D}$.  In each batch $b \in [2^D]$, the rewards $\vut$ will be supported on actions $a_v,\dot{a}_v$ with $v$ on the root-to-$\ell_b$ path.  We denote this path by $P_b = (v_{(b,0)}-v_{(b,1)}-\cdots-v_{(b,D)})$, where $v_{(b,0)}$ is the root and $v_{(b,D)}$ is the leaf $\ell_b$.  The reward of all other actions will be $0$.  In fact, for each internal node $v\in P_b$, $v \ne \ell_b$, only one of the two actions $a_{v},\dot{a}_{v}$ will have non-zero reward during the batch. We denote this rewarded action $\mathring{a}_{v}(b) \in \set{a_{v},\dot{a}_{v}}$, and it will be determined in terms of a Bernoulli random variable $r_v$ as follows:
\begin{equation*}
    \mathring{a}_{v}(b) = \begin{cases}
        a_{v} &\text{ if $\ell_b$ is a left descendant of $v$}\\
        a_{v} &\text{ if $\ell_b$ is a right descendant of $v$ and $r_{v}=0$}\\
        \dot{a}_{v} &\text{ if $\ell_b$ is a right descendant of $v$ and $r_{v}=1$}
    \end{cases}
\end{equation*}
where $\Pr[r_{v}=1] = 1/2$ i.i.d. for all internal nodes $v$. Informally, the rewarded action will be $a_v$ throughout all batches $b$ where the leaf $\ell_b$ is a left descendant of $v$, and it will switch to $\dot{a}_v$ with probability $1/2$ once $\ell_b$ becomes a right descendant of $v$. 

Similarly for the leaves, at each time step $t$, only one of the actions $\mathring{a}_{\ell_b}(t) \in \set{a_{\ell_b},\dot{a}_{\ell_b}}$ receives reward.  However, unlike the internal nodes, this rewarded action is chosen independently at each time step $t$, rather than being constant for the entirety of batch $b$.  $\mathring{a}_{\ell_b}(t)$ is determined by a Bernoulli random variable $r_{(\ell_{b},t)}$ with $\Pr[r_{(\ell_{b},t)}=1]=1/2$ i.i.d. as follows 
\begin{equation*}
    \mathring{a}_{\ell_b}(t) = \begin{cases}
        {a}_{\ell_b} &\text{ if $r_{(\ell_{b},t)}=0$}\\
        \dot{a}_{\ell_b} &\text{ if $r_{(\ell_{b},t)}=1$} 
    \end{cases}
\end{equation*}

All of these Bernoulli random variables will be sampled before the first round of learning; their values will initially be unknown to the learner; and they will be the only source of randomness the adversary relies on.

Now, we are ready to define the rewards chosen by the adversary.  For all batches $b$, for each time step $t$ in batch $b$, we have

\begin{equation}\label{eq:adv}
    \begin{aligned}
        &\vut[\mathring{a}_{v_{(b,d)}}(b)] &&= \frac{d}{2D} &&\quad\text{for $d<D$ (internal nodes in $P_b$)}\\
        &\vut[\mathring{a}_{\ell_b}(t)] &&= 1 &&\quad\text{(the leaf of $P_b$)}\\
        &\vut[a] &&= 0 &&\quad\text{(for all other actions)}  
    \end{aligned}
\end{equation}

Lastly, we must have $4\cdot 2^D-2 \leq N$ and also $2^D \leq B\cdot 2^D \leq T$.  So, we choose 
\begin{align*}
    D &= \left \lfloor \log_2 \min\p{T,\frac{N+2}{4}}\right\rfloor & B &= \left \lfloor \frac{T}{2^D} \right\rfloor
\end{align*}
Note, $B=1$ for $T \leq 4N-2$.  Additionally, for $t> T' = B2^D$, we simply have $\vut = \vect{0}$.  Note, $T' \geq \frac{B}{B+1}T \geq T/2$, so not utilizing these rounds will only impact the swap regret by a constant factor.

Also, importantly, at every single time step $t\leq T'$,
$\norm{\vut}_1 = \sum_{d=1}^{D-1} \frac{d}{2D} + 1 = \frac{D+3}{4}$.  Running the same algorithm scaling each reward by $\frac{4}{D+3}$ gives an algorithm for which $\norm{\vut}_1 \leq 1$ for all $t$ that achieves the following:

\begin{corollary}\label{cor:lower-main}
    Given $T$ rounds of online learning over $N$ actions, there exists a randomized oblivious adversarial strategy $\mathcal{D} \in \Delta \p{S_N^{T}}$ where $S_N = \set{\vu \in [0,1]^N \middle| \norm{\vu}_1 \leq 1}$ that forces all learners to incur $\tilde \Omega\p{\min\set{1, \sqrt{N/T}}}$ swap regret, in expectation over the sampled adversarial actions $\vu\^{1:T} \sim \mathcal{D}$.  That is, for all $t \in [T]$, $\vxt: S_N^{t-1} \to \Delta_N$, we have
    \begin{equation*}
        \EE_{\vu\^{1:T} \sim \mathcal{D}}\ps{\DSR\p{\vx\^{t \in 1:T}\p{\vu\^{1:t-1}},\vu\^{1:T}}} = \Omega \p{\min\p{\frac{1}{\log^6(T)},\frac{\sqrt{N/T}}{\log^6(N)}}}
    \end{equation*}
\end{corollary}

\subsection{Proof of \cref{thm:lower-main}}
\subsubsection{Outline}
    We introduce the following notation for the time steps of learning.  For a node $v$, define $\tl_v$ to be the first time step of the first batch in which $v$ appears on the root-to-$\ell_b$ path.  If that is batch $b$, then $\tl_v = (b-1)B+1$.  Similarly, let $\tr_v$ be the last time step of the last batch in which $v$ appears on the root-to-$\ell_b$ path.  If that is batch $b$, then $\tr_v = bB$.\\

    We also introduce the following string-concatenation notation to index the nodes of the tree. For a non-leaf node $v$, let $vL$ and $vR$ denote the left and right child of $v$ respectively.  This enables us to index grandchildren as such: $vLL,vLR,vRL,vRR$.  Under this notation, we also refer to the root by the empty string: $\roo$.  As a sanity check, realize $\tl_{v}=\tl_{vL}$ and $\tr_{v}=\tr_{vR}$.  Also observe that, for $t \leq \tr_{vL}$, the reward $\vut$ is independent from Bernoulli random variable $r_v$.  $\tr_{vL}$ is the final time step before we switch to a batch $b$ where $\ell_b$ is a right descendant of $v$, and $\mathring{a}_v(b)$ is dependent on $r_v$.  
    At every time step $t$, the learner selects $\vxt$ as a function of the observed rewards $\vu\^{1:t-1}$. The rewards $\vu\^{1:t-1}$ only depend on the set of Bernoulli random variables that have been observed before time $t$.  %
    The set $r^{(t)}$ contains all these variables, defined formally as:
    \begin{equation*}
        r\^t = \set{r_v \middle | \text{ non-leaf $v$ s.t. } \tr_{vL} < t} \cup \set{r_{(\ell,\tau)} \middle | \text{ leaves $\ell$ and time $\tau<t$}}
    \end{equation*}
    Fixing the strategy of the learner, for every $t$, we can view the learner's action $\vxtr = \vxt(\vu\^{1:t-1}(r\^t))$ as a function of the random variables $r\^t$.  %
    This function is deterministic since there is no reason for the learner to randomize against an oblivious adversary.

    We have
    \begin{align*}
        \EE_{r\^T}\ps{\DSR\p{\vx\^{t \in 1:T}_{r\^{t}},\vu\^{1:T}}} &= \frac{1}{T}\sum_{a \in [N]} \EE_{r\^T}\ps{\max_{a' \in [N]} \Swap{a}{a'}}
    \end{align*}
    where
    \begin{equation*}
        \Swap{a}{a'}=\sum_{t=1}^T \vxtr[a]\p{\vut[a']-\vut[a]} ~.
    \end{equation*}
    For each action $a \in [N]$, we establish a lower bound on $\EE_{r\^T}\ps{\max_{a' \in [N]} \Swap{a}{a'}}$ in terms of the amount of time that $a$ is played in rounds $[1:T']$ (recall $T'=2^DB$ is the last iteration where the adversary gives rewards).
    We split into 6 cases: 
    \begin{enumerate}
        \item $a = a_v$ where $\dep(v) \leq D-2$ (Lemma~\ref{lem:case1})
        \item $a = \dot{a}_v$ where $\dep(v) \leq D-2$ (Lemma~\ref{lem:case2})
        \item $a = a_v$ where $\dep(v) = D-1$ (Lemma~\ref{lem:case3}) 
        \item $a = \dot{a}_v$ where $\dep(v) = D-1$ (Lemma~\ref{lem:case4})
        \item $a = a_v$ or $a = \dot{a}_v$ where $\dep(v) = D$ (Lemma~\ref{lem:case5})
        \item $a \in [N]$ such that $a\ne a_v,\dot{a}_v$ for all nodes $v$ in the tree --- these actions receive no reward (Lemma~\ref{lem:case6})
    \end{enumerate}
    For cases 1,2,3,4, and 6, we prove
    \begin{align} \label{eq:bnd-one-v2}
    \EE_{r\^T}\ps{\max_{a' \in [N]} \Swap{a}{a'}} \ge \Omega\p{\frac{1}{D^5}} \EE_{r\^T}\ps{\sum_{t=1}^{T'} \vxtr[a]}
    \end{align}
    For case 5 where $v$ is a leaf, we prove
    \begin{align}
    \EE_{r\^T}\ps{\max_{a' \in [N]} \Swap{a_v}{a'}} &+ \EE_{r\^T}\ps{\max_{a' \in [N]} \Swap{\dot{a}_v}{a'}} \nonumber\\
    &\ge \Omega\p{\frac{1}{D^4\sqrt{B}}} \p{\EE_{r\^T}\ps{\sum_{t=1}^{T'} \vxtr[a_{v}]+\vxtr[\dot{a}_{v}]}-\frac{B}{2}} \label{eq:bnd-one-v3}
    \end{align}
    Summing these bounds over all $a \in [N]$ gives
    \begin{align*}
    \EE_{r\^T}\ps{\DSR\p{\vx\^{t \in 1:T}_{r\^{t}},\vu\^{1:T}}} &\ge \Omega\p{\frac{1}{D^5 \sqrt{B} T}} \p{\EE_{r\^T}\ps{\sum_{t=1}^{T'} \sum_{a\in [N]} \vxtr[a]} - (2^D)\frac{B}{2} }\\
    \intertext{subtracting a $B/2$ term for each of the $2^D$ leaves}
    &=\Omega\p{\frac{1}{D^5 \sqrt{B} T}} \p{T'/2}\\
    \intertext{since $\sum_{a\in[N]} \vxtr[a] = 1$ in each iteration}
    &= \Omega\p{\frac{1}{D^5\sqrt{B}}} = \Omega \p{\min\p{\frac{1}{\log^5(T)},\frac{\sqrt{N/T}}{\log^5(N)}}}
    \end{align*}
    since $T' \geq T/2$, as desired. In Appendix \ref{app:lower-bound}, we establish lemmas proving equations \eqref{eq:bnd-one-v2} and \eqref{eq:bnd-one-v3} for each of these 6 cases respectively. Below, we will outline the proof for these cases, while getting into more detail in only some of them.
    \subsubsection{Case 1: $v$ is an internal node and $a=a_v$}
    Let's consider the first case where $a = a_v$ with $v$ an internal node. 
    The interval of time over which this action is potentially rewarded, $[\tl_v,\tr_v]$, can be broken into 4 intervals: $[\tl_{vLL},\tr_{vLL}],[\tl_{vLR},\tr_{vLR}],[\tl_{vRL},\tr_{vRL}]$ and $[\tl_{vRR},\tr_{vRR}]$, also referred to as first interval, second interval etc. Denote by $z_1$ the total weight put on $a_v$ during the first interval (a random variable in $r\^{t}$):
    \[
    z_1 = \sum_{t=\tl_{vLL}}^{\tr_{vLL}}\vxtr[a_v].
    \]
    Similarly, denote by $z_2,z_3$ and $z_4$ the weight in the second, third and fourth intervals. Additionally, denote by $z_5$ the weight on $a_v$ outside of these intervals, namely, when $v \notin P_b$:
    \[
    z_5 = \sum_{t=1}^{\tl_{v}-1}\vxtr[a_v] + 
                \sum_{t=\tr_{v}+1}^{T}\vxtr[a_v].
    \]
    We will prove the following claims:
    \begin{itemize}
        \item The regret for swapping $a_v$ to $a_{vL}$ will be large, if $a_v$ is played a lot during the first interval:
        \begin{equation}\label{eq:internal-q1}
        \EE\ps{\max_{a'\in [N]}\Swap{a_v}{a'}} \ge \EE\ps{\max\p{\Swap{a_v}{a_{vL}},0}} \ge \EE\ps{\frac{z_1}{C_1' D}}\enspace,
        \end{equation}
        where $C_1'>0$ is a bounded constant.
        Notice that the first inequality is due to that fact that the regret for swapping from $a_v$ to itself is always zero. The second inequality is elaborated below.
        \item The regret for swapping to $a_{vL}$ or $\dot{a}_{vL}$ is large if $a_v$ is played a lot during the second quarter but not during the first quarter:
        \begin{align} \label{eq:internal-q2}
        \EE\ps{\max_{a'\in [N]}\Swap{a_v}{a'}} \ge \EE\ps{\max\p{\Swap{a_v}{a_{vL}},\Swap{a_v}{\dot{a}_{vL}},0}} \\\ge \EE\ps{\frac{z_2}{C_2' D} - C_2 z_1}\enspace, \notag
        \end{align}
        where $C_2$ and $C_2'$ are bounded constants (and similarly for the constants in the other cases).
        \item Similarly, we lower bound the swap regret for swapping to $a_{vR}$:
        \begin{equation} \label{eq:internal-q3}
        \EE\ps{\max_{a'\in [N]}\Swap{a_v}{a'}} \ge \EE\ps{\max\p{\Swap{a_v}{a_{vR}},0}} \ge \EE\ps{\frac{z_3}{C_3' D} - C_3(z_1+z_2)}\enspace,
        \end{equation}
        \item The swap regret to $\dot{a}_{vR}$:
        \begin{align}\label{eq:internal-q4}
        \EE\ps{\max_{a'\in [N]}\Swap{a_v}{a'}} \ge \EE\ps{\max\p{\Swap{a_v}{a_{vR}},\Swap{a_v}{\dot{a}_{vR}},0}} \\
        \ge \EE\ps{\frac{z_4}{C_4' D} - C_4(z_1+z_2+z_3)}\enspace, \notag
        \end{align}
        \item Lastly, the swap regret to the root, denoted here as $\text{root}$:
        \begin{align}\label{eq:internal-q5}
        \EE\ps{\max_{a'\in [N]}\Swap{a_v}{a'}} \ge \EE\ps{\max\p{\Swap{a_v}{a_{\text{root}}}, \Swap{a_v}{\dot{a}_{\text{root}}} ,0}} \\
        \notag \ge \EE\ps{\frac{z_5}{C_5' D} - C_5(z_1+z_2+z_3+z_4)}\enspace,
        \end{align}
    \end{itemize}

    Next, we add the above equations with appropriate coefficients:
    \begin{align*}
        \eqref{eq:internal-q1} + \frac{1}{5C_1'C_2 D}\eqref{eq:internal-q2} + \frac{1}{5C_1'C_2C_2'C_3 D^2}\eqref{eq:internal-q3} + \frac{1}{5C_1'C_2C_2'C_3C'_3 C_4 D^3}\eqref{eq:internal-q4}\\ + \frac{1}{5C_1'C_2C_2'C_3C'_3 C_4 C_4' C_5 D^4}\eqref{eq:internal-q5}
    \end{align*}
    The left hand side will be a constant times the maximal swap regret from $a_v$ to any other action. The right hand side will be $\Omega(\EE\ps{z_1+\cdots+z_5}/D^5)$. This yields Eq.~\eqref{eq:bnd-one-v2} as desired. Below, we outline the first two inequalities, Eq.~\eqref{eq:internal-q1} and Eq.~\eqref{eq:internal-q2}, in more depth, and the other inequalities will be discussed more lightly. 
    \paragraph{Proving Eq.~\eqref{eq:internal-q1}.}

    To compute $\Swap{a_v}{a_{vL}}$, we need to understand how much reward is given to each of $a_v$ and $a_{vL}$, as a function $r\^T$, in each of the intervals $[\tl_{vLL},\tr_{vLL}],[\tl_{vLR},\tr_{vLR}],[\tl_{vRL},\tr_{vRL}]$ and $[\tl_{vRR},\tr_{vRR}]$. Assume that $v$ is of depth $d$, and notice:
    \begin{itemize}
        \item In the first interval, the reward of action $a_v$ is $d/(2D)$, while the reward of $a_{vL}$ is $(d+1)/2D$.
        \item In the second interval, if $r_{vL} = 0$ then the reward of $a_v$ is $d/(2D)$ and the reward of $a_{vL}$ is $(d+1)/2D$. We will not care about what happens when $r_{vL}=1$.
        \item In the third and fourth intervals, if $r_v=1$ then both actions get a reward of zero. We will not care what happens when $r_v=0$.
        \item In iterations outside these intervals, when $v,vL \notin P_b$, both actions receive a reward of zero.
    \end{itemize}
    This implies that if $r_{vL}=0$ and $r_{v}=1$ then the total reward for action $r_v$ over all the $T$ rounds is $\frac{d}{2D} (z_1+z_2)$, since $z_1+z_2$ is the total weight put on action $a_v$ in these two intervals. The reward that would be obtained from playing $a_{vL}$ instead is $\frac{d+1}{2D}(z_1+z_2)$. Therefore, if $r_{vL}=1$ and $r_v=0$:
    \[
    \Swap{a_v}{a_{vL}} \ge \frac{d+1}{2D} (z_1+z_2)-\frac{d}{2D} (z_1+z_2)
    = \frac{z_1+z_2}{2D} \ge \frac{z_1}{2D}.
    \]
    Next, we use the fact that $z_1$ is independent from $r_{vL}$ and $r_v$, since $z_1$ is the weight on action $a_v$ in the first interval, before the learner observes $r_v$ and $r_{vL}$. We derive that:
    \begin{align*}
    \EE_{r\^T}\ps{\max\p{\Swap{a_v}{a_{vL}},0}}
    &\ge \EE_{r\^T}\ps{\max\p{\Swap{a_v}{a_{vL}},0} \1\ps{r_v=1,r_{vL}=0}}\\
    &\ge \EE_{r\^T}\ps{\Swap{a_v}{a_{vL}}\1\ps{r_v=1,r_{vL}=0}}\\
    &\ge \EE_{r\^T}\ps{\frac{z_1}{2D}\1\ps{r_v=1,r_{vL}=0}} \\
    &\ge \EE_{r\^T}\ps{\frac{z_1}{2D}} \EE\ps{\1\ps{r_v=1,r_{vL}=0}}\\
    &\ge \frac{\EE[z_1]}{8D}.
    \end{align*}
    This is what we wanted to prove.

    \paragraph{Proving Equation~\eqref{eq:internal-q2}}
    We divide into cases according to $r_{vL}$. As before, we will condition on $r_v=1$. First, if $r_{vL}=0$, then, as argued in the proof of Eq.~\eqref{eq:internal-q1},
    \[
    \Swap{a_v}{a_{vL}} \ge \frac{z_1+z_2}{2D} \ge \frac{z_2}{2D}.
    \]
    For the case that $r_{vL}=1$, we lower bound the swap regret to $\dot{a}_{vL}$. We notice that:
    \begin{itemize}
        \item In the first interval, the reward of $a_v$ is $d/(2D)$ and the reward of $\dot a_{vL} = 0$.
        \item In the second interval, the reward of $a_{v}$ is still $d/(2D)$ while the reward of $\dot a_{vL}$ is $(d+1)/2D$.
        \item In the remaining iterations, both $a_v$ and $\dot a_{vL}$ have reward of $0$.
    \end{itemize}
    Consequently, the reward of $a_v$ is $\frac{d}{2D}(z_1+z_2)$, whereas the reward of swapping $a_v$ to $\dot a_{vL}$ is $\frac{d+1}{2d}z_2$. Hence,
    \[
    \Swap{a_v}{\dot a_{vL}} \ge \frac{z_2}{2D} - \frac{dz_1}{2D}.
    \]
    By combining the two cases above, we obtain that whenever $a_v = 1$,
    \[
    \max_{a' \in [N]} \Swap{a_v}{a'} \ge \frac{z_2}{2D} - \frac{dz_1}{2D}~.
    \]
Hence, we obtain that    
    \begin{align*}
    \EE_{r\^T}\ps{\max_{a'\in[N]}\Swap{a_v}{a'}}
    &\ge \EE_{r\^T}\ps{\max_{a'\in[N]}\Swap{a_v}{a'} \1[r_v=1]}\\
    &\ge \EE_{r\^T}\ps{\p{\frac{z_2}{2D}-\frac{dz_1}{2D}} \1[r_v=1]}\\
    &= \EE_{r\^T}\ps{\p{\frac{z_2}{2D}-\frac{dz_1}{2D}}}\EE_{r\^T}\ps{\1[r_v=1]}\\
    &= \frac{1}{4D}\EE_{r\^T}\ps{z_2-dz_1},
    \end{align*}
    again, using independence of $r_v$ with $z_1$ and $z_2$. This is what we wanted to prove.

    \paragraph{Proving Equations~\eqref{eq:internal-q3}~to~\eqref{eq:internal-q5}:}
    Eq.~\eqref{eq:internal-q3} and Eq.~\eqref{eq:internal-q4} are proved similarly to Eq.~\eqref{eq:internal-q1} and Eq.~\eqref{eq:internal-q2}, with the difference that now, we also need to consider the case that $r_v=0$. Eq.~\eqref{eq:internal-q5} is also proved similarly. Here, we consider weight put on $a_v$ during iterations where $a_v$ is not rewarded at all, and we bound the regret of swapping to the root.

    \subsubsection{Cases 2,3,4,6: other cases where $v$ is not a leaf}
    For the case that $v$ is an internal node of depth at most $D-2$, and $a=\dot{a}_v$, the proof is very similar to the case that $a=a_v$. Similarly, the cases where $v$ is an internal node of depth $D-1$ and $a=a_v,\dot{a}_v$ are very similar to the cases of depth at most $D-2$.  They are separate cases simply because the child nodes in this case are leaves, altering the equations.  For the case when $a$ is associated with no nodes at all and receives 0 reward, we obtain our lower bound by simply considering the swap to the root.
    
    \subsubsection{Case 5: $v$ is a leaf and $a=a_v$ or $\dot{a}_v$}
     \todo{Replace $z$ for $\sum x$ to match style of other subsections}
    For Case 5, when $v=\ell_b$ is a leaf of the tree, there are Bernoulli random variables $r_{(v,t)}$ for every time step $t$ in $[\tl_{v},\tr_{v}]$.  By definition,
    \begin{align*}
        \vut[a_{v}] &= \1 \ps{r_{(v,t)}=0} & \vut[\dot{a}_{v}] &= \1 \ps{r_{(v,t)}=1}
    \end{align*}
    for $t \in [\tl_{v},\tr_{v}]$.  We have,
    \begin{align*}
        & \EE_{r\^T}\ps{\max_{a' \in [N]} \Swap{a_v}{a'}} + \EE_{r\^T}\ps{\max_{a' \in [N]} \Swap{\dot{a}_v}{a'}}\\
        & \geq \EE_{r\^T}\ps{\max\set{\sum_{t=\tl_v}^{\tr_v} \p{\vxtr[a_v]+\vxtr[\dot{a}_v]}\1 \ps{r_{(v,t)}=0},\sum_{t=\tl_v}^{\tr_v} \p{\vxtr[a_v]+\vxtr[\dot{a}_v]}\1 \ps{r_{(v,t)}=1}}}\\
        &- \frac{1}{2} \EE_{r\^T}\ps{\sum_{t=\tl_v}^{\tr_v} \p{\vxtr[a_v]+\vxtr[\dot{a}_v]}}\\
        &=\frac{1}{2} \EE_{r\^T}\ps{\card{\sum_{t=\tl_v}^{\tr_v} \p{\vxtr[a_v]+\vxtr[\dot{a}_v]} Z\^{t}}} \qquad \text{where $Z\^t =\begin{cases}
        1 &\text{ if }r_{(v,t)}=1\\
        -1 &\text{ if }r_{(v,t)}=0
    \end{cases}$}
    \end{align*}
    If we denote by $X\^t = \sum_{\tau = \tl_v}^{t} \p{\vxtr[a_v]+\vxtr[\dot{a}_v]} Z\^{t}$, then $X\^{\tl_v},\cdots,X\^{\tr_v}$ is a martingale due to the independence of $\vxtr[a_v],\vxtr[\dot{a}_v]$ and $Z\^{t}$ for all $t$.  We use the following lemma which applies Azuma's inequality,
    \begin{lemma}
        Consider an algorithm which, at any round $t=1,\dots,B$, selects a parameter $x_t \in [0,1]$, possibly at random. Then, it observes the outcome of a random variable $Z_t \sim \mathrm{Uniform}(\{-1,1\})$. Assume that $\EE\ps{\sum_{t=1}^B x_t^2} \ge \epsilon B$ for some $\epsilon>0$. Then,
        \[
        \EE\ps{\card{ \sum_{t=1}^B x_t Z_t }} \ge \frac{\epsilon \sqrt{B}}{4\sqrt{\log(1/\epsilon)}} ~.
        \]
    \end{lemma}
    Using this lemma, we establish
    \begin{align*}
        \EE_{r\^T}\ps{\max_{a' \in [N]} \Swap{a_v}{a'}} &+ \EE_{r\^T}\ps{\max_{a' \in [N]} \Swap{\dot{a}_v}{a'}}\nonumber\\
        &\geq \frac{1}{1024 D^3 \sqrt{B}} \p{\sum_{t=\tl_{v}}^{\tr_{v}} \EE_{r\^T}\ps{\vxtr[a_{v}]+\vxtr[\dot{a}_{v}]}}-\frac{\sqrt{B}}{8192D^4}
    \end{align*}

    \section*{Acknowledgements}
    We are grateful to Binghui Peng for helpful comments on the manuscript. 
    
    \appendix
    \section{Proofs remaining for Theorem~\ref{thm:lower-main}}\label{app:lower-bound}

    \begin{lemma}[Case 1]\label{lem:case1}
        Let $v$ be a node in the tree of depth $d \leq D-2$.  Then,
        {\upshape
        \begin{align}
            184D^5\EE_{r\^T}\ps{\max_{a' \in [N]} \Swap{a_v}{a'}} \geq \sum_{t=1}^{T'} \EE_{r\^T}\ps{\vxtr[a_v]} \label{eq:case1sum}
        \end{align}
        }
    \end{lemma}
    \begin{proof}[Proof of Lemma \ref{lem:case1}]
    We have
    \begin{align*}
        \Swap{a_v}{a_{vL}}&=\sum_{t=1}^T \vxtr[a_v]\p{\vut[a_{vL}]-\vut[a_v]}\\
        &= \frac{d+1}{2D} \p{\sum_{t=\tl_{vLL}}^{\tr_{vLL}} \vxtr[a_v] + \1 \ps{r_{vL}=0} \sum_{t=\tl_{vLR}}^{\tr_{vLR}} \vxtr[a_v]}\\
        &- \frac{d}{2D} \p{\sum_{t=\tl_{vL}}^{\tr_{vL}} \vxtr[a_v] + \1 \ps{r_{v}=0} \sum_{t=\tl_{vR}}^{\tr_{vR}} \vxtr[a_v]}
    \end{align*}
    Therefore,
    \begin{align}
        \1[r_{vL}=0 \wedge r_v = 1] \cdot \Swap{a_v}{a_{vL}} &= \frac{1}{2D} \1[r_{vL}=0 \wedge r_v = 1] \sum_{t=\tl_{vL}}^{\tr_{vL}} \vxtr[a_v] \nonumber\\
        &= \frac{1}{2D} \1[r_{vL}=0 \wedge r_v = 1] \sum_{t=\tl_{vLL}}^{\tr_{vLL}} \vxtr[a_v] \label{eq:independence11}\\
        &+ \frac{1}{2D} \1[r_v = 1] \sum_{t=\tl_{vLR}}^{\tr_{vLR}} \vxtr[a_v]\1[r_{vL}=0] \label{eq:independence12}
    \end{align}
    and
    \begin{align}
        \EE_{r\^T}\ps{\1[r_{vL}=0 \wedge r_v = 1]  \sum_{t=\tl_{vLL}}^{\tr_{vLL}} \vxtr[a_v]} &= \EE_{r\^T}\ps{\1[r_{vL}=0 \wedge r_v = 1]}\EE_{r\^T}\ps{ \sum_{t=\tl_{vLL}}^{\tr_{vLL}} \vxtr[a_v]} \nonumber\\
        &= \frac{1}{4} \sum_{t=\tl_{vLL}}^{\tr_{vLL}} \EE_{r\^T}\ps{ \vxtr[a_v]} \label{eq:independence13}
    \end{align}
    due to the independence of $r_{vL},r_v$ from $r\^t$ for $t \leq \tr_{vLL}$. And,
    \begin{align}
        \EE_{r\^T}\ps{\1[r_v = 1] \sum_{t=\tl_{vLR}}^{\tr_{vLR}} \vxtr[a_v]\1[r_{vL}=0]} &= \EE_{r\^T}\ps{\1[r_v = 1]}\EE_{r\^T}\ps{\sum_{t=\tl_{vLR}}^{\tr_{vLR}} \vxtr[a_v]\1[r_{vL}=0]} \nonumber\\
        &= \frac{1}{2} \sum_{t=\tl_{vLR}}^{\tr_{vLR}} \EE_{r\^T}\ps{\vxtr[a_v]\1[r_{vL}=0]} \label{eq:independence14}
    \end{align}
    due to the independence of $r_v$ from $r_{vL},r\^t$ for $t \leq \tr_{vLR}$. So, combining equations \eqref{eq:independence11}, \eqref{eq:independence12}, \eqref{eq:independence13}, and \eqref{eq:independence14}, we have
    \begin{align*}
        &\EE_{r\^T}\ps{\1[r_{vL}=0 \wedge r_v = 1]\cdot \Swap{a_v}{a_{vL}}}\\
        &= \frac{1}{8D} \sum_{t=\tl_{vLL}}^{\tr_{vLL}} \EE_{r\^T}\ps{ \vxtr[a_v]} + \frac{1}{4D} \sum_{t=\tl_{vLR}}^{\tr_{vLR}} \EE_{r\^T}\ps{\vxtr[a_v]\1[r_{vL}=0]}
    \end{align*}
    Lastly, since $\Swap{a_v}{a_v} = 0$,
    \begin{align}
        \EE_{r\^T}\ps{\max_{a' \in [N]} \Swap{a_v}{a'}}&\geq \EE_{r\^T}\ps{\max\set{\Swap{a_v}{a_{vL}},0}}\nonumber\\
        &\geq \EE_{r\^T}\ps{\1[r_{vL}=0 \wedge r_v = 1]\cdot\max\set{\Swap{a_v}{a_{vL}},0}}\nonumber\\
        &\geq \EE_{r\^T}\ps{\1[r_{vL}=0 \wedge r_v = 1]\cdot\Swap{a_v}{a_{vL}}}\nonumber\\
        &= \frac{1}{8D} \sum_{t=\tl_{vLL}}^{\tr_{vLL}} \EE_{r\^T}\ps{ \vxtr[a_v]} + \frac{1}{4D} \sum_{t=\tl_{vLR}}^{\tr_{vLR}} \EE_{r\^T}\ps{\vxtr[a_v]\1[r_{vL}=0]} \label{eq:case11}
    \end{align}
    
    Similarly,
    \begin{align*}
        \Swap{a_v}{\dot{a}_{vL}}&=\sum_{t=1}^T \vxtr[a_v]\p{\vut[\dot{a}_{vL}]-\vut[a_v]}\\
        &= \frac{d+1}{2D} \p{\1 \ps{r_{vL}=1} \sum_{t=\tl_{vLR}}^{\tr_{vLR}} \vxtr[a_v]}- \frac{d}{2D} \p{\sum_{t=\tl_{vL}}^{\tr_{vL}} \vxtr[a_v] + \1 \ps{r_{v}=0} \sum_{t=\tl_{vR}}^{\tr_{vR}} \vxtr[a_v]}
    \end{align*}
    Therefore,
    \begin{align}
        \1[r_v = 1] \cdot \Swap{a_v}{\dot{a}_{vL}} &= \frac{1}{2D} \1[r_v=1] \sum_{t=\tl_{vLR}}^{\tr_{vLR}} \vxtr[a_v] \1[r_{vL}=1] \label{eq:independence21}\\
        &- \frac{d}{2D} \1[r_v=1] \p{\sum_{t=\tl_{vLL}}^{\tr_{vLL}} \vxtr[a_v] + \sum_{t=\tl_{vLR}}^{\tr_{vLR}} \vxtr[a_v] \1[r_{vL}=0]}\label{eq:independence22}
    \end{align}
    and
    \begin{align}
        \EE_{r\^T}\ps{\1[r_v=1] \sum_{t=\tl_{vLR}}^{\tr_{vLR}} \vxtr[a_v] \1[r_{vL}=1]} &=\EE_{r\^T}\ps{\1[r_v=1]}\EE_{r\^T}\ps{ \sum_{t=\tl_{vLR}}^{\tr_{vLR}} \vxtr[a_v] \1[r_{vL}=1]}\nonumber \\
        &=\frac{1}{2}\sum_{t=\tl_{vLR}}^{\tr_{vLR}}\EE_{r\^T}\ps{\vxtr[a_v] \1[r_{vL}=1]}\label{eq:independence23}
    \end{align}
    due to the independence of $r_v$ from $r_{vL},r\^t$ for $t \leq \tr_{vLR}$. And,
    \begin{align}
        &\EE_{r\^T}\ps{\1[r_v=1] \p{\sum_{t=\tl_{vLL}}^{\tr_{vLL}} \vxtr[a_v] + \sum_{t=\tl_{vLR}}^{\tr_{vLR}} \vxtr[a_v] \1[r_{vL}=0]}} \nonumber\\
        &=\EE_{r\^T}\ps{\1[r_v=1]}\EE_{r\^T}\ps{\sum_{t=\tl_{vLL}}^{\tr_{vLL}} \vxtr[a_v] + \sum_{t=\tl_{vLR}}^{\tr_{vLR}} \vxtr[a_v] \1[r_{vL}=0]} \nonumber\\
        &=\frac{1}{2}\sum_{t=\tl_{vLL}}^{\tr_{vLL}} \EE_{r\^T}\ps{\vxtr[a_v]} + \frac{1}{2}\sum_{t=\tl_{vLR}}^{\tr_{vLR}} \EE_{r\^T}\ps{\vxtr[a_v] \1[r_{vL}=0]} \label{eq:independence24}
    \end{align}
    due to the independence of $r_v$ from $r\^t$ for $t \leq \tr_{vLL}$.  So, combining equations \eqref{eq:independence21}, \eqref{eq:independence22}, \eqref{eq:independence23}, and \eqref{eq:independence24}, we have
    \begin{align*}
        &\EE_{r\^T}\ps{\1[r_v = 1] \cdot \Swap{a_v}{\dot{a}_{vL}}}\\
        &=\frac{1}{4D}\sum_{t=\tl_{vLR}}^{\tr_{vLR}}\EE_{r\^T}\ps{\vxtr[a_v] \1[r_{vL}=1]}-\frac{d}{4D}\sum_{t=\tl_{vLL}}^{\tr_{vLL}}\EE_{r\^T}\ps{\vxtr[a_v]}-\frac{d}{4D}\sum_{t=\tl_{vLR}}^{\tr_{vLR}} \EE_{r\^T}\ps{\vxtr[a_v] \1[r_{vL}=0]}\\
        &\geq \frac{1}{4D}\sum_{t=\tl_{vLR}}^{\tr_{vLR}}\EE_{r\^T}\ps{\vxtr[a_v] \1[r_{vL}=1]}-\frac{1}{4}\sum_{t=\tl_{vLL}}^{\tr_{vLL}}\EE_{r\^T}\ps{\vxtr[a_v]}-\frac{1}{4}\sum_{t=\tl_{vLR}}^{\tr_{vLR}} \EE_{r\^T}\ps{\vxtr[a_v] \1[r_{vL}=0]}
    \end{align*}
    since $d \leq D$.  Lastly, since $\Swap{a_v}{a_v} = 0$,
    \begin{align}
        &\EE_{r\^T}\ps{\max_{a' \in [N]} \Swap{a_v}{a'}}\nonumber\\
        &\geq \EE_{r\^T}\ps{\max\set{\Swap{a_v}{\dot{a}_{vL}},0}}\nonumber\\
        &\geq \EE_{r\^T}\ps{\1[r_v = 1]\cdot\max\set{\Swap{a_v}{\dot{a}_{vL}},0}}\nonumber\\
        &\geq \EE_{r\^T}\ps{\1[r_v = 1]\cdot\Swap{a_v}{\dot{a}_{vL}}}\nonumber\\
        &\geq \frac{1}{4D}\sum_{t=\tl_{vLR}}^{\tr_{vLR}}\EE_{r\^T}\ps{\vxtr[a_v] \1[r_{vL}=1]}-\frac{1}{4}\sum_{t=\tl_{vLL}}^{\tr_{vLL}}\EE_{r\^T}\ps{\vxtr[a_v]}-\frac{1}{4}\sum_{t=\tl_{vLR}}^{\tr_{vLR}} \EE_{r\^T}\ps{\vxtr[a_v] \1[r_{vL}=0]}\label{eq:case12}
    \end{align}

    Similarly,
    \begin{align*}
        &\Swap{a_v}{a_{vR}}=\sum_{t=1}^T \vxtr[a_v]\p{\vut[a_{vR}]-\vut[a_v]}\\
        &= \frac{d+1}{2D} \p{\sum_{t=\tl_{vRL}}^{\tr_{vRL}} \vxtr[a_v] + \1 \ps{r_{vR}=0} \sum_{t=\tl_{vRR}}^{\tr_{vRR}} \vxtr[a_v]}- \frac{d}{2D} \p{\sum_{t=\tl_{vL}}^{\tr_{vL}} \vxtr[a_v] + \1 \ps{r_{v}=0} \sum_{t=\tl_{vR}}^{\tr_{vR}} \vxtr[a_v]}\\
        &\geq \frac{d+1}{2D} \p{\sum_{t=\tl_{vRL}}^{\tr_{vRL}} \vxtr[a_v] + \1 \ps{r_{vR}=0} \sum_{t=\tl_{vRR}}^{\tr_{vRR}} \vxtr[a_v]}- \frac{d}{2D} \p{\sum_{t=\tl_{v}}^{\tr_{v}} \vxtr[a_v]}
    \end{align*}
    bounding $\1[r_v = 0] \leq 1$.  Therefore,
    \begin{align}
        \1[r_{vR}=0] \cdot \Swap{a_v}{a_{vR}} &\geq \frac{1}{2D} \1[r_{vR}=0] \sum_{t=\tl_{vRL}}^{\tr_{vRL}} \vxtr[a_v] \label{eq:independence31}\\
        &+ \frac{1}{2D} \sum_{t=\tl_{vRR}}^{\tr_{vRR}} \vxtr[a_v]\1[r_{vR}=0] \label{eq:independence32}\\
        &- \frac{d}{2D} \1[r_{vR}=0] \sum_{t=\tl_{vL}}^{\tr_{vL}} \vxtr[a_v] \label{eq:independence33}
    \end{align}
    and
    \begin{align}
        \EE_{r\^T}\ps{\1[r_{vR}=0] \sum_{t=\tl_{vRL}}^{\tr_{vRL}} \vxtr[a_v]} &= \EE_{r\^T}\ps{\1[r_{vR}=0]}\EE_{r\^T}\ps{ \sum_{t=\tl_{vRL}}^{\tr_{vRL}} \vxtr[a_v]} \nonumber\\
        &= \frac{1}{2} \sum_{t=\tl_{vRL}}^{\tr_{vRL}} \EE_{r\^T}\ps{\vxtr[a_v]} \label{eq:independence34}
    \end{align}
    due to the independence of $r_{vR}$ from $r\^t$ for $t \leq \tr_{vRL}$. And,
    \begin{align}
        \EE_{r\^T}\ps{\1[r_{vR}=0] \sum_{t=\tl_{vL}}^{\tr_{vL}} \vxtr[a_v]} &= \EE_{r\^T}\ps{\1[r_{vR}=0]}\EE_{r\^T}\ps{ \sum_{t=\tl_{vL}}^{\tr_{vL}} \vxtr[a_v]} \nonumber\\
        &= \frac{1}{2} \sum_{t=\tl_{vL}}^{\tr_{vL}} \EE_{r\^T}\ps{\vxtr[a_v]} \label{eq:independence35}
    \end{align}
    due to the independence of $r_{vR}$ from $r\^t$ for $t \leq \tr_{vL}$. So, combining equations \eqref{eq:independence31}, \eqref{eq:independence32}, \eqref{eq:independence33}, \eqref{eq:independence34}, and \eqref{eq:independence35}, we have
    \begin{align*}
        &\EE_{r\^T}\ps{\1[r_{vR}=0] \cdot \Swap{a_v}{a_{vR}}}\\
        &\geq \frac{1}{4D} \sum_{t=\tl_{vRL}}^{\tr_{vRL}} \EE_{r\^T}\ps{\vxtr[a_v]} + \frac{1}{2D} \sum_{t=\tl_{vRR}}^{\tr_{vRR}} \EE_{r\^T}\ps{\vxtr[a_v]\1[r_{vR}=0]} - \frac{d}{4D} \sum_{t=\tl_{vL}}^{\tr_{vL}} \EE_{r\^T}\ps{\vxtr[a_v]}\\
        &\geq \frac{1}{4D} \sum_{t=\tl_{vRL}}^{\tr_{vRL}} \EE_{r\^T}\ps{\vxtr[a_v]} + \frac{1}{2D} \sum_{t=\tl_{vRR}}^{\tr_{vRR}} \EE_{r\^T}\ps{\vxtr[a_v]\1[r_{vR}=0]} - \frac{1}{4} \sum_{t=\tl_{vL}}^{\tr_{vL}} \EE_{r\^T}\ps{\vxtr[a_v]}
    \end{align*}
    since $d \leq D$.  Lastly, since $\Swap{a_v}{a_v} = 0$,
    \begin{align}
        &\EE_{r\^T}\ps{\max_{a' \in [N]} \Swap{a_v}{a'}}\nonumber\\
        &\geq \EE_{r\^T}\ps{\max\set{\Swap{a_v}{a_{vR}},0}}\nonumber\\
        &\geq \EE_{r\^T}\ps{\1[r_{vR}=0]\cdot\max\set{\Swap{a_v}{a_{vR}},0}}\nonumber\\
        &\geq \EE_{r\^T}\ps{\1[r_{vR}=0]\cdot\Swap{a_v}{a_{vR}}}\nonumber\\
        &\geq \frac{1}{4D} \sum_{t=\tl_{vRL}}^{\tr_{vRL}} \EE_{r\^T}\ps{\vxtr[a_v]} + \frac{1}{2D} \sum_{t=\tl_{vRR}}^{\tr_{vRR}} \EE_{r\^T}\ps{\vxtr[a_v]\1[r_{vR}=0]} - \frac{1}{4} \sum_{t=\tl_{vL}}^{\tr_{vL}} \EE_{r\^T}\ps{\vxtr[a_v]}\label{eq:case13}
    \end{align}

    Similarly,
    \begin{align*}
        &\Swap{a_v}{\dot{a}_{vR}}=\sum_{t=1}^T \vxtr[a_v]\p{\vut[\dot{a}_{vR}]-\vut[a_v]}\\
        &= \frac{d+1}{2D} \p{\1 \ps{r_{vR}=1} \sum_{t=\tl_{vRR}}^{\tr_{vRR}} \vxtr[a_v]}- \frac{d}{2D} \p{\sum_{t=\tl_{vL}}^{\tr_{vL}} \vxtr[a_v] + \1 \ps{r_{v}=0} \sum_{t=\tl_{vR}}^{\tr_{vR}} \vxtr[a_v]}\\
        &\geq \frac{d+1}{2D} \p{\1 \ps{r_{vR}=1} \sum_{t=\tl_{vRR}}^{\tr_{vRR}} \vxtr[a_v]}- \frac{d}{2D} \p{\sum_{t=\tl_{v}}^{\tr_{v}} \vxtr[a_v]}\\
        \intertext{bounding $\1[r_v = 0] \leq 1$}
        &= \frac{1}{2D} \sum_{t=\tl_{vRR}}^{\tr_{vRR}} \vxtr[a_v] \1 \ps{r_{vR}=1} - \frac{d}{2D} \p{\sum_{t=\tl_{vL}}^{\tr_{vRL}} \vxtr[a_v] + \sum_{t=\tl_{vRR}}^{\tr_{vRR}} \vxtr[a_v]\1 \ps{r_{vR}=0}}
    \end{align*}
    So, we have
    \begin{align}
        &\EE_{r\^T}\ps{\max_{a' \in [N]} \Swap{a_v}{a'}} \geq \EE_{r\^T}\ps{\Swap{a_v}{\dot{a}_{vR}}}\nonumber\\
        &\geq \frac{1}{2D} \sum_{t=\tl_{vRR}}^{\tr_{vRR}} \EE_{r\^T}\ps{\vxtr[a_v] \1 \ps{r_{vR}=1}} - \frac{d}{2D} \sum_{t=\tl_{vL}}^{\tr_{vRL}} \EE_{r\^T}\ps{\vxtr[a_v]} - \frac{d}{2D} \sum_{t=\tl_{vRR}}^{\tr_{vRR}} \EE_{r\^T}\ps{\vxtr[a_v]\1 \ps{r_{vR}=0}}\nonumber\\
        &\geq \frac{1}{2D} \sum_{t=\tl_{vRR}}^{\tr_{vRR}} \EE_{r\^T}\ps{\vxtr[a_v] \1 \ps{r_{vR}=1}} - \frac{1}{2} \sum_{t=\tl_{vL}}^{\tr_{vRL}} \EE_{r\^T}\ps{\vxtr[a_v]} - \frac{1}{2} \sum_{t=\tl_{vRR}}^{\tr_{vRR}} \EE_{r\^T}\ps{\vxtr[a_v]\1 \ps{r_{vR}=0}}\label{eq:case14}
    \end{align}
    since $d \leq D$.  %

    Similarly, recalling our notation that the empty string $\emptyset$ represents the root of the tree,
    \begin{align*}
        &\Swap{a_v}{a_{\emptyset}}=\sum_{t=1}^T \vxtr[a_v]\p{\vut[a_{\emptyset}]-\vut[a_v]}\\
        &= \frac{1}{2D} \p{\sum_{t=\tl_{\emptyset L}}^{\tr_{\emptyset L}} \vxtr[a_v] + \1 \ps{r_{\emptyset}=0} \sum_{t=\tl_{\emptyset R}}^{\tr_{\emptyset R}} \vxtr[a_v]}- \frac{d}{2D} \p{\sum_{t=\tl_{vL}}^{\tr_{vL}} \vxtr[a_v] + \1 \ps{r_{v}=0} \sum_{t=\tl_{vR}}^{\tr_{vR}} \vxtr[a_v]}\\
        &\geq \frac{1}{2D} \p{\sum_{t=\tl_{\emptyset L}}^{\tr_{\emptyset L}} \vxtr[a_v] + \1 \ps{r_{\emptyset}=0} \sum_{t=\tl_{\emptyset R}}^{\tr_{\emptyset R}} \vxtr[a_v]}- \frac{d}{2D} \p{\sum_{t=\tl_{v}}^{\tr_{v}} \vxtr[a_v]}\\
    \end{align*}
    bounding $\1[r_v = 0] \leq 1$.  Also,
    \begin{align*}
        &\Swap{a_v}{\dot{a}_{\emptyset}}=\sum_{t=1}^T \vxtr[a_v]\p{\vut[\dot{a}_{\emptyset}]-\vut[a_v]}\\
        &= \frac{1}{2D} \p{\1 \ps{r_{\emptyset}=1} \sum_{t=\tl_{\emptyset R}}^{\tr_{\emptyset R}} \vxtr[a_v]}- \frac{d}{2D} \p{\sum_{t=\tl_{vL}}^{\tr_{vL}} \vxtr[a_v] + \1 \ps{r_{v}=0} \sum_{t=\tl_{vR}}^{\tr_{vR}} \vxtr[a_v]}\\
        &\geq \frac{1}{2D} \p{\1 \ps{r_{\emptyset}=1} \sum_{t=\tl_{\emptyset R}}^{\tr_{\emptyset R}} \vxtr[a_v]}- \frac{d}{2D} \p{\sum_{t=\tl_{v}}^{\tr_{v}} \vxtr[a_v]}
    \end{align*}
    So,
    \begin{align*}
        \Swap{a_v}{a_{\emptyset}}+ \Swap{a_v}{\dot{a}_{\emptyset}} &\geq \frac{1}{2D} \p{\sum_{t=1}^{T'} \vxtr[a_v]}- \frac{d}{D} \p{\sum_{t=\tl_{v}}^{\tr_{v}} \vxtr[a_v]}\\
        &\geq \frac{1}{2D} \p{\sum_{t=1}^{T'} \vxtr[a_v]}- \p{\sum_{t=\tl_{v}}^{\tr_{v}} \vxtr[a_v]}
    \end{align*}
    since $d \leq D$. Thus,
    \begin{align}
        \EE_{r\^T}\ps{\max_{a' \in [N]} \Swap{a_v}{a'}} &\geq \frac{1}{2}\EE_{r\^t}\ps{\Swap{a_v}{a_{\emptyset}}+ \Swap{a_v}{\dot{a}_{\emptyset}}}\nonumber \\
        &\geq \frac{1}{4D} \sum_{t=1}^{T'} \EE_{r\^T}\ps{\vxtr[a_v]}- \frac{1}{2}\sum_{t=\tl_{v}}^{\tr_{v}} \EE_{r\^T}\ps{\vxtr[a_v]}\label{eq:case15}
    \end{align}

    Collecting equations \eqref{eq:case11}, \eqref{eq:case12}, \eqref{eq:case13}, \eqref{eq:case14}, and \eqref{eq:case15},

    \begin{align*}
        &\EE_{r\^T}\ps{\max_{a' \in [N]} \Swap{a_v}{a'}} \nonumber\\
        &\geq \frac{1}{8D} \sum_{t=\tl_{vLL}}^{\tr_{vLL}} \EE_{r\^T}\ps{ \vxtr[a_v]} + \frac{1}{4D} \sum_{t=\tl_{vLR}}^{\tr_{vLR}} \EE_{r\^T}\ps{\vxtr[a_v]\1[r_{vL}=0]} &\text{\eqref{eq:case11}}\\
        &\EE_{r\^T}\ps{\max_{a' \in [N]} \Swap{a_v}{a'}} \nonumber\\
        &\geq \frac{1}{4D}\sum_{t=\tl_{vLR}}^{\tr_{vLR}}\EE_{r\^T}\ps{\vxtr[a_v] \1[r_{vL}=1]}-\frac{1}{4}\sum_{t=\tl_{vLL}}^{\tr_{vLL}}\EE_{r\^T}\ps{\vxtr[a_v]}-\frac{1}{4}\sum_{t=\tl_{vLR}}^{\tr_{vLR}} \EE_{r\^T}\ps{\vxtr[a_v] \1[r_{vL}=0]}&\text{\eqref{eq:case12}}\\
        &\EE_{r\^T}\ps{\max_{a' \in [N]} \Swap{a_v}{a'}} \nonumber\\
        &\geq \frac{1}{4D} \sum_{t=\tl_{vRL}}^{\tr_{vRL}} \EE_{r\^T}\ps{\vxtr[a_v]} + \frac{1}{2D} \sum_{t=\tl_{vRR}}^{\tr_{vRR}} \EE_{r\^T}\ps{\vxtr[a_v]\1[r_{vR}=0]} - \frac{1}{4} \sum_{t=\tl_{vL}}^{\tr_{vL}} \EE_{r\^T}\ps{\vxtr[a_v]}&\text{\eqref{eq:case13}}\\
        &\EE_{r\^T}\ps{\max_{a' \in [N]} \Swap{a_v}{a'}} \nonumber\\
        &\geq \frac{1}{2D} \sum_{t=\tl_{vRR}}^{\tr_{vRR}} \EE_{r\^T}\ps{\vxtr[a_v] \1 \ps{r_{vR}=1}} - \frac{1}{2} \sum_{t=\tl_{vL}}^{\tr_{vRL}} \EE_{r\^T}\ps{\vxtr[a_v]} - \frac{1}{2} \sum_{t=\tl_{vRR}}^{\tr_{vRR}} \EE_{r\^T}\ps{\vxtr[a_v]\1 \ps{r_{vR}=0}}&\text{\eqref{eq:case14}}\\
        &\EE_{r\^T}\ps{\max_{a' \in [N]} \Swap{a_v}{a'}} \nonumber\\
        &\geq \frac{1}{4D} \sum_{t=1}^{T'} \EE_{r\^T}\ps{\vxtr[a_v]}- \frac{1}{2}\sum_{t=\tl_{v}}^{\tr_{v}} \EE_{r\^T}\ps{\vxtr[a_v]}&\text{\eqref{eq:case15}}
    \end{align*}

    Summing the inequalities $$4D\text{\eqref{eq:case15}}+4D^2\text{\eqref{eq:case14}}+16D^3\text{\eqref{eq:case13}}+32D^4\text{\eqref{eq:case12}}+128D^5\text{\eqref{eq:case11}}$$ gives
    \begin{align*}
        &184D^5\EE_{r\^T}\ps{\max_{a' \in [N]} \Swap{a_v}{a'}}\nonumber\\
        &\geq \sum_{t=1}^{T'} \EE_{r\^T}\ps{\vxtr[a_v]}\nonumber\\
        &+(-2D-2D^2-4D^3-8D^4+16D^4)\sum_{t=\tl_{vLL}}^{\tr_{vLL}} \EE_{r\^T}\ps{ \vxtr[a_v]}\nonumber\\
        &+(-2D-2D^2-4D^3-8D^4+32D^4)\sum_{t=\tl_{vLR}}^{\tr_{vLR}} \EE_{r\^T}\ps{\vxtr[a_v]\1[r_{vL}=0]}\nonumber\\
        &+(-2D-2D^2-4D^3+8D^3)\sum_{t=\tl_{vLR}}^{\tr_{vLR}} \EE_{r\^T}\ps{\vxtr[a_v]\1[r_{vL}=1]}\nonumber\\
        &+(-2D-2D^2+4D^2)\sum_{t=\tl_{vRL}}^{\tr_{vRL}} \EE_{r\^T}\ps{\vxtr[a_v]}\nonumber\\
        &+(-2D-2D^2+8D^2)\sum_{t=\tl_{vRR}}^{\tr_{vRR}} \EE_{r\^T}\ps{\vxtr[a_v]\1[r_{vR}=0]}\nonumber\\
        &+(-2D+2D)\sum_{t=\tl_{vRR}}^{\tr_{vRR}} \EE_{r\^T}\ps{\vxtr[a_v]\1[r_{vR}=1]}\nonumber\\
        & \geq \sum_{t=1}^{T'} \EE_{r\^T}\ps{\vxtr[a_v]}
    \end{align*}
    which gives \eqref{eq:case1sum}, as desired.
\end{proof}

\begin{lemma}[Case 2]\label{lem:case2}
        Let $v$ be a node in the tree of depth $d \leq D-2$.  Then,
        {\upshape
        \begin{align}
            24D^3\EE_{r\^T}\ps{\max_{a' \in [N]} \Swap{\dot{a}_v}{a'}} \geq \sum_{t=1}^{T'} \EE_{r\^T}\ps{\vxtr[\dot{a}_v]} \label{eq:case2sum}
        \end{align}
        }
    \end{lemma}
\begin{proof}[Proof of Lemma \ref{lem:case2}]
    We have
    \begin{align*}
        &\Swap{\dot{a}_v}{a_{vR}}=\sum_{t=1}^T \vxtr[\dot{a}_v]\p{\vut[a_{vR}]-\vut[\dot{a}_v]}\\
        &= \frac{d+1}{2D} \p{\sum_{t=\tl_{vRL}}^{\tr_{vRL}} \vxtr[\dot{a}_v] + \1 \ps{r_{vR}=0} \sum_{t=\tl_{vRR}}^{\tr_{vRR}} \vxtr[\dot{a}_v]}- \frac{d}{2D} \p{\1 \ps{r_{v}=1} \sum_{t=\tl_{vR}}^{\tr_{vR}} \vxtr[\dot{a}_v]}\\
        &= \frac{d+1}{2D} \p{\sum_{t=\tl_{vRL}}^{\tr_{vRL}} \vxtr[\dot{a}_v] + \1 \ps{r_{vR}=0} \sum_{t=\tl_{vRR}}^{\tr_{vRR}} \vxtr[\dot{a}_v]}- \frac{d}{2D} \p{\sum_{t=\tl_{vR}}^{\tr_{vR}} \vxtr[\dot{a}_v]}
    \end{align*}
    bounding $\1[r_v = 1] \leq 1$.  Therefore,
    \begin{align}
        \1[r_{vR}=0] \cdot \Swap{\dot{a}_v}{a_{vR}} \geq \frac{1}{2D} \1[r_{vR}=0] \sum_{t=\tl_{vRL}}^{\tr_{vRL}} \vxtr[\dot{a}_v] + \frac{1}{2D} \sum_{t=\tl_{vRR}}^{\tr_{vRR}} \vxtr[\dot{a}_v]\1[r_{vR}=0] \label{eq:independence61}
    \end{align}
    and
    \begin{align}
        \EE_{r\^T}\ps{\1[r_{vR}=0] \sum_{t=\tl_{vRL}}^{\tr_{vRL}} \vxtr[\dot{a}_v]} &= \EE_{r\^T}\ps{\1[r_{vR}=0]}\EE_{r\^T}\ps{ \sum_{t=\tl_{vRL}}^{\tr_{vRL}} \vxtr[\dot{a}_v]} \nonumber\\
        &= \frac{1}{2} \sum_{t=\tl_{vRL}}^{\tr_{vRL}} \EE_{r\^T}\ps{\vxtr[\dot{a}_v]} \label{eq:independence62}
    \end{align}
    due to the independence of $r_{vR}$ from $r\^t$ for $t \leq \tr_{vRL}$. So, combining equations \eqref{eq:independence61}, and \eqref{eq:independence62}, we have
    \begin{align*}
        \EE_{r\^T}\ps{\1[r_{vR}=0] \cdot \Swap{\dot{a}_v}{a_{vR}}}\geq \frac{1}{4D} \sum_{t=\tl_{vRL}}^{\tr_{vRL}} \EE_{r\^T}\ps{\vxtr[\dot{a}_v]} + \frac{1}{2D} \sum_{t=\tl_{vRR}}^{\tr_{vRR}} \EE_{r\^T}\ps{\vxtr[\dot{a}_v]\1[r_{vR}=0]}
    \end{align*}
    Lastly, since $\Swap{\dot{a}_v}{\dot{a}_v} = 0$,
    \begin{align}
        \EE_{r\^T}\ps{\max_{a' \in [N]} \Swap{\dot{a}_v}{a'}}&\geq \EE_{r\^T}\ps{\max\set{\Swap{\dot{a}_v}{a_{vR}},0}}\nonumber\\
        &\geq \EE_{r\^T}\ps{\1[r_{vR}=0]\cdot\max\set{\Swap{\dot{a}_v}{a_{vR}},0}}\nonumber\\
        &\geq \EE_{r\^T}\ps{\1[r_{vR}=0]\cdot\Swap{\dot{a}_v}{a_{vR}}}\nonumber\\
        &\geq \frac{1}{4D} \sum_{t=\tl_{vRL}}^{\tr_{vRL}} \EE_{r\^T}\ps{\vxtr[\dot{a}_v]} + \frac{1}{2D} \sum_{t=\tl_{vRR}}^{\tr_{vRR}} \EE_{r\^T}\ps{\vxtr[\dot{a}_v]\1[r_{vR}=0]}\label{eq:case21}
    \end{align}

    Similarly,
    \begin{align*}
        &\Swap{\dot{a}_v}{\dot{a}_{vR}}=\sum_{t=1}^T \vxtr[\dot{a}_v]\p{\vut[\dot{a}_{vR}]-\vut[\dot{a}_v]}\\
        &= \frac{d+1}{2D} \p{\1 \ps{r_{vR}=1} \sum_{t=\tl_{vRR}}^{\tr_{vRR}} \vxtr[\dot{a}_v]}- \frac{d}{2D} \p{\1 \ps{r_{v}=1} \sum_{t=\tl_{vR}}^{\tr_{vR}} \vxtr[\dot{a}_v]}\\
        &\geq \frac{d+1}{2D} \p{\1 \ps{r_{vR}=1} \sum_{t=\tl_{vRR}}^{\tr_{vRR}} \vxtr[\dot{a}_v]}- \frac{d}{2D} \p{\sum_{t=\tl_{vR}}^{\tr_{vR}} \vxtr[\dot{a}_v]}\\
        \intertext{bounding $\1[r_v = 1] \leq 1$}
        &= \frac{1}{2D} \sum_{t=\tl_{vRR}}^{\tr_{vRR}} \vxtr[\dot{a}_v] \1 \ps{r_{vR}=1} - \frac{d}{2D} \p{\sum_{t=\tl_{vRL}}^{\tr_{vRL}} \vxtr[\dot{a}_v] + \sum_{t=\tl_{vRR}}^{\tr_{vRR}} \vxtr[\dot{a}_v]\1 \ps{r_{vR}=0}}
    \end{align*}
    So, we have
    \begin{align}
        &\EE_{r\^T}\ps{\max_{a' \in [N]} \Swap{\dot{a}_v}{a'}} \geq \EE_{r\^T}\ps{\Swap{\dot{a}_v}{\dot{a}_{vR}}}\nonumber\\
        &\geq \frac{1}{2D} \sum_{t=\tl_{vRR}}^{\tr_{vRR}} \EE_{r\^T}\ps{\vxtr[\dot{a}_v] \1 \ps{r_{vR}=1}} - \frac{d}{2D} \sum_{t=\tl_{vRL}}^{\tr_{vRL}} \EE_{r\^T}\ps{\vxtr[\dot{a}_v]} - \frac{d}{2D} \sum_{t=\tl_{vRR}}^{\tr_{vRR}} \EE_{r\^T}\ps{\vxtr[\dot{a}_v]\1 \ps{r_{vR}=0}}\nonumber\\
        &\geq \frac{1}{2D} \sum_{t=\tl_{vRR}}^{\tr_{vRR}} \EE_{r\^T}\ps{\vxtr[\dot{a}_v] \1 \ps{r_{vR}=1}} - \frac{1}{2} \sum_{t=\tl_{vRL}}^{\tr_{vRL}} \EE_{r\^T}\ps{\vxtr[\dot{a}_v]} - \frac{1}{2} \sum_{t=\tl_{vRR}}^{\tr_{vRR}} \EE_{r\^T}\ps{\vxtr[\dot{a}_v]\1 \ps{r_{vR}=0}}\label{eq:case22}
    \end{align}
    since $d \leq D$.  %

    Similarly, recalling our notation that the empty string $\emptyset$ represents the root of the tree,
    \begin{align*}
        &\Swap{\dot{a}_v}{a_{\emptyset}}=\sum_{t=1}^T \vxtr[\dot{a}_v]\p{\vut[a_{\emptyset}]-\vut[\dot{a}_v]}\\
        &= \frac{1}{2D} \p{\sum_{t=\tl_{\emptyset L}}^{\tr_{\emptyset L}} \vxtr[\dot{a}_v] + \1 \ps{r_{\emptyset}=0} \sum_{t=\tl_{\emptyset R}}^{\tr_{\emptyset R}} \vxtr[\dot{a}_v]}- \frac{d}{2D} \p{\1 \ps{r_{v}=1} \sum_{t=\tl_{vR}}^{\tr_{vR}} \vxtr[\dot{a}_v]}\\
        &\geq \frac{1}{2D} \p{\sum_{t=\tl_{\emptyset L}}^{\tr_{\emptyset L}} \vxtr[\dot{a}_v] + \1 \ps{r_{\emptyset}=0} \sum_{t=\tl_{\emptyset R}}^{\tr_{\emptyset R}} \vxtr[\dot{a}_v]}- \frac{d}{2D} \p{\sum_{t=\tl_{vR}}^{\tr_{vR}} \vxtr[\dot{a}_v]}\\
    \end{align*}
    bounding $\1[r_v = 1] \leq 1$.  Also,
    \begin{align*}
        &\Swap{\dot{a}_v}{\dot{a}_{\emptyset}}=\sum_{t=1}^T \vxtr[\dot{a}_v]\p{\vut[\dot{a}_{\emptyset}]-\vut[\dot{a}_v]}\\
        &= \frac{1}{2D} \p{\1 \ps{r_{\emptyset}=1} \sum_{t=\tl_{\emptyset R}}^{\tr_{\emptyset R}} \vxtr[\dot{a}_v]}- \frac{d}{2D} \p{\1 \ps{r_{v}=1} \sum_{t=\tl_{vR}}^{\tr_{vR}} \vxtr[\dot{a}_v]}\\
        &\geq \frac{1}{2D} \p{\1 \ps{r_{\emptyset}=1} \sum_{t=\tl_{\emptyset R}}^{\tr_{\emptyset R}} \vxtr[\dot{a}_v]}- \frac{d}{2D} \p{\sum_{t=\tl_{vR}}^{\tr_{vR}} \vxtr[\dot{a}_v]}
    \end{align*}
    So,
    \begin{align*}
        \Swap{\dot{a}_v}{a_{\emptyset}}+ \Swap{\dot{a}_v}{\dot{a}_{\emptyset}} &\geq \frac{1}{2D} \p{\sum_{t=1}^{T'} \vxtr[\dot{a}_v]}- \frac{d}{D} \p{\sum_{t=\tl_{vR}}^{\tr_{vR}} \vxtr[\dot{a}_v]}\\
        &\geq \frac{1}{2D} \p{\sum_{t=1}^{T'} \vxtr[\dot{a}_v]}- \p{\sum_{t=\tl_{vR}}^{\tr_{vR}} \vxtr[\dot{a}_v]}
    \end{align*}
    since $d \leq D$. Thus,
    \begin{align}
        \EE_{r\^T}\ps{\max_{a' \in [N]} \Swap{\dot{a}_v}{a'}} &\geq \frac{1}{2}\EE_{r\^t}\ps{\Swap{\dot{a}_v}{a_{\emptyset}}+ \Swap{\dot{a}_v}{\dot{a}_{\emptyset}}}\nonumber\\
        &\geq \frac{1}{4D} \sum_{t=1}^{T'} \EE_{r\^T}\ps{\vxtr[\dot{a}_v]}- \frac{1}{2}\sum_{t=\tl_{vR}}^{\tr_{vR}} \EE_{r\^T}\ps{\vxtr[\dot{a}_v]}\label{eq:case23}
    \end{align}

    Collecting equations \eqref{eq:case21}, \eqref{eq:case22}, and \eqref{eq:case23},

    \begin{align*}
        &\EE_{r\^T}\ps{\max_{a' \in [N]} \Swap{\dot{a}_v}{a'}}\nonumber\\
        &\geq \frac{1}{4D} \sum_{t=\tl_{vRL}}^{\tr_{vRL}} \EE_{r\^T}\ps{\vxtr[\dot{a}_v]} + \frac{1}{2D} \sum_{t=\tl_{vRR}}^{\tr_{vRR}} \EE_{r\^T}\ps{\vxtr[\dot{a}_v]\1[r_{vR}=0]} &\text{\eqref{eq:case21}}\\
        &\EE_{r\^T}\ps{\max_{a' \in [N]} \Swap{\dot{a}_v}{a'}}\nonumber\\
        &\geq \frac{1}{2D} \sum_{t=\tl_{vRR}}^{\tr_{vRR}} \EE_{r\^T}\ps{\vxtr[\dot{a}_v] \1 \ps{r_{vR}=1}} - \frac{1}{2} \sum_{t=\tl_{vRL}}^{\tr_{vRL}} \EE_{r\^T}\ps{\vxtr[\dot{a}_v]} - \frac{1}{2} \sum_{t=\tl_{vRR}}^{\tr_{vRR}} \EE_{r\^T}\ps{\vxtr[\dot{a}_v]\1 \ps{r_{vR}=0}} &\text{\eqref{eq:case22}}\\
        &\EE_{r\^T}\ps{\max_{a' \in [N]} \Swap{\dot{a}_v}{a'}}\nonumber\\
        &\geq \frac{1}{4D} \sum_{t=1}^{T'} \EE_{r\^T}\ps{\vxtr[\dot{a}_v]}- \frac{1}{2}\sum_{t=\tl_{vR}}^{\tr_{vR}} \EE_{r\^T}\ps{\vxtr[\dot{a}_v]} &\text{\eqref{eq:case23}}
    \end{align*}

    Summing the inequalities
    $$4D\text{\eqref{eq:case23}}+4D^2\text{\eqref{eq:case22}}+16D^3\text{\eqref{eq:case21}}$$ gives

    \begin{align}
        24D^3\EE_{r\^T}\ps{\max_{a' \in [N]} \Swap{\dot{a}_v}{a'}}&\geq \sum_{t=1}^{T'} \EE_{r\^T}\ps{\vxtr[\dot{a}_v]}\nonumber\\
        &+(-2D-2D^2+4D^2)\sum_{t=\tl_{vRL}}^{\tr_{vRL}} \EE_{r\^T}\ps{\vxtr[\dot{a}_v]}\nonumber\\
        &+(-2D-2D^2+8D^2)\sum_{t=\tl_{vRR}}^{\tr_{vRR}} \EE_{r\^T}\ps{\vxtr[\dot{a}_v]\1[r_{vR}=0]}\nonumber\\
        &+(-2D+2D)\sum_{t=\tl_{vRR}}^{\tr_{vRR}} \EE_{r\^T}\ps{\vxtr[\dot{a}_v]\1[r_{vR}=1]}\nonumber\\
        & \geq \sum_{t=1}^{T'} \EE_{r\^T}\ps{\vxtr[\dot{a}_v]} %
    \end{align}
    which gives \eqref{eq:case2sum}, as desired.
\end{proof}

\begin{lemma}[Case 3]\label{lem:case3}
        Let $v$ be a node in the tree of depth $d = D-1$.  Then,
        {\upshape
        \begin{align}
            24D^3\EE_{r\^T}\ps{\max_{a' \in [N]} \Swap{a_v}{a'}}\geq \sum_{t=1}^{T'} \EE_{r\^T}\ps{\vxtr[a_v]} \label{eq:case3sum}
        \end{align}
        }
    \end{lemma}
\begin{proof}[Proof of Lemma \ref{lem:case3}]
    In this case, we have nodes $vL$ and $vR$ both leaves.  So, there are Bernoulli random variables $r_{(vL,t)}$ and $r_{(vR,t)}$ for every time step $t$ in $[\tl_{vL},\tr_{vL}]$ and $[\tl_{vR},\tr_{vR}]$ respectively.  By definition,
    \begin{align*}
        \vut[a_{vL}] &= \1 \ps{r_{(vL,t)}=0} & \vut[\dot{a}_{vL}] &= \1 \ps{r_{(vL,t)}=1}\\
        \vut[a_{vR}] &= \1 \ps{r_{(vR,t)}=0} & \vut[\dot{a}_{vR}] &= \1 \ps{r_{(vR,t)}=1}
    \end{align*}
    for $t$ in $[\tl_{vL},\tr_{vL}]$ and $[\tl_{vR},\tr_{vR}]$ respectively.  We have
    \begin{align*}
        &\Swap{a_v}{a_{vL}}+\Swap{a_v}{\dot{a}_{vL}}=\sum_{t=1}^T \vxtr[a_v]\p{\vut[a_{vL}]+\vut[\dot{a}_{vL}]-2\vut[a_v]}\\
        &= \p{\sum_{t=\tl_{vL}}^{\tr_{vL}} \vxtr[a_v]} - \frac{2(D-1)}{2D} \p{\sum_{t=\tl_{vL}}^{\tr_{vL}} \vxtr[a_v] + \1 \ps{r_{v}=0} \sum_{t=\tl_{vR}}^{\tr_{vR}} \vxtr[a_v]}
    \end{align*}
    Therefore,
    \begin{align*}
        \1[r_v=1]\cdot\p{\Swap{a_v}{a_{vL}}+\Swap{a_v}{\dot{a}_{vL}}}= \frac{1}{D} \1 \ps{r_{v}=1} \sum_{t=\tl_{vL}}^{\tr_{vL}} \vxtr[a_v]
    \end{align*}
    and
    \begin{align*}
        \EE_{r\^T}\ps{\1 \ps{r_{v}=1} \sum_{t=\tl_{vL}}^{\tr_{vL}} \vxtr[a_v]} = \EE_{r\^T}\ps{\1 \ps{r_{v}=1}}\EE_{r\^T}\ps{ \sum_{t=\tl_{vL}}^{\tr_{vL}} \vxtr[a_v]}= \frac{1}{2} \sum_{t=\tl_{vL}}^{\tr_{vL}} \EE_{r\^T}\ps{ \vxtr[a_v]}
    \end{align*}
    due to the independence of $r_{v}$ from $r\^t$ for $t \leq \tr_{vL}$.  Lastly, since $\Swap{a_v}{a_v} = 0$,
    \begin{align}
        \EE_{r\^T}\ps{\max_{a' \in [N]} \Swap{a_v}{a'}}&\geq \EE_{r\^T}\ps{\max\set{\frac{1}{2}\p{\Swap{a_v}{a_{vL}}+\Swap{a_v}{\dot{a}_{vL}}},0}}\nonumber\\
        &\geq \EE_{r\^T}\ps{\1[r_{v}=1]\cdot\max\set{\frac{1}{2}\p{\Swap{a_v}{a_{vL}}+\Swap{a_v}{\dot{a}_{vL}}},0}}\nonumber\\
        &\geq \frac{1}{2}\EE_{r\^T}\ps{\1[r_{v}=1]\cdot\p{\Swap{a_v}{a_{vL}}+\Swap{a_v}{\dot{a}_{vL}}}}\nonumber\\
        &\geq \frac{1}{4D} \sum_{t=\tl_{vL}}^{\tr_{vL}} \EE_{r\^T}\ps{\vxtr[a_v]}\label{eq:case31}
    \end{align}

    Similarly
    \begin{align*}
        &\Swap{a_v}{a_{vR}}+\Swap{a_v}{\dot{a}_{vR}}=\sum_{t=1}^T \vxtr[a_v]\p{\vut[a_{vR}]+\vut[\dot{a}_{vR}]-2\vut[a_v]}\\
        &= \p{\sum_{t=\tl_{vR}}^{\tr_{vR}} \vxtr[a_v]} - \frac{2(D-1)}{2D} \p{\sum_{t=\tl_{vL}}^{\tr_{vL}} \vxtr[a_v] + \1 \ps{r_{v}=0} \sum_{t=\tl_{vR}}^{\tr_{vR}} \vxtr[a_v]}\\
        &\geq \frac{1}{D}\sum_{t=\tl_{vR}}^{\tr_{vR}} \vxtr[a_v] - \sum_{t=\tl_{vL}}^{\tr_{vL}} \vxtr[a_v]
    \end{align*}
    bounding $\1 \ps{r_{v}=0}\leq 1$.  Thus,
    \begin{align}
        \EE_{r\^T}\ps{\max_{a' \in [N]} \Swap{a_v}{a'}}&\geq \EE_{r\^T}\ps{\frac{1}{2}\p{\Swap{a_v}{a_{vR}}+\Swap{a_v}{\dot{a}_{vR}}}}\nonumber\\
        &\geq \frac{1}{2D} \sum_{t=\tl_{vR}}^{\tr_{vR}} \EE_{r\^T}\ps{\vxtr[a_v]}-\frac{1}{2} \sum_{t=\tl_{vL}}^{\tr_{vL}} \EE_{r\^T}\ps{\vxtr[a_v]}\label{eq:case32}
    \end{align}

    Similarly
    \begin{align*}
        &\Swap{a_v}{a_{\emptyset}}+\Swap{a_v}{\dot{a}_{\emptyset}}=\sum_{t=1}^T \vxtr[a_v]\p{\vut[a_{\emptyset}]+\vut[\dot{a}_{\emptyset}]-2\vut[a_v]}\\
        &= \frac{1}{2D}\p{\sum_{t=1}^{T'} \vxtr[a_v]} - \frac{2(D-1)}{2D} \p{\sum_{t=\tl_{vL}}^{\tr_{vL}} \vxtr[a_v] + \1 \ps{r_{v}=0} \sum_{t=\tl_{vR}}^{\tr_{vR}} \vxtr[a_v]}\\
        &\geq \frac{1}{2D}\sum_{t=1}^{T'} \vxtr[a_v] - \sum_{t=\tl_{v}}^{\tr_{v}} \vxtr[a_v]
    \end{align*}
    bounding $\1 \ps{r_{v}=0}\leq 1$.  Thus,
    \begin{align}
        \EE_{r\^T}\ps{\max_{a' \in [N]} \Swap{a_v}{a'}}&\geq \EE_{r\^T}\ps{\frac{1}{2}\p{\Swap{a_v}{a_{\emptyset}}+\Swap{a_v}{\dot{a}_{\emptyset}}}}\nonumber\\
        &\geq \frac{1}{4D} \sum_{t=1}^{T'} \EE_{r\^T}\ps{\vxtr[a_v]}-\frac{1}{2} \sum_{t=\tl_{v}}^{\tr_{v}} \EE_{r\^T}\ps{\vxtr[a_v]}\label{eq:case33}
    \end{align}

    Collecting equations \eqref{eq:case31}, \eqref{eq:case32}, and \eqref{eq:case33},
    \begin{align*}
        \EE_{r\^T}\ps{\max_{a' \in [N]} \Swap{a_v}{a'}}&\geq \frac{1}{4D} \sum_{t=\tl_{vL}}^{\tr_{vL}} \EE_{r\^T}\ps{\vxtr[a_v]} &\text{\eqref{eq:case31}}\\
        \EE_{r\^T}\ps{\max_{a' \in [N]} \Swap{a_v}{a'}}&\geq \frac{1}{2D} \sum_{t=\tl_{vR}}^{\tr_{vR}} \EE_{r\^T}\ps{\vxtr[a_v]}-\frac{1}{2} \sum_{t=\tl_{vL}}^{\tr_{vL}} \EE_{r\^T}\ps{\vxtr[a_v]}&\text{\eqref{eq:case32}}\\
        \EE_{r\^T}\ps{\max_{a' \in [N]} \Swap{a_v}{a'}}&\geq \frac{1}{4D} \sum_{t=1}^{T'} \EE_{r\^T}\ps{\vxtr[a_v]}-\frac{1}{2} \sum_{t=\tl_{v}}^{\tr_{v}} \EE_{r\^T}\ps{\vxtr[a_v]} &\text{\eqref{eq:case33}}
    \end{align*}

    Summing the inequalities $$4D\text{\eqref{eq:case33}}+4D^2\text{\eqref{eq:case32}}+16D^3\text{\eqref{eq:case31}}$$ gives
    \begin{align*}
        24D^3\EE_{r\^T}\ps{\max_{a' \in [N]} \Swap{a_v}{a'}}&\geq \sum_{t=1}^{T'} \EE_{r\^T}\ps{\vxtr[a_v]}\nonumber\\
        &+(-2D-2D^2+4D^2) \sum_{t=\tl_{vL}}^{\tr_{vL}} \EE_{r\^T}\ps{\vxtr[a_v]}\nonumber\\
        &+(-2D+2D) \sum_{t=\tl_{vR}}^{\tr_{vR}} \EE_{r\^T}\ps{\vxtr[a_v]}\nonumber\\
        & \geq \sum_{t=1}^{T'} \EE_{r\^T}\ps{\vxtr[a_v]} %
    \end{align*}
    which gives \eqref{eq:case3sum}, as desired.
\end{proof}

\begin{lemma}[Case 4]\label{lem:case4}
        Let $v$ be a node in the tree of depth $d = D-1$.  Then,
        {\upshape
        \begin{align}
            8D^2\EE_{r\^T}\ps{\max_{a' \in [N]} \Swap{\dot{a}_v}{a'}} \geq \sum_{t=1}^{T'} \EE_{r\^T}\ps{\vxtr[\dot{a}_v]} \label{eq:case4sum}
        \end{align}
        }
    \end{lemma}
\begin{proof}[Proof of Lemma \ref{lem:case4}]
    In this case, we have that node $vR$ is a leaf.  So, there are Bernoulli random variables $r_{(vR,t)}$ for every time step $t\in [\tl_{vR},\tr_{vR}]$.  By definition,
    \begin{align*}
        \vut[a_{vR}] &= \1 \ps{r_{(vR,t)}=0} & \vut[\dot{a}_{vR}] &= \1 \ps{r_{(vR,t)}=1}
    \end{align*}
    for $t \in [\tl_{vR},\tr_{vR}]$.  We have
    \begin{align*}
        \Swap{\dot{a}_v}{a_{vR}}+\Swap{\dot{a}_v}{\dot{a}_{vR}}&=\sum_{t=1}^T \vxtr[\dot{a}_v]\p{\vut[a_{vR}]+\vut[\dot{a}_{vR}]-2\vut[\dot{a}_v]}\\
        &= \p{\sum_{t=\tl_{vR}}^{\tr_{vR}} \vxtr[\dot{a}_v]} - \frac{2(D-1)}{2D} \p{\1 \ps{r_{v}=1} \sum_{t=\tl_{vR}}^{\tr_{vR}} \vxtr[\dot{a}_v]}\\
        &\geq \frac{1}{D}\sum_{t=\tl_{vR}}^{\tr_{vR}} \vxtr[\dot{a}_v]
    \end{align*}
    bounding $\1 \ps{r_{v}=1}\leq 1$.  Thus,
    \begin{align}
        \EE_{r\^T}\ps{\max_{a' \in [N]} \Swap{\dot{a}_v}{a'}}&\geq \EE_{r\^T}\ps{\frac{1}{2}\p{\Swap{\dot{a}_v}{a_{vR}}+\Swap{\dot{a}_v}{\dot{a}_{vR}}}}\nonumber\\
        &\geq \frac{1}{2D} \sum_{t=\tl_{vR}}^{\tr_{vR}} \EE_{r\^T}\ps{\vxtr[\dot{a}_v]}\label{eq:case41}
    \end{align}

    Similarly
    \begin{align*}
        &\Swap{\dot{a}_v}{a_{\emptyset}}+\Swap{\dot{a}_v}{\dot{a}_{\emptyset}}=\sum_{t=1}^T \vxtr[\dot{a}_v]\p{\vut[a_{\emptyset}]+\vut[\dot{a}_{\emptyset}]-2\vut[\dot{a}_v]}\\
        &= \frac{1}{2D}\p{\sum_{t=1}^{T'} \vxtr[\dot{a}_v]} - \frac{2(D-1)}{2D} \p{\1 \ps{r_{v}=1} \sum_{t=\tl_{vR}}^{\tr_{vR}} \vxtr[\dot{a}_v]}\\
        &\geq \frac{1}{2D}\sum_{t=1}^{T'} \vxtr[\dot{a}_v] - \sum_{t=\tl_{vR}}^{\tr_{vR}} \vxtr[\dot{a}_v]
    \end{align*}
    bounding $\1 \ps{r_{v}=1}\leq 1$.  Thus,
    \begin{align}
        \EE_{r\^T}\ps{\max_{a' \in [N]} \Swap{\dot{a}_v}{a'}}&\geq \EE_{r\^T}\ps{\frac{1}{2}\p{\Swap{\dot{a}_v}{a_{\emptyset}}+\Swap{\dot{a}_v}{\dot{a}_{\emptyset}}}}\nonumber\\
        &\geq \frac{1}{4D} \sum_{t=1}^{T'} \EE_{r\^T}\ps{\vxtr[\dot{a}_v]}-\frac{1}{2} \sum_{t=\tl_{vR}}^{\tr_{vR}} \EE_{r\^T}\ps{\vxtr[\dot{a}_v]}\label{eq:case42}
    \end{align}

    Collecting equations \eqref{eq:case41} and \eqref{eq:case42},

    \begin{align*}
        \EE_{r\^T}\ps{\max_{a' \in [N]} \Swap{a_v}{a'}}&\geq \frac{1}{2D} \sum_{t=\tl_{vR}}^{\tr_{vR}} \EE_{r\^T}\ps{\vxtr[\dot{a}_v]}&\text{\eqref{eq:case41}}\\
        \EE_{r\^T}\ps{\max_{a' \in [N]} \Swap{a_v}{a'}}&\geq \frac{1}{4D} \sum_{t=1}^{T'} \EE_{r\^T}\ps{\vxtr[\dot{a}_v]}-\frac{1}{2} \sum_{t=\tl_{vR}}^{\tr_{vR}} \EE_{r\^T}\ps{\vxtr[\dot{a}_v]}&\text{\eqref{eq:case42}}
    \end{align*}

    Summing the inequalities $$4D\text{\eqref{eq:case41}}+4D^2\text{\eqref{eq:case42}}$$ gives
    \begin{align*}
        8D^2\EE_{r\^T}\ps{\max_{a' \in [N]} \Swap{\dot{a}_v}{a'}}&\geq \sum_{t=1}^{T'} \EE_{r\^T}\ps{\vxtr[\dot{a}_v]}\nonumber\\
        &+(-2D+2D) \sum_{t=\tl_{vR}}^{\tr_{vR}} \EE_{r\^T}\ps{\vxtr[\dot{a}_v]}\nonumber\\
        & \geq \sum_{t=1}^{T'} \EE_{r\^T}\ps{\vxtr[\dot{a}_v]} %
    \end{align*}
    which gives \eqref{eq:case4sum}, as desired.
\end{proof}

\begin{lemma}[Case 5]\label{lem:case5}
        Let $v$ be a leaf node in the tree (depth $d = D$).  Then,
        {\upshape
        \begin{align}
            4100D^4\sqrt{B}\p{\EE_{r\^T}\ps{\max_{a' \in [N]} \Swap{a_v}{a'}} + \EE_{r\^T}\ps{\max_{a' \in [N]} \Swap{\dot{a}_v}{a'}}}\nonumber\\
            \geq \sum_{t=1}^{T'} \EE_{r\^T}\ps{\vxtr[a_{v}]+\vxtr[\dot{a}_{v}]} - \frac{B}{2} \label{eq:case5sum}
        \end{align}
        }    
    \end{lemma}
\begin{proof}[Proof of Lemma \ref{lem:case5}]
    Since $v$ is a leaf node, there are Bernoulli random variables $r_{(v,t)}$ for every time step $t$ in $[\tl_{v},\tr_{v}]$.  By definition,
    \begin{align*}
        \vut[a_{v}] &= \1 \ps{r_{(v,t)}=0} & \vut[\dot{a}_{v}] &= \1 \ps{r_{(v,t)}=1}
    \end{align*}
    for $t \in [\tl_{v},\tr_{v}]$.  We have
    \begin{align}
        \EE_{r\^T}\ps{\max_{a' \in [N]} \Swap{a_v}{a'}}&\geq \EE_{r\^T}\ps{\frac{1}{2}\p{\Swap{a_v}{a_{\emptyset}}+\Swap{a_v}{\dot{a}_{\emptyset}}}}\nonumber\\
        &=\frac{1}{2}\sum_{t=1}^T \EE_{r\^T}\ps{\vxtr[a_v]\p{\vut[a_{\emptyset}]+\vut[\dot{a}_{\emptyset}]-2\vut[a_v]}}\nonumber\\
        &\geq \frac{1}{4D}\sum_{t=1}^{T'} \EE_{r\^T}\ps{\vxtr[a_v]} - \sum_{t=\tl_{v}}^{\tr_{v}} \EE_{r\^T}\ps{\vxtr[a_v]} \label{eq:case51}
    \end{align}

    Similarly,
    \begin{align}
        \EE_{r\^T}\ps{\max_{a' \in [N]} \Swap{\dot{a}_v}{a'}}&\geq \EE_{r\^T}\ps{\frac{1}{2}\p{\Swap{\dot{a}_v}{a_{\emptyset}}+\Swap{\dot{a}_v}{\dot{a}_{\emptyset}}}}\nonumber\\
        &=\frac{1}{2}\sum_{t=1}^T \EE_{r\^T}\ps{\vxtr[\dot{a}_v]\p{\vut[a_{\emptyset}]+\vut[\dot{a}_{\emptyset}]-2\vut[\dot{a}_v]}}\nonumber\\
        &\geq \frac{1}{4D}\sum_{t=1}^{T'} \EE_{r\^T}\ps{\vxtr[\dot{a}_v]} - \sum_{t=\tl_{v}}^{\tr_{v}} \EE_{r\^T}\ps{\vxtr[\dot{a}_v]}\label{eq:case52}
    \end{align}

    Now,
    \begin{align*}
        &\EE_{r\^T}\ps{\max_{a' \in [N]} \Swap{a_v}{a'}}\geq \EE_{r\^T}\ps{\max\set{\Swap{a_v}{a_v},\Swap{a_v}{\dot{a}_v}}}\\
        &=\EE_{r\^T}\ps{\max\set{\sum_{t=1}^T \vxtr[a_v]\p{\vut[a_{v}]-\vut[a_v]},\sum_{t=1}^T \vxtr[a_v]\p{\vut[\dot{a}_{v}]-\vut[a_v]}}}\\
        &=\EE_{r\^T}\ps{\max\set{\sum_{t=\tl_v}^{\tr_v} \vxtr[a_v]\1 \ps{r_{(v,t)}=0},\sum_{t=\tl_v}^{\tr_v} \vxtr[a_v]\1 \ps{r_{(v,t)}=1}}}-\sum_{t=\tl_v}^{\tr_v} \EE_{r\^T}\ps{\vxtr[a_v]\1 \ps{r_{(v,t)}=0}}
    \end{align*}
    Similarly,
    \begin{align*}
        &\EE_{r\^T}\ps{\max_{a' \in [N]} \Swap{\dot{a}_v}{a'}}\\
        &\geq \EE_{r\^T}\ps{\max\set{\sum_{t=\tl_v}^{\tr_v} \vxtr[\dot{a}_v]\1 \ps{r_{(v,t)}=0},\sum_{t=\tl_v}^{\tr_v} \vxtr[\dot{a}_v]\1 \ps{r_{(v,t)}=1}}}-\sum_{t=\tl_v}^{\tr_v} \EE_{r\^T}\ps{\vxtr[\dot{a}_v]\1 \ps{r_{(v,t)}=1}}
    \end{align*}
    Now,
    \begin{align*}
        \EE_{r\^T}\ps{\vxtr[a_v]\1 \ps{r_{(v,t)}=0}} &= \EE_{r\^T}\ps{\vxtr[a_v]}\EE_{r\^T}\ps{\1 \ps{r_{(v,t)}=0}}\\
        &= \frac{1}{2}\EE_{r\^T}\ps{\vxtr[a_v]}
    \end{align*}
    and
    \begin{align*}
        \EE_{r\^T}\ps{\vxtr[\dot{a}_v]\1 \ps{r_{(v,t)}=1}}= \frac{1}{2}\EE_{r\^T}\ps{\vxtr[\dot{a}_v]}
    \end{align*}
    Also, using as shorthand $\vxtr[a_v+\dot{a}_v]=\vxtr[a_v]+\vxtr[\dot{a}_v]$, we have
    \begin{align*}
        &\max\set{\sum_{t=\tl_v}^{\tr_v} \vxtr[a_v]\1 \ps{r_{(v,t)}=0},\sum_{t=\tl_v}^{\tr_v} \vxtr[a_v]\1 \ps{r_{(v,t)}=1}}\\
        +&\max\set{\sum_{t=\tl_v}^{\tr_v} \vxtr[\dot{a}_v]\1 \ps{r_{(v,t)}=0},\sum_{t=\tl_v}^{\tr_v} \vxtr[\dot{a}_v]\1 \ps{r_{(v,t)}=1}}\\
        \geq&\max\set{\sum_{t=\tl_v}^{\tr_v} \vxtr[a_v+\dot{a}_v]\1 \ps{r_{(v,t)}=0},\sum_{t=\tl_v}^{\tr_v} \vxtr[a_v+\dot{a}_v]\1 \ps{r_{(v,t)}=1}}
    \end{align*}
    Therefore,
    \begin{align*}
        & \EE_{r\^T}\ps{\max_{a' \in [N]} \Swap{a_v}{a'}} + \EE_{r\^T}\ps{\max_{a' \in [N]} \Swap{\dot{a}_v}{a'}}\\
        & \geq \EE_{r\^T}\ps{\max\set{\sum_{t=\tl_v}^{\tr_v} \vxtr[a_v+\dot{a}_v]\1 \ps{r_{(v,t)}=0},\sum_{t=\tl_v}^{\tr_v} \vxtr[a_v+\dot{a}_v]\1 \ps{r_{(v,t)}=1}}}\\
        &- \frac{1}{2} \EE_{r\^T}\ps{\sum_{t=\tl_v}^{\tr_v} \vxtr[a_v+\dot{a}_v]}\\
        &=\frac{1}{2} \EE_{r\^T}\ps{\card{\sum_{t=\tl_v}^{\tr_v} \vxtr[a_v+\dot{a}_v] Z\^{t}}}
    \end{align*}
    where $Z\^t =\begin{cases}
        1 &\text{ if }r_{(v,t)}=1\\
        -1 &\text{ if }r_{(v,t)}=0
    \end{cases}$.  If we denote by $X\^t = \sum_{\tau = \tl_v}^{t} \vxtr[a_v+\dot{a}_v] Z\^{t}$, then $X\^{\tl_v},\cdots,X\^{\tr_v}$ is a martingale due to the independence of $\vxtr[a_v+\dot{a}_v]$ and $Z\^{t}$ for all $t$.  We use the following lemma.  %
    \begin{lemma}\label{lem:martingale}
        Consider an algorithm which, at any round $t=1,\dots,B$, selects a parameter $x_t \in [0,1]$, possibly at random. Then, it observes the outcome of a random variable $Z_t \sim \mathrm{Uniform}(\{-1,1\})$. Assume that $\EE\ps{\sum_{t=1}^B x_t^2} \ge \epsilon B$ for some $\epsilon>0$. Then,
        \[
        \EE\ps{\card{ \sum_{t=1}^B x_t Z_t }} \ge \frac{\epsilon \sqrt{B}}{4\sqrt{\log(1/\epsilon)}}
        \]
    \end{lemma}

    We want to apply the lemma with $x_t = \frac{1}{2}\EE_{r\^T}\ps{\vxtr[a_v+\dot{a}_v]}$.  We have $\frac{1}{2}\vxtr[a_v+\dot{a}_v] \leq 1$ for all $t$.  If we assume, 
    \begin{equation}\label{eq:case5assumption}
        \sum_{t=\tl_v}^{\tr_v} \EE_{r\^T}\ps{\vxtr[a_v+\dot{a}_v]} \geq \frac{B}{8D}
    \end{equation}
    Jensen's inequality gives
    \begin{align*}
        \sum_{t=\tl_v}^{\tr_v} \p{\frac12 \EE_{r\^T}\ps{\vxtr[a_v+\dot{a}_v]}}^2 &\geq \frac{B}{4} \p{\frac{1}{B} \sum_{t=\tl_v}^{\tr_v} \EE_{r\^T}\ps{\vxtr[a_v+\dot{a}_v]}}^2 \geq \frac{B}{256D^2}
    \end{align*}
    and the preconditions of Lemma \eqref{lem:martingale} hold for $\epsilon=\frac{1}{256D^2}$.  Thus,
    \begin{align*}
        \EE_{r\^T}\ps{\max_{a' \in [N]} \Swap{a_v}{a'}} + \EE_{r\^T}\ps{\max_{a' \in [N]} \Swap{\dot{a}_v}{a'}}&\geq \EE_{r\^T}\ps{\card{\sum_{t=\tl_v}^{\tr_v} \p{\frac{1}{2} \vxtr[a_v+\dot{a}_v]} Z\^{t}}}\\
        &\geq \frac{\sqrt{B}}{1024 D^3}
    \end{align*}
    for $D \geq 3$.  Equivalently, incorporating assumption \eqref{eq:case5assumption} into the equation,
    \begin{align}
        &\EE_{r\^T}\ps{\max_{a' \in [N]} \Swap{a_{v}}{a'}} + \EE_{r\^T}\ps{\max_{a' \in [N]} \Swap{\dot{a}_{v}}{a'}}\nonumber\\
        &\geq \frac{\sqrt{B}}{1024 D^3} \1\ps{\sum_{t=\tl_{v}}^{\tr_{v}} \EE_{r\^T}\ps{\vxtr[a_{v}+\dot{a}_{v}]} \geq \frac{B}{8D}}\nonumber\\
        &\geq \frac{\sqrt{B}}{1024 D^3} \frac{1}{B}\p{\sum_{t=\tl_{v}}^{\tr_{v}} \EE_{r\^T}\ps{\vxtr[a_{v}+\dot{a}_{v}]}}\1\ps{\sum_{t=\tl_{v}}^{\tr_{v}} \EE_{r\^T}\ps{\vxtr[a_{v}+\dot{a}_{v}]} \geq \frac{B}{8D}}\nonumber\\
        \intertext{because $\sum_{t=\tl_{v}}^{\tr_{v}} \EE_{r\^T}\ps{\vxtr[a_{v}+\dot{a}_{v}]} \leq B$}
        &\geq \frac{1}{1024 D^3 \sqrt{B}} \p{\p{\sum_{t=\tl_{v}}^{\tr_{v}} \EE_{r\^T}\ps{\vxtr[a_{v}+\dot{a}_{v}]}}-\frac{B}{8D}}\nonumber\\
        &= \frac{1}{1024 D^3 \sqrt{B}} \p{\sum_{t=\tl_{v}}^{\tr_{v}} \EE_{r\^T}\ps{\vxtr[a_{v}+\dot{a}_{v}]}}-\frac{\sqrt{B}}{8192D^4}\label{eq:case53}
    \end{align}

    Collecting equations \eqref{eq:case51}, \eqref{eq:case52}, and \eqref{eq:case53},

    \begin{align*}
        \EE_{r\^T}\ps{\max_{a' \in [N]} \Swap{a_v}{a'}}&\geq \frac{1}{4D}\sum_{t=1}^{T'} \EE_{r\^T}\ps{\vxtr[a_v]} - \sum_{t=\tl_{v}}^{\tr_{v}} \EE_{r\^T}\ps{\vxtr[a_v]}&\text{\eqref{eq:case51}}\\
        \EE_{r\^T}\ps{\max_{a' \in [N]} \Swap{\dot{a}_v}{a'}}&\geq \frac{1}{4D}\sum_{t=1}^{T'} \EE_{r\^T}\ps{\vxtr[\dot{a}_v]} - \sum_{t=\tl_{v}}^{\tr_{v}} \EE_{r\^T}\ps{\vxtr[\dot{a}_v]}&\text{\eqref{eq:case52}}\\
        \EE_{r\^T}\ps{\max_{a' \in [N]} \Swap{a_v}{a'}} &+ \EE_{r\^T}\ps{\max_{a' \in [N]} \Swap{\dot{a}_v}{a'}}\nonumber\\
        &\geq \frac{1}{1024 D^3 \sqrt{B}} \p{\sum_{t=\tl_{v}}^{\tr_{v}} \EE_{r\^T}\ps{\vxtr[a_{v}]+\vxtr[\dot{a}_{v}]}}-\frac{\sqrt{B}}{8192D^4}&\text{\eqref{eq:case53}}
    \end{align*}

    Summing the inequalities 
    $$4D\text{\eqref{eq:case51}}+4D\text{\eqref{eq:case52}}+4096D^4\sqrt{B}\eqref{eq:case53}$$ gives
    \begin{align*}
        &4100D^4\sqrt{B}\p{\EE_{r\^T}\ps{\max_{a' \in [N]} \Swap{a_v}{a'}} + \EE_{r\^T}\ps{\max_{a' \in [N]} \Swap{\dot{a}_v}{a'}}}\\
        &\geq \sum_{t=1}^{T'} \EE_{r\^T}\ps{\vxtr[a_{v}]+\vxtr[\dot{a}_{v}]} +(-4D+4D)\sum_{t=\tl_{v}}^{\tr_{v}} \EE_{r\^T}\ps{\vxtr[a_{v}]+\vxtr[\dot{a}_{v}]} - \frac{B}{2} \nonumber \\
        &=\sum_{t=1}^{T'} \EE_{r\^T}\ps{\vxtr[a_{v}]+\vxtr[\dot{a}_{v}]} - \frac{B}{2} %
    \end{align*}

    which gives \eqref{eq:case5sum}, as desired.
\end{proof}

\begin{proof}[Proof of Lemma \ref{lem:martingale}]
    Denote by $X = \sum_{t=1}^B x_t Z_t$.
    First, it holds that
    \[
    \EE\ps{X^2} = \EE\ps{\sum_{t=1}^B x_t^2}
    \]
    Next, from Azuma's inequality, we know that for any $C>0$, $\Pr[|X| \ge C] \le 2\exp(-C^2/(2B))$. For any $R>0$, we can write
    \begin{align*}
    \epsilon B
    &\le \EE\ps{\sum_{t=1}^B x_t^2}\\
    &= \EE[X^2] \\
    &= \EE[X^2 \1(|X| \le R)]
    + \EE[X^2 \1(|X| > R)] \\
    &\le \EE[|X|R \1(|X| \le R)]
    + \EE[X^2 \1(|X| > R)] \\
    &\le R \EE[|X|] + \EE[X^2 \1(|X| > R)].
    \end{align*}
    Consequently,
    \[
    \EE[|X|] \ge \p{\epsilon B - \EE[X^2 \1(|X| > R)]}/R.
    \]
    Notice that
    \begin{align*}
    \EE[X^2 \1(|X| > R)] 
    &= \int_{R}^\infty y \Pr[|X| \ge y] dy\\
    &\le \int_{R}^\infty y e^{-y^2/(2B)} dy \\
    &= 2B e^{-R^2/(2B)}.
    \end{align*}
    Substituting above, we obtain that
    \[
    \EE[|X|] \ge \frac{B}{R}\p{\epsilon - 2e^{-R^2/(2B)}}
    \]
    Substituting $R = \lambda \sqrt{B}$, we obtain
    \[
    \EE[|X|] \ge \frac{\sqrt{B}}{\lambda}\p{\epsilon - 2e^{-\lambda^2/2}} \ge \frac{\epsilon \sqrt{B}}{4\sqrt{\log(1/\epsilon)}}
    \]
    where the last inequality holds for $\epsilon \leq 0.25$ when setting $\lambda = 2\sqrt{\log(1/\epsilon)}$.
\end{proof}

    \begin{lemma}[Case 6]\label{lem:case6}
        Let $a \in [N]$ such that $a\ne a_v,\dot{a}_v$ for any nodes $v$ in the tree.  Then,
        {\upshape
        \begin{align}
            4D\EE_{r\^T}\ps{\max_{a' \in [N]} \Swap{a}{a'}}&\geq \sum_{t=1}^{T'} \EE_{r\^T}\ps{\vxtr[a_v]}\label{eq:case6sum}
        \end{align}
        }
    \end{lemma}

\begin{proof}[Proof of Lemma \ref{lem:case6}]
    We have
    \begin{align*}
        \EE_{r\^T}\ps{\max_{a' \in [N]} \Swap{a}{a'}}&\geq \EE_{r\^T}\ps{\frac{1}{2}\p{\Swap{a}{a_{\emptyset}}+\Swap{a}{\dot{a}_{\emptyset}}}}\\
        &=\frac{1}{2}\sum_{t=1}^T \EE_{r\^T}\ps{\vxtr[a]\p{\vut[a_{\emptyset}]+\vut[\dot{a}_{\emptyset}]-2\vut[a]}}\\
        &= \frac{1}{4D}\sum_{t=1}^{T'} \EE_{r\^T}\ps{\vxtr[a]}
    \end{align*}

    which gives \eqref{eq:case6sum}, as desired.
\end{proof}

    \section{Dimensions of games and function classes} \label{app:dimensions}
In this section, we review the definitions of sequential complexity measures for real-valued \emph{function classes} $\MH$, namely a set of \emph{concepts} $h: \MZ \to \MY$, where $\MZ$ is called the \emph{domain set} and $\MY$ is called the \emph{label set}.

\paragraph{Trees.} For a set $\MZ$, an \emph{$\MZ$-valued tree $\MT$} of \emph{depth $d$} is a complete rooted binary tree $\MT$ each of whose nodes are labeled by an internal node. Each node of the tree is associated to a sequence $(\ep_1, \ldots, \ep_t) \in \{ -1,1\}^{t-1}$, describing the root to leaf path for that node (i.e., with $+1$ corresponding to `right' and $-1$ corresponding to `left'). Accordingly, the tree $\MT$ may be specified by a sequence $\bz = (\bz_1, \ldots, \bz_d)$ of functions $\bz_t : \{-1,1\}^{t-1} \to \MZ$, for $t \in [d]$, where $\bz_t(\ep_{1:t-1})$ denotes the label of the node $\ep_{1:t-1}$.

\begin{definition}[Sequential Rademacher complexity]
  \label{def:seq-rad}
  For a real-valued function class $\MH : \MZ \to \BR$ and an integer $T \in \mathbb{N}$, its \emph{sequential Rademacher complexity} (at depth $T$) is defined as
  \begin{align}
\Rad_T(\MH) := \sup_{\bz} \E \left[ \sup_{h \in \MH} \frac 1T \sum_{t=1}^T \ep_t h(\bz_t(\ep_{1:t-1})) \right],
  \end{align}
  where the expectation is over i.i.d.~Rademacher sequences $\ep_1, \ldots, \ep_T \sim \mathsf{Unif}(\{-1,1\})$, and the supremum is over all $\MZ$-valued trees $\bz$ of depth $T$.
\end{definition}

The sequential Rademacher complexity is known to tightly characterize the sample complexity of (agnostically) online learning a class $\MH$. In particular, the minimax external regret can be upper bounded as follows:
\begin{theorem}[Theorem 7 of \cite{rakhlin2014online}\footnote{The minimax regret as defined in \cite{rakhlin2014online} is defined slightly differently to the expression in \cref{eq:mh-regret}, in that the adversary in \cite{rakhlin2014online} can observe the draws $h\^t \sim \bq\^t$ before choosing $\bz\^{t+1}$. It is straightforward to see that the two are equivalent.}]
  \label{thm:rakhlin-online}
  For any function class $\MH \subset \BR^\MZ$, there is a (randomized) algorithm which, for any adaptive adversary choosing a sequence $\bz^1, \ldots, \bz\^T \in \MZ$, produces a sequence of distributions $\bq\^1, \ldots, \bq\^T \in \Delta(\MH)$ so that  \begin{align}
\ExtRegret(\bq\^{1:T}, \bz\^{1:T}) = \sup_{h^\st \in \MH} \frac 1T \sum_{t=1}^T \left( h^\st(\bz\^t) - \E_{h\^t \sim \bq\^t}[h(\bz\^t)] \right) \leq 2 \Rad_T(\MH)\label{eq:mh-regret}.
  \end{align}
\end{theorem}

\paragraph{Combinatorial complexity measures.}
As a corollary of \cref{thm:rakhlin-online}, the external regret for a function class may be upper bounded by combinatorial complexity measures. We first consider the binary case, for which the relevant complexity measure is the \emph{Littlestone dimension}. 
\begin{definition}[Littlestone Dimension] \label{def:littlestone}
  For a function class $\MH$ with domain $\MZ$ and  binary label set $\MY=\set{0,1}$, define its Littlestone Dimension $\Ldim(\MF)$ to be the maximum depth $d$ of a $\MZ$-valued tree $\bz = (\bz_1, \ldots, \bz_d)$ so that, for all $\ep_{1:d} \in \{-1,1\}^d$, there is some $h \in \MH$ so that $h(\bz_t(\ep_{1:t-1})) = \ep_t$ for each $t \in [d]$. 
\end{definition}

The analogue of Littlestone dimension for real-valued function classes, is the \emph{sequential fat-shattering dimension}:
\begin{definition}[$\delta$-Sequential Fat Shattering Dimension] \label{def:esfsd}
  For a function class $\MH$ with domain $\MZ$ and label set $\MY=\RR$, denote its $\delta$-sequential fat shattering dimension $\esfsd{\MH}$ is the maximum integer $d$ so that there are complete binary trees $\bs, \bz$ of depth $d$ so that for all $\ep_{1:d} \in \{-1,1\}^d$, there is some $h \in \MH$ so that
  \[
\ep_t \cdot \left(h(\bz_t(\ep_{1:t-1})) - \bs_t(\ep_{1:t-1})\right) \geq \delta.
    \]
  \end{definition}

  The following result, which upper bounds sequential Rademacher complexity in terms of the Littlestone and sequential fat-shattering dimensions, may be combined with \cref{thm:rakhlin-online} to obtain an upper bound on the external regret in terms of the respective combinatorial complexity measures:
  \begin{proposition}[Proposition 18 of \cite{block2021majorizing} \& Proposition 9 of \cite{rakhlin2014online}]
    \label{prop:ldim-sfat}
    Consider a function class $\MH \subset \BR^\MZ$. Then:
    \begin{itemize}
    \item If $\MH$ is binary-valued (i.e., $\MH \subset \{0,1\}^\MZ$), then $\Rad_T(\MH) \leq \sqrt{\Ldim(\MH)/T}$.
    \item If $\esfsd{\MH} \leq O(\delta^{-p})$ for some $p \in [0,2)$, then $\Rad_T(\MH) \leq O(1/\sqrt{T}) \cdot \int_0^1 \sqrt{\esfsd{\MH}} d\delta$.
    \item In general, $\Rad_T(\MH) \leq O\left( \alpha + \frac{1}{\sqrt{T}} \cdot \int_\alpha^1 \sqrt{\esfsd{\MH} \log(T/\delta)}d\delta\right)$ for any $\alpha \in (0,1)$. 
    \end{itemize}
  \end{proposition}

\paragraph{Complexity measures for games.} The complexity measures for function classes introduced above may be extended to games in the intuitive way.  Consider an $m$-player game $(S,A)$, where $S = S_1 \times \cdots \times S_m$ denotes the joint action set, and $A_j : S \to \RR$ denotes player $j$'s payoff function.

For each player $j$, we define a function class $\MX_j : S_j \to \RR$ as follows: for each action profile of the other players $s_{-j} \in S_{-j}$, define $f_{s_{-j}}(s_j)  := A_j(s_j, s_{-j})$, for $s_j \in S_j$. Then set
\[
\MX_j = \{ f_{s_{-j}} \ : \ s_{-j} \in S_{-j} \},
\]
so that $\MX_j$ is indexed by $S_{-j}$. For a given complexity measure, we define its value for the game $(S,A)$ to be its maximum value over the functions classes $\MX_1, \ldots, \MX_m$. 
\begin{definition}[Complexity measures for multiplayer games]\label{def:mat-form}
Let $(S,A)$ be an $m$-player game.  If $A_j: S \to \set{0,1}$ is binary-valued for all players, we define
\begin{align*}
    \Ldim(S,A) &= \max_j \Ldim(\MX_j)\\
    \intertext{and if the game is real-valued, define}
  \esfsd{S,A} &= \max_j \esfsd{\MX_j} \\
  \Rad_T(S,A) &= \max_j \Rad_T(\MX_j).
\end{align*}
\end{definition}

{
  \section{An adaptive lower bound on the swap regret}
\label{sec:adaptive-lb}
  In this section, we present \cref{thm:adaptive-lb}, which gives an alternative lower bound on the swap regret. For simplicity, we focus on the setting when $N\geq T$, for which we obtain a lower bound of $\Omega(\log^{-3} T)$ on the swap regret; our techniques readily extend to obtain a lower bound that scales as $\Omega\p{\frac{\sqrt{N/T}}{\log^{O(1)} N}}$ when $T > N$. 

  Compared to \cref{thm:lower-main}, the adversary established by \cref{thm:adaptive-lb} is \emph{adaptive} and does not satisfy the property that its reward vectors $\bu\^t$ have $\ell_1$ norm bounded above by 1. On the other hand, the bound obtained by \cref{thm:adaptive-lb} is quantitatively stronger than \cref{thm:lower-main}: in the regime $T \leq N$, \cref{thm:lower-main} obtains a lower bound of $\Omega(\log^{-5} T)$, which is smaller than the bound of $\Omega(\log^{-3} T)$ of \cref{thm:adaptive-lb}. 
  \begin{theorem}
    \label{thm:adaptive-lb}
    Fix any $T \in \mathbb{N}$. Then, for any learning algorithm on $N = T$ actions, there is an adaptive adversary guaranteeing that the swap regret is bounded below by
    \begin{align}
\SwapRegret(T) \geq \Omega \left( \frac{1}{\log^3 T} \right)\nonumber.
    \end{align}
  \end{theorem}

  \paragraph{Proof overview for \cref{thm:adaptive-lb}.} Roughly speaking, the adversary constructs a  sequence of ``template'' utility vectors $\ubase\^t \in [-1,1]^N$, for $t \in [T]$ (defined formally in \cref{eq:defubase}). The vector $\ubase\^1$ is a monotonically decreasing vector, where the entries decrease by $\Delta = O(1/\log N)$ on a logarithmic scale: in particular, the first $2$ entries are equal to $1$, the next $2 \cdot 2^0$ entries are equal to $1-\Delta$, the next $2 \cdot 2^1$ entries are equal to $1-2\Delta$, the next $2 \cdot 2^3$ entries are equal to $1-3\Delta$, and so on. Then, $\ubase\^t$ is a shift of $\ubase\^1$ rightward by $2 \cdot (t-1)$ positions, with entries on the left filled in with $-1$. For each $t$, we set $\bu\^t$ to be equal to $\ubase\^t$ with some modifications, which we proceed to describe.

  At a high level, on round $t$, the learner is best off by playing either action $2t-1$ or $2t$ (since both have utility equal to 1 for $\ubase\^t$). To make the learner ``pay'' for doing so, one of $\bu\^t[2t-1], \bu\^t[2t]$ is randomly perturbed by a small  amount  (namely, $\Delta/2$) for each round $t$, so if the learner spends only $O(1)$ rounds playing actions $2t-1, 2t$, they will incur $\Omega(1)$ swap regret for the actions $2t-1, 2t$. In particular, swapping either $2t-1$ to $2t$ or $2t$ to $2t-1$ will yield $\Omega(1)$ swap regret.  Adding this quantity up over all actions would yield $\Omega(T)$ swap regret (this corresponds to Case 2 in the proof below).

  However, the learner could attempt to proceed more cleverly so as to minimize the contribution of the random perturbations to their swap regret: suppose they partition $[T]$ into ``moderate-sized'' sub-intervals, within each of which they play a fixed action $a$ with utility close to 1 throughout the duration of that sub-interval. If they do so, then they will nevertheless typically incur swap regret $\Omega(\Delta)$ per round due to the ability to swap to some action $2s-1 < a$ during rounds $t \leq s$. This lower bound crucially uses the logarithmic scaling of the utilities $\bu\^t$.

 While the above-described adversary is oblivious, it does not quite rule out small swap regret: indeed, a  ``trivial'' learner  which plays some fixed action $a$ for all $t$ rounds will obtain small swap regret against the above-described adversary. To rule such a learner out, whenever an action in $\{2s-1, 2s\}$ (for any $s \in [N/2]$) has been played for $\Omega(1)$ rounds, the adversary sets the utility of $s$ to be $-1$ at all future rounds. In this way, even if the learner decides to play $2s-1$ or $2s$ in later rounds, it incurs large swap regret for not switching to some action $a' > 2s$. This latter modification leads the proof to be somewhat technical, as we need to carefully account for actions which have been ``switched'' to $-1$ in this manner.  
  
\begin{proof}[Proof of \cref{thm:adaptive-lb}]
  Fix $T \in \mathbb{N}$. 
  Fix any learning algorithm $\mathscr{A}$ which, at each step $t \in [T]$, outputs a distribution $\bx\^t \in \Delta_N$. We will define an adaptive adversary which generates reward vectors $\bu\^1, \ldots, \bu\^T \in [0,1]^N$. At each round $t$, the output of the learner at step $t$ may be described by some function $\mathscr{A}_t(\bx\^{1:t-1}, \bu\^{1:t-1}) \in \Delta_{\Delta_N}$ (namely, $\mathscr{A}_t(\bx\^{1:t-1}, \bu\^{1:t-1})$ is a probability distribution over vectors in $\Delta_N$). We define
  \begin{align}
\bp\^t(\bx\^{1:t-1}, \bu\^{1:t-1}) := \E_{\bx \sim \mathscr{A}_t(\bx\^{1:t-1}, \bu\^{1:t-1})}[\bx]\nonumber.
  \end{align}
  We will often abbreviate $\bp\^t(\bx\^{1:t-1}, \bu\^{1:t-1})$ as $\bp\^t$. 
  Let $\MF\^t := \sigma(\{\bx\^{1:t}, \bu\^{1:t}\})$ be the sigma algebra generated by the learner's and adversary's plays up to round $t$. Thus $\bp\^t$ is $\MF\^{t-1}$-measurable.

  \paragraph{Construction of the adversary.} Given $T \in \mathbb{N}$, define $L = \lfloor \log(T/2) \rfloor$ and $\Delta := 1/L$. We set $N := 2 \cdot ( 1 + 2 + \cdots + 2^{L-1}) \leq T$. For $0 \leq i < N$, define
  $ 
    \Fbase(i) :=
      \lfloor \log(1 + \frac i2)\rfloor \cdot \Delta $,
    where the logarithm is taken base $2$. Note that $\Fbase(0) = 0$, $\Fbase(2 \cdot (2^j-1)) = \Delta j$ for integers $j \geq 1$,  and  $\Fbase(N-2) = L=(L-1)\Delta = 1-\Delta$. For $t \in [T]$, define $\ubase\^t \in [-1,1]^N$ as follows: for $a \in [N]$, 
  \begin{align}
    \ubase\^t[a] := \begin{cases}
      1 - \Fbase(a-(2t-1)) &: a \geq 2t-1 \\
      -1&: a < 2t-1.
    \end{cases}\label{eq:defubase}
  \end{align}
In words, $\ubase\^t$ is a shift of the function $1-\Fbase(\cdot)$ rightward by $2t-1$ units, with entries before $2t-1$ all set to $-1$. 
  
  Fix $\zeta = 1/(32L)$. We pair up each action $2t-1$ with action $2t$, for $t \in [N/2]$: for $a \in [N]$, we let $\mathfrak{p}(a)$ denote its pair. Moreover, write $\bar\bp\^t[a] := \bp\^t[a] + \bp\^t[\pair(a)]$ (so that $\bar\bp\^t[a] = \bar\bp\^t[\pair(a)]$ for all $a \in [N]$). 
  We now define the reward vectors $\bu\^t$ chosen by the adversary (as a function of $\bp\^t$), as follows: let $r\^1, \ldots, r\^{N/2} \in \{0,1\}$, denote a sequence of independent and uniformly distributed bits. For each $t \leq N/2$, $a \in [N]$ and $0 \leq k < L$, define the sets $\MG\^t, \MS\^t(a,k) \subset [t]$ and real numbers $\sigma\^t(a,k), \Sigma\^t(a) \geq 0$ recursively with respect to $t$, as follows:
  \begin{align}
    \MG\^t := & \{ s \in [t] \ : \ \Sigma\^{s-1}(2s-1) <  \zeta \}\nonumber\\
    \MS\^t(a, k) := & \left\{ s \in \MG\^t \ : \ a > 2s, \  \left\lfloor \log\left(  \frac{a - (2s-1)}{2} \right) \right\rfloor= k \right\} \nonumber\\
    \sigma\^t(a,k) := & \sum_{s \in \MS\^t(a,k)} \bar\bp\^s[a], \qquad \Sigma\^t(a) := \max_{0 \leq k < L} \sigma\^t(a,k) \nonumber.
  \end{align}
Note that $\Sigma\^{t-1}(2t-1) = \Sigma\^{t-1}(2t)$ by our definition of $\bar\bp\^t$.  
  We write $\MG := \MG\^{N/2}$. Intuitively, the meaning of $\MG\^t, \MS\^t(a,k), \sigma\^t(a,k), \Sigma\^t(a)$ are as follows:
  \begin{itemize}
  \item $\MG\^t$ denotes the set of \emph{active} rounds $s$ up to round $t$: a round $s$ is active, if, roughly speaking, actions $2s-1, 2s \in [N]$ have not been played too much by the algorithm $\SA$ up to round $s$. 
  \item $\MS\^t(a,k)$ denotes the set of active rounds $s$ up to step $t$ for which $\lfloor\frac{a-(2s-1)}{2}\rfloor$ is in $\{2^k, 2^k + 1, \ldots, 2^{k+1}-1 \}$. Since $\Fbase(2j) \geq \Delta +  \Fbase(2(j-2^k))$ whenever $j \in \{2^k, 2^k + 1, \ldots, 2^{k+1}-1\}$, it follows that for all $s \in \MS\^t(a,k)$, for any $b$ satisfying $a/2 >  b \geq 2^k$,
$      \ubase\^s[a-2b] - \ubase\^s[a] = \Fbase(a-(2s-1)) - \Fbase(a-2b-(2s-1)) \geq \Delta.$ Since $s' \in \MS\^t(a,k)$ implies that $\frac{a-(2s'-1)}{2} \geq 2^k$, it follows by choosing $b = \frac{a-(2s'-1)}{2}$ that, for any $s,s' \in \MS\^t(a,k)$,
  \begin{align}
    \label{eq:ssprime-delta}
\ubase\^s[2s'-1]-    \ubase\^s[a]  \geq \Delta.
  \end{align}
\item $\sigma\^t(a,k)$ denotes the aggregate amount of mass that $\SA$ puts on action $a$ in rounds in $\MS\^t(a,k)$.
\item $\Sigma\^t(a)$ denotes the maximum amount of mass that $\SA$ puts on action $a$ in any of the sets $\MS\^t(a,k)$. 
  \end{itemize}

  We will now define $\bu\^t \in [-1,1]^N$ as follows: if $t \not \in \MG\^t$, then we set $\bu\^t=  0$. (Note that $\MG\^t$ only depends on $\bp\^s$ for $s < t$, so this operation is well-defined.) Otherwise, we define %
  \begin{align}
    \bu\^t[a] & := \begin{cases}
      \ubase\^t[a] = -1 &: a < 2t-1 \\
      \ubase\^t[a] -\frac{\Delta}{2} \cdot \One\{r\^t\equiv a \pmod{2} \} &: a \in \{2t-1, 2t\}\\
      \ubase\^t[a]%
      &: a > 2t,\ \Sigma\^{t-1}(a) < \zeta \\
      -1 &: a > 2t, \Sigma\^{t-1}(a) \geq \zeta.
    \end{cases}\label{eq:ut-define}
  \end{align}

  Finally, for $N/2 < t \leq T$, define $\bu\^t[a] = -1$  for all $a \in [N]$. %

  \paragraph{Proof of regret lower bound.} %
  For each $a \in [N]$, define
  \begin{align}
\tau(a) := \min \{ t \in [N/2] \ : \ a < 2t-1 \mbox{ or } \Sigma\^{t-1}(a) \geq \zeta \}\nonumber.
  \end{align}
  Roughly speaking, $\tau(a)$ denotes the first round at which either $2t-1$ exceeds $a$ or else $a$ is played ``too much'' by $\SA$ in the sense that $\Sigma\^{t-1} \geq \zeta$. We say that an action $a$ is \emph{stale} at round $t$ if $t \geq \tau(a)$. 
  Note that for all $t \in \MG$ with $t > \tau(a)$, we have $\bu\^t[a] = -1$.  
  We define the following quantities, for $a \in [N]$:%
  \begin{align}
P_-(a) := \sum_{t\in \MG :\ t < \tau(a), a > 2t } \bar\bp\^t[a], \qquad P_0(a) := \sum_{t\in \MG :\ a \in \{2t-1,2t\}} \bar\bp\^t[a], \qquad P_+(a) := \sum_{t\in \MG :\ t \geq \tau(a)} \bar\bp\^t[a]\nonumber.
  \end{align}
  $P_-(a)$ denotes the total mass placed on $a$ and $ \pair(a)$ in all rounds when $a$ is active except $2t-1, 2t$, $P_0(a)$ denotes the mass places on $a$ and $\pair(a)$ during the rounds $2t-1, 2t$, and $P_+(a)$ denotes the mass placed on $a$ and $ \pair(a)$ in the remaining rounds.

  Also write $P(a) = P_-(a) + P_0(a) + P_+(a)$ and $P := \sum_{a \in [N]} P(a)$.  Finally, we define
  \begin{align}
\MA_0 := \left\{ a \in [N] \ : \ \Sigma\^{t-1}(a) \geq \zeta \mbox{ for } t = \left\lfloor \frac{a+1}{2} \right\rfloor \right\}, \quad \MA_1 := \left\{ a \in [N] \ : \ P(a) \geq 4 \zeta L \right\}, \quad \MA := \MA_0 \cup \MA_1\nonumber. %
  \end{align}
  $\MA_0$ denotes the set of actions $a$ which have become stale at some point prior to the unique round $t$ for which $a \in \{2t-1, 2t\}$. $\MA_1$ denotes the set of actions $a$ which are ``played a lot'' by $\SA$ over all rounds in $\MG$. %
  We next state the following claim, whose proof is provided following the proof of the theorem.
  \begin{claim}
    \label{clm:a-bound}
    It holds that $\sum_{a \in \MA} P(a) \geq \zeta N / 4$. 
\end{claim}

    Consider any $a \in \MA$. One of the below cases must hold:

\paragraph{Case 1: $P_-(a)  \geq P(a)/4$.} %
We claim that in fact $a \in \MA_0$. To see this, suppose not, which means that $a \in \MA_1$, and 
let $\tau'(a) < \tau(a)$ denote the largest integer $t \in \MG$ which is strictly less than $\tau(a)$. Then, letting $t_a = \lfloor (a+1)/2 \rfloor$, 
\begin{align}
\zeta L \leq P(a) / 4 \leq P_-(a) = \sum_{k=0}^{L-1} \sigma\^{\tau'(a)}(a,k) =\sum_{k=0}^{L-1} \sigma\^{t_a-1}(a, k)\label{eq:zetal-lb}.
\end{align}
The first equality above uses the fact that each $s \in \MG\^t$ must belong to (exactly) one of the sets $\MS\^t(a,k)$, for $0 \leq k < L$. Thus there is some $0 \leq k_a < L$ so that $\sigma\^{t_a-1}(a,k_a) = \sigma\^{\tau'(a)}(a,k_a) \geq \zeta$, which implies that $\Sigma\^{t_a-1}(a) \geq \zeta$, i.e., $a \in \MA_0$, thus establishing our claim.

We remark for later use that by definition of $\tau'(a)$, $\sigma\^{\tau'(a)}(a,k_a)$ must in fact be the largest of the values $\sigma\^{\tau'(a)}(a,k)$, for $k < L$ (as otherwise there would be some $t \leq \tau'(a)$ so that $\Sigma\^{t-1}(a) \geq \zeta$, contradicting the definition of $\tau(a)$).
    We may now write the regret for not swapping actions $a$ and $\pair(a)$ to another action $a'$ as follows:
    \begin{align}
      & \max_{a' \in [N]} \sum_{s=1}^{N/2}\left( \bp\^s[a] \cdot(\bu\^s[a'] - \bu\^s[a]) + \bp\^s[\pair(a)] \cdot (\bu\^s[a'] - \bu\^s[\pair(a)]) \right)\nonumber\\
      \geq  & \max_{a' \in [N]} \sum_{s \in \MG,\ s \leq \tau'(a)} \bar\bp\^s[a] \cdot (\bu\^s[a'] - \bu\^s[a])\nonumber\\
      \geq & \sum_{s \in \MS\^{\tau'(a)}(a, k_a)} \bar\bp\^s[a] \cdot (\bu\^s[2\tau'(a)-1] - \bu\^s[a]) \geq \Delta \cdot \frac{P(a)}{8L}\label{eq:complicated-ineqs},
    \end{align}
    where the first inequality holds because all $s \not \in \MG$ have $\bu\^s = \mathbf{0}$, for $s \in \MG$ with $s > \tau'(a)$ (and thus $s \geq \tau(a)$), we have $\bu\^s[a]=-1$, and for $s \in \MG$ with $s \leq \tau'(a)$, we have $\bu\^s[a] = \bu\^s[\pair(a)]$. The second inequality follows by the choice of $a' = 2\tau'(a) - 1$ together with the fact that, since $\tau'(a) \in \MG$, for $s \in \MG$ with $s \leq \tau'(a)$, we have $\Sigma\^{s-1}(2\tau'(a)-1) < \zeta$ and thus $\bu\^s[2\tau'(a) - 1] = \ubase\^s[2\tau'(a)-1] \geq \ubase\^s[a] \geq \bu\^s[a]$. 

    Finally, the third inequality in \cref{eq:complicated-ineqs} follows because $\sigma\^{\tau'(a)}(a, k_a) \geq P_-(a)/L \geq P(a)/(4L)$ from \cref{eq:zetal-lb} and $\bu\^s[2\tau'(a)-1] - \bu\^s[a] \geq \Delta$ for all $s \in \MS\^{\tau'(a)}(a,k_a)$. To see this latter implication, first note that $\ubase\^s[2\tau'(a)-1] - \ubase\^s[a] \geq \Delta$ by (\ref{eq:ssprime-delta}) and the fact that $\tau'(a) \in \MS\^{\tau'(a)}(a, k_a)$ (since $\sigma\^{\tau'(a)}(a, k_a)$ %
    must surpass $\zeta$ during iteration $\tau'(a)$, and the only way for this to happen is that $\tau'(a) \in \MS\^{\tau'(a)}(a, k_a)$). Then we note that, by definition of $\tau'(a)$, $s \in \MS\^{\tau'(a)}(a, k_a)$ satisfies $\ubase\^s[a] = \bu\^s[a]$ (since $\tau'(a) < \tau(a)$) and $\ubase\^s[2\tau'(a)-1] -\frac{\Delta}{2} \leq  \bu\^s[2\tau'(a)-1]$. Then
    \[
\bu\^s[2\tau'(a) - 1] - \bu\^s[a] \geq -\frac{\Delta}{2} + \ubase\^s[2\tau'(a)-1] - \ubase\^s[a] \geq \frac{\Delta}{2}.
      \]
    
      \paragraph{Case 2: $P_0 (a) \geq P(a) / 12$.} %
      By replacing $a$ with its pair $\pair(a)$ if necessary, we may assume that $\sum_{t \in \MG :\ a \in \{2t-1, 2t\}} \bp\^t[a] \geq P(a)/24$. 
    Let us further suppose that $a$ is odd (the case that $a$ is even is handled symmetrically). Write  $t_a = \frac{a+1}{2}$, and note that, since for all $t \neq t_a$, we have $\bu\^t[a] = \bu\^t[a+1]$, 
    \begin{align}
      \sum_{t \in \MG} \bp\^t[a] \cdot (\bu\^t[a+1] - \bu\^t[a]) = \frac{\Delta}{2} \cdot \bp\^{t_a}[a] \cdot (2r\^{t_a} - 1) \nonumber.
    \end{align}
    Thus, the expected swap regret for action $a$ may be lower bounded by \begin{align}\label{eq:P0-swap}\E\left[ \max \left\{ 0, \frac{\Delta}{2} \cdot \bp\^{t_a}[a] \cdot (2r\^{t_a} - 1) \right\}\mid \MF\^{t_a-1} \right] = \frac{\Delta}{2} \cdot \bp\^{t_a}[a]  \geq \frac{\Delta \cdot P(a)}{48}.\end{align}

    \paragraph{Case 3: $P_+(a) \geq 2P(a)/3$.} %
    Note that for all $t \in \MG$ with $t > \tau(a)$, we have $\bu\^t[a] = -1$. Let $a^\st \in [N]$ denote the largest action so that $2a^\st - 1\in\MG$, and suppose first that $a \not \in \{a^\st, \pair(a^\st)\}$ (hence $a < a^\st$). Then the swap regret for actions $a, \pair(a)$ can be shown to be large by using the swap function which swaps to the action $a^\st > a$:  %
    \begin{align}
      & \max_{a' \in [N]} \sum_{s=1}^{N/2} \left(\bp\^s[a] \cdot (\bu\^s[a'] - \bu\^s[a]) + \bp^s[\pair(a)] \cdot (\bu\^s[a'] - \bu\^s[\pair(a)] )\right)\nonumber\\
      = & \max_{a' \in [N]} \sum_{s \in \MG} \left(\bp\^s[a] \cdot (\bu\^s[a'] - \bu\^s[a]) + \bp^s[\pair(a)] \cdot (\bu\^s[a'] - \bu\^s[\pair(a)] )\right)\nonumber\\
      \geq & -\sum_{s \in \MG : s \leq \tau(a)} \bar\bp\^s[a] + \sum_{s \in \MG : s > \tau(a)} \bar\bp\^s[a] \nonumber\\
      =&  P_+(a) - (P_0(a) + P_-(a)) \geq P(a) / 3, \label{eq:pa3}
    \end{align}
    where the first equality uses that $\bu\^s[a] = -1$ for all $a$ and $s \not \in \MG$, and the  inequality uses the fact that $\bu\^s[a^\st] - \bu\^s[a] \geq -1$ for all $s \in \MG$ with $s \leq \tau(a)$ as well as the fact that for $s > \tau(a)$, we have $\bu\^s[a] = -1$ and $\bu\^s[a^\st] \geq 0$ since $s \in \MG$ and all $s \in \MG$ satisfy $\Sigma\^{s-1}(a^\st) < \zeta$ (by our choice of $a^\st$ as large as possible so that $2a^\st -1 \in \MG$). 

    In the event that $a \in\{ a^\st, \pair(a^\st)\}$, we may use the same argument as above with $a^\st$ instead being the \emph{second-largest} action so that $2a^\st - 1\in\MG$, which gives a lower bound of $P(a)/3-2$ in (\ref{eq:pa3}). (The $-1$ is insignificant in our final lower bound on swap regret, since $\{ a^\st, \pair(a^\st )\}$ only contains two actions.) 

    Combining the above cases, we obtain that
    \begin{align}
2 \cdot \E \left[\sum_{a\in \MA} \max_{a' \in [N]}  \sum_{s=1}^{N/2} \bp\^s[a] \cdot (\bu\^s[a'] - \bu\^s[a]) \right]\geq -2 + \sum_{a \in \MA} \frac{P(a) \Delta}{48 L} \geq  \frac{N \cdot \zeta \Delta }{48 L}-1\nonumber,
    \end{align}
    where the factor of 2 on the left-hand side arises because in our arguments above we have double-counted each action $a$ with its $\pair(a)$ in \cref{eq:complicated-ineqs,eq:P0-swap,eq:pa3}, and the final inequality uses Claim \ref{clm:a-bound}. 
    Thus the expected swap regret is bounded below by $\Omega(N/ \log^3 N) = \Omega(T / \log^3 T)$.
  
  \end{proof}

  \begin{proof}[Proof of Claim \ref{clm:a-bound}]
  We consider two cases, depending on $|\MG|$:

  In the first case, we assume $|\MG| < N/4$. Consider any $a \in \MA_0$, and let $t_a = \lfloor \frac{a+1}{2} \rfloor$, so that $a \in \{2t_a-1,2t_a\}$. Since $a \in \MA_0$, we must have $\Sigma\^{t_a-1}(a) \geq \zeta$, i.e., $t_a \in [N/2] \backslash \MG$. Since $[N/2]\backslash \MG$ has size at least $N/4$ by assumption, we see that
  \begin{align}
\sum_{a \in \MA_0} P(a) \geq \sum_{t \in [N/2]\backslash \MG}  \Sigma\^{t-1}(2t-1) \geq \zeta N / 4\nonumber.
  \end{align}
  
 In the second case, we have $|\MG| \geq N/4$. %
 Then $\sum_{a \in [N]} P(a) = \sum_{a \in [N]} \sum_{t \in \MG} \bar\bp\^t[a] \geq |\MG| \geq N/4$. %
 Since $4\zeta LN \leq N/8$ by our choice of $\zeta = 1/(32L)$, it follows that $\sum_{a \in \MA_1} P(a) \geq N/4 - N/8 \geq \zeta N / 4$. 
\end{proof}

}

\printbibliography

\end{document}